%% file: bare_jrnl.tex
\documentclass[journal,onecolumn]{IEEEtran}
\usepackage{amsmath}
\usepackage{amssymb}
\usepackage{nicefrac}
\usepackage{color}
\definecolor{ao(english)}{rgb}{0.0, 0.5, 0.0}

\usepackage{rotating}
\usepackage{comment}

\usepackage{caption}
\usepackage{subcaption}

\usepackage{url}
\usepackage{tabularx}
\usepackage{multirow}
\usepackage[skip=1ex]{caption}

\newtheorem{theorem}{Theorem}[section]

\usepackage[utf8]{inputenc}
\usepackage{babel}

\begin{document}
%
\title{DeepFlow: Two Convergent Differentiable Approaches for Pairwise Conditional Random Fields}
%
%
%

\author{John~S.H.~Baxter*
\thanks{\hspace{-2.5mm}(*) Laboratoire Traitement du Signal et de l'Image (INSERM UMR 1099), Universit\'e de Rennes, Rennes, France}
\thanks{\hspace{-2.5mm}J.S.H. Baxter is the corresponding author. E-mail: john.baxter@univ-rennes.fr.}
}

%
%

\markboth{ }%
{Baxter \& Jannin: DeepFlow Supplementary Material}
%



\maketitle

%
\IEEEpeerreviewmaketitle
\input{1_intro}
\input{2_theory}
\input{3_exp_marg}
\input{4_exp_learning}
\input{5_discussion}

\bibliographystyle{IEEEtran}
\bibliography{deepflow.bib}

\newpage
\begin{appendices}
\input{A1_plbp_1}
\input{A2_plbp_2}

\input{A3_mfi_1}
\input{A4_mfi_2}

\end{appendices}

\end{document}

%% file: 1_intro.tex
\begin{abstract}
Basic probabilistic models form the basis of modern deep learning largely by defining the final activation layer of a network and allowing the networks outputs to be interpreted through a probabilistic lens. By increasing the complexity of these final activation layers to include features such as spatial regularisation, allow for these models to improve in performance without requiring significantly more data. The challenge however is to ensure that these more complex probabilistic graphical models adapt well to being used in tandem with deep learning, that is, ensuring they converge and have meaningful derivatives. The goal of this paper is to show two algorithms with provable convergence properties to allow this type of regularisation for image and signal classification models. Through both rigorous mathematical analysis and in silico experiments, we can show the differing performance between these solution types.
\end{abstract}

\section{Introduction}
\IEEEPARstart{P}{robabilistic} models have a long history in both physics and artificial intelligence and form the basis of our understanding of current machine learning \cite{koller2009probabilistic}.
At its most general level, a probabilistic model can be described by the following equation:
\begin{equation*}
    P[u|\psi] = \frac{e^{E_{[\psi,u]}}}{Z_{[\psi]}}
\end{equation*}
where $\psi$ are the model's parameters and $u$ is a possible configuration, i.e. a possible way of assigning values to variables. The formulation involves an energy equation $E_{[\psi,u]}$ which expresses the log-likelihood of a particular configuration and $Z_{[\psi]}$ is the \textit{partition function} which normalises the equation to be a valid probability distribution by summing over all valid configurations:
\begin{equation*}
    Z[\psi] = \sum_{u\in\mathbb{U}}e^{E_{[\psi,u]}}
\end{equation*}
The benefit of probabilistic graphical models (PGMs) is that they simplify the potentially complex energy term, $E_{[\psi,u]}$, as the sum of a large number of simpler terms each corresponding to cliques in a particular graph whose vertices correspond to particular variables. Given the graph and the parameters (often expressed as weights in this summation), there are three particularly relevant problems: (i) estimation of the exact probability of a given configuration which relies on an estimation of the partition function $Z_[\psi]$, (ii) estimation the marginal probability distribution of a particular configuration within a single clique, and (iii) estimation of the maximum probability configuration. Even with the simplification of the energy function, these problems can be difficult or intractable, leading to an interesting in particular families of PGMS

One particularly interesting subset are called \textit{conditional random fields} (CRFs) in which the cliques of interest are limited to two vertices, thus only having energy terms and parameters associated with vertices and edges. These models proved to be very influential in computer vision and medical image processing due to the proliferation of a group of techniques known as \textit{max-flow/min-cut} or \textit{graph-cuts} which could efficiently estimate the maximum probability configuration of CRFs where the pairwise terms had a regularising influence \cite{chen2018survey}. At the time, researchers found that even relatively simple \textit{bias terms} (the parameters associated with individual vertices) and \textit{smoothness} or \textit{regularisation terms} (the parameters associated with individual edges) could produce high-performance results for image segmentation \cite{chen2018survey} and stereoscopic depth estimation \cite{bleyer2013stereo} problems.

Since the proliferation of convolutional neural networks (CNNs), especially those specifically designed for image-to-image problems such as U-Nets \cite{ronneberger2015u}, the use of CRFs has largely fallen to the wayside as these models tend to produce much more highly accurate bias terms to a point where regularisation terms are largely seen as unnecessary. From a technical point of view, relying on models constructed only of bias terms has distinct advantages, notably that the three aforementioned problems above are trivially computable. 

Despite the movement towards more complex data terms and away from regularisation, there are several potential reasons why we may want to integrate CRF back into our deep learning frameworks. Many frameworks have found that post-processing with graph-cuts allows for improved application-time predictions from deep learning based segmentation models \cite{}. There has also been raised interest in end-to-end learning, integrating these post-processing techniques into training-time using differentiable estimators. For example, Xie \textit{et al.} \cite{xie2025gc} use a binary-graph-cut-specific loss that allows for them to use the non-differentiable graph-cuts operator (for finding maximum probability configurations) which side-stepping the issue of its (mostly) zero gradients which would prevent model learning. Another approach has been to estimate the marginal probability distribution for each label (a fundamentally differentiable problem) using a series mean-field estimation operators which can be automatically differentiated \cite{zheng2015conditional,arnab2016higher,xu2020fast,guan2024neural}. The benefit of such an approach is that the marginal distribution provides a reasonable equivalent to traditional output activation functions for segmentation (sigmoid for binary segmentation and softmax for multi-region segmentation) and can be used with any existing loss function. However, they do require the number of iterations to be set ahead of time and no in-depth investigation (to the best of our knowledge) has been performed to ensure that increasing the number of said iterations does not lead to instability due to a lack of convergence. In addition, iterative mean-field estimation is not the only way to estimate marginal probability distributions and relies on approximations that can lead to lower performance compared to other algorithms, such as loopy belief propagation which has long been seen as the \textit{de facto} algorithm for marginal probability estimation since its inception in the late 1980's \cite{pearl1987distributed}. These lead to two distinct areas of investigation for improving the use of simple regularisation terms and CRFs alongside deep learning architectures.

\subsection{Contributions}

The paper presents two iterative marginal probability estimation frameworks: one based on parallel loopy belief propagation and the other on iterative mean-field estimation. These two algorithms have been chosen as they can be implemented in a highly parallelisable form in which all information is updated for the entire graph simultaneously, rather than in a particular algorithmically defined order. These parallel algorithms are particularly useful for modern GPGPU computing as they are inherently parallel.

This article will start with a theoretical section giving specific, easily verifiable constraints on the regularisation that guarantee convergence and differentiability. It will then show the relative performance of these two approaches for marginal probability estimation on synthetic Ising and Potts models. The last section concerns the application of these regularisation frameworks in larger artificial neural network models under various loss functions.

Code for all the solvers as well as the experiments is provided open-source at \url{https://github.com/JSHBaxter/DeepFlow/}. This includes the different versions of the marginal probability estimators (mean-field iteration and loopy belief propagation), the max-flow-based \textit{maximum a posteriori} configuration solver, and the corresponding negative log-likelihood functions. For models, it includes the binary model (which degenerates to sigmoid if no interaction terms are used), the Potts model (which degenerates to softmax) and the hierarchical version of the Potts model which allows for more complex regularisation structures. The code also includes the files for all three synthetic experiments.

%% file: 2_theory.tex
\section{Convergent and differentiable marginal probability estimation}
Although the paper presents three classes of model, a binary model, a Potts model, and a hierarchical model, we shall generally focus on the latter, being the most general. We  over nodes $x$ and edges $e$ subject to:
\begin{equation*}
    u_L(x) \in \{0,1\} \mbox{, } u_P(x) = \mbox{\hspace{-3mm}}\sum_{L\in P.C^*\cap\mathcal{L}}\mbox{\hspace{-3mm}}u_L(x) \text{, and } u_S(x)=1
\end{equation*}
and the energy equation formulated as:
\begin{equation*}
    E_{[\psi,u]} \hspace{-0.5mm}=\hspace{-0.5mm} \sum_L \hspace{-1mm}\left(\sum_x \hspace{-0.6mm}\psi_L(x) u_L(x) - \hspace{-3mm}\sum_{e=(x,y)}\hspace{-2.6mm} \psi_L(e)|u_L(x)-u_L(y)| \right) \label{eq:model}
\end{equation*}
noting the sign of the regularisation term.
Each parameter term $\psi$ is non-negative, although said condition is only a limitation for the interaction terms $\psi_L(e)$. An example of a hierarchy is given below. These hierarchies are composed of three types of labels: `leaf'-labels (the set of which we call $\mathcal{L}$) together form a partition of the image and are all mutually disjoint; `branch'-labels are composed of other labels, grouping them together based on their spatial organisation, i.e. if the branch label is a geometrically simpler object; and the `source'-label, $S$, represents the entire image (so $u_S(x) = 1$ for all $x$) and thus does not contribute to the interaction terms in Eq. \ref{eq:model}. With the exception of the source label, each other label, $L$ has a single parent label, denoted as $L.P$, which allows the labels to be organised as a directed acyclic graph such as a directed tree, with $S$ as the root. An example of which is given in Figure \ref{fig:my_label}. We use the Kleene star $L.P^*$ to denote the set $\{L,L.P,L.P.P, \ldots, S\}$, i.e. the inclusive set of all ancestors, and well as a similar set $L.C^* = \{L' | L \in L'.P^*\}$ i.e. the inclusive set of all descendants. Certain families of configurations have simplified forms, such as a Potts model which has no branch-labels, meaning that the only interaction terms are associated with the leaf-labels. This makes it analogous to a spatially regularised soft-max function.

\begin{figure}[ht]
    \centering
    \includegraphics[width=0.47575\textwidth]{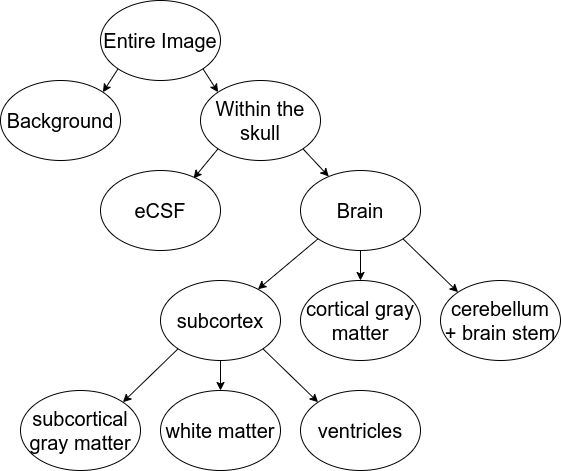}
    \caption{Example hierarchy used in the brain tissue segmentation experiment. The leaf-labels include the background, external cerebral spinal fluid, cortical gray matter, subcortical gray matter, white matter, ventricles, and a combined label for the cerebellum and brain stem. The additional branch-labels are the within-skull label, the brain, and the sub-cortex.}
    \label{fig:my_label}
\end{figure}

As stated in the previous section, the most common method for addressing these spatially regularised activation functions has been via iterative mean-field estimation \cite{zheng2015conditional,arnab2016higher,xu2020fast,guan2024neural} which estimates a marginal by taking a current estimate and updating it to better reflect the various terms involved in the model. The benefit of such updates for MFI is that they are mathematically simple and thus can easily allow for automatic differentiation. Loopy belief propagation involves more complex update steps but also can be automatically differentiated. However, the issue with automatic differentiation is that a certain amount of state must be preserved of all iterations. This means that, in training time, the models experience linear memory growth with respect to the number of iterations. Most methods address this by using a relatively small number of iterations while privileging models that have denser connectivity and thus more long-range regularisation, such as fully-connected CRFs \cite{zheng2015conditional} but these require parameterised models of the interaction terms and thus lack the flexibility of sparser, more local regularisation terms. That being said, one of our goals is to provide a rigorous framework for any pairwise CRF with categorical labels and regularising interaction terms.

Regardless of type or degree of connectivity, these algorithms are fundamentally iterative, they can be interpreted as trying to find the fixed point of some function, $f$ that reflects a constraint fulfilled only by (an approximation of) the marginal distribution. The iterative nature of these functions immediately raises questions such as:
\begin{itemize}
    \item Does the forward pass converge?
    \item Does the derivative of an infinite forward pass exist? and
    \item Can we approximate the true derivative using finite memory with respect to the number of iterations?
\end{itemize}
The last point is of particular technical importance as we will want to limit ourselves to algorithms that only require information about the input to the model (i.e. $\psi$) and the output of the model (i.e. the converged marginal estimates) without saving any information about the intermediate steps. Answering these questions in the affirmative leads to several important properties, notably that one can perform gradient backpropagation over them regardless of what comes before or after them in a network. This is particularly beneficial specifically for medical image segmentation in which multiple different losses have been proposed and are often used in combination \cite{ma2021loss}. This also means they can be used in models with adversarial losses \cite{pan2020loss} where there are some other components (i.e. a discriminator) between the marginal estimates and the specific loss functions.

This is the key difference between the algorithms presented in this paper and similar previous loss-agnostic works. Previous works such as Zheng \textit{et al.}'s \cite{zheng2015conditional} interpretation of finding the MAP probability via an RNN-like sequence of operators, Xu \textit{et al.}'s \cite{xu2020fast} differentiable message passing, or Guan \textit{et al.}'s \cite{guan2024neural} neural MRF all accumulate information during each iteration of the forward pass necessary to calculate an exact derivative during the backwards pass. This means that their methods are largely limited to relatively small numbers of iterations (five in the case of Zheng \textit{et al.} \cite{zheng2015conditional}, ten for Kuck \cite{kuck2020belief}, twenty for Satorras \& Welling \cite{satorras2021neural}, fifty for Xu \textit{et al.} \cite{xu2020fast} and ``a fixed number'' for Guan \textit{et al.} \cite{guan2024neural}) which is problematic for models that require a large number of iterations. In these cases, either the number of iterations must be increased which could cause issues with excessive memory consumption, or the results would be inaccurate and not take into account longer-range effects. This limits the application of these models to problems such as volumetric medical image processing in which both range and memory consumption are of large concern. To give an example, a standard 256-by-256-by-256 volumetric MR image would require at least 767 iterations to allow for information from one edge of the volume to reach the other edge, and saving this number of fully volumetric images is well outside the capability of modern GPGPU. (If we assume that only a single 16-bit float is saved for each voxel for computing the backwards pass, this would require over 191 Gb per volume, just to store the required information for computing the derivative for the backwards pass. And this is not considered to be a particularly highly-resolved structural MR image.)

Determining the answer to our three questions above encourages us to consider any initialization, especially as the random weights at earlier training stages. For this most general case, we rely on the following theorem to define constraints on parameterizing these models.

\vspace{5mm}

\begin{theorem}\label{the:main_conv}
Consider a function $f(q,\psi): Q \times \Psi \to Q$ where $Q$ and $\Psi$ are subspaces of $\mathbb{R}^{N_q}$ and $\mathbb{R}^{N_\psi}$ which is $C^1$ smooth with respect to $q$. Let $D_y f$ be the derivative of $f$ with respect to the arguments $y$ (either $q$ or $\psi$). If there exists a bound $b < 1$ such that the spectral norm $\rho(D_q f) \leq b$ then iterating $f(q,\psi)$ (while keeping $\psi$ constant) converges linearly to a unique fixed point.
\end{theorem}

\vspace{5mm}

\noindent \underline{Proof:} 
For the first point, we will show that the function is contractive with respect to $q$. Under the Banach fixed-point theorem, this means that iterating $f$ is convergent.
Since $\rho(A) = \underset{|\cdot|}{\operatorname{inf}}|A|$ where $|\cdot|$ are consistent matrix norms then $\rho(D_qf) \leq b$ implies the existence of a vector norm that induces a matrix norm where $|D_qf|\leq b'<1$. Using the fundamental theorem of calculus for line integrals for each component of $f$ simultaneously, we get:
\begin{equation*}
\begin{aligned}
\left|f(q_0,\psi)-f(q_1,\psi)\right|
=& \left| \int_C (D_q f) dq \right| \mbox{ for any curve } C \mbox{ connecting } q_0 \mbox{ and } q_1\\
=& \left| \int_0^1 (D_q f) \frac{d\left(tq_0+(1-t)q_1\right)}{dt} dt \right| \mbox{ using } C \mbox{ as a straight line}\\
\leq& \int_0^1 \left| (D_q f) \left(q_0-q_1\right) \right| dt \mbox{ due to subaddititivity}\\
\leq& \int_0^1 \left|D_qf\right|  \left|q_0-q_1\right| dt \mbox{ due to the norm being consistent} \\
\leq& b' \left|q_j^0-q_j^1\right| \\
\end{aligned}
\end{equation*}
Thus showing that $f$ is contractive with respect to the used vector norm. Therefore, iterating $f$ converges to a unique fixed value linearly with respect again to said vector norm.
\hfill$\blacksquare$

\vspace{5mm}

A corollary of this is that if the spectral norm $\rho(D_q f)\leq b < 1$, then the theorem still applies. Note that this is a reformulation of the proofs given in \cite{mooij2012sufficient} Section 3.1 and 3.2 although not specific to pairwise binary models and LBP. Later we can rederive slightly stronger bounds for these models as well as explicit bounds for others.

\vspace{5mm}

\begin{theorem}\label{the:main_diff}
Consider a function $f(q,\psi): Q \times \Psi \to Q$ where $X$ and $\Psi$ are subspaces of $\mathbb{R}^{N_q}$ and $\mathbb{R}^{N_\psi}$ which is $C^1$ smooth with respect to $q$ and $\psi$ with bounded derivatives. If iterating $f$ converges linearly to $q_{[\infty]}$ (with respect to some vector norm, $M$) for all $q$ in an open set that includes $q_\infty$, then the derivative of the fixed-point with respect to $\psi$ exists and said derivative can be computed using only $\psi$ and $q_{[\infty]}$.
\end{theorem}

\vspace{5mm}

We will construct the final derivative and show that the sequence of derivatives with respect to $\psi$ for a growing number of iterations converges towards it. For the sake of notation, let us define the iterative process as $q_{[n+1,q_0]} = f(q_{[n,q_o]},\psi)$ with $q_0$ as any arbitrary starting point in the region of convergence around $q_{[\infty]}$, i.e. $q_{[0,q_0]}=q_0$. We can also define matrices $A_{[n]} = \left.D_\psi f \right|_{q=q_{[n,q_0]}}$, $B_{[n]} = \left.D_q f\right|_{q=q_{[n,q_0]}}$. Because $q_{[\infty]}$ must be in the interior of this open set, the function $f$ is $C^1$ smooth (i.e. it has continuous first derivatives), and the sequence $q_{[n,q_0]}$ is convergent, the sequences $A_{[n]}$ and $B_{[n]}$ must also be convergent towards those from the analysis of $f$ at $q_{[\infty]}$. That is, $\lim_{n\to\infty} A_{[n]} = \left.D_\psi f\right|_{q=q_{[\infty]}}$ and $\lim_{n\to\infty} B_{[n]} = \left.D_q f\right|_{q=q_{[\infty]}}$, giving us an unambiguous definition of $A_{[\infty]}$ and $B_{[\infty]}$.

We can now use the linear convergence property to bound the matrix $B_{[\infty]}$. Consider a point $q_{[\infty]}+\epsilon v$ with unit direction vector $v$ and $\epsilon > 0$ that is sufficiently small such that $q_{[\infty]}+\epsilon v$ is in the region of convergence around $q_{[\infty]}$. We can use this to define the norm of the directional derivative of $f$ with respect to norm $M$:
$$
    |B_{[\infty]}v|
    = \lim_{\epsilon\to 0} \frac{|f(q_{[\infty]}+\epsilon v,\psi)-f(q_{[\infty]},\psi)|}{|\epsilon|} \leq \lim_{\epsilon\to 0} \frac{\beta |\epsilon v|}{|\epsilon|} = b |v| = b < 1
$$
given the convergence parameter $b < 1$ with respect to norm $M$ used for defining the $|\cdot|$ operator. This means that there is an induced norm of $B_{[\infty]}$ strictly less than one; thus the spectral radius of $B_{[\infty]}$ must also be strictly less than 1 and the eigenvalues of $B_{[\infty]}$ must have magnitudes less than 1. Therefore, we can define $C = (I-B_{[\infty]})^{-1}A_{[\infty]}$ as our constructed derivative noting that $I-B_{[\infty]}$ is invertible since $B_{[\infty]}$'s eigenvalues are all sub-unit. This matrix, $C$, also has the property that:
\begin{equation*}
\begin{aligned}
	A_{[\infty]}+B_{[\infty]}C 
	=& A_{[\infty]}+B_{[\infty]}(I-B_{[\infty]})^{-1}A_{[\infty]} \\
	=& \left(I+B_{[\infty]}(I-B_{[\infty]})^{-1}\right)(I-B_{[\infty]})(I-B_{[\infty]})^{-1}A_{[\infty]}\\
	=& (I-B_{[\infty]})^{-1}A_{[\infty]} = C
\end{aligned}
\end{equation*}
We must now show that the sequence of actual derivatives approaches $C$. For a finite number of iterations, we can express these derivatives with respect to $\psi$ recursively as:
\begin{equation*}
D_\psi q_{[n,q_0]} = A_{[n]} + B_{[n]}D_\psi q_{[n-1,q_0]} \mbox{ terminating the recursion with } D_\psi q_{[0,q_0]}=0
\end{equation*}
Define $\epsilon_{A_{[n]}}$, $\epsilon_{B_{[n]}}$, and $\epsilon_{C_{[n]}}$ as $\epsilon_{A_{[n]}} = A_{[n]}-A_{[\infty]}$, $\epsilon_{B_{[n]}}=B_{[n]}-B_{[\infty]}$, and $\epsilon_{C_{[n]}} = D_\psi q_{[n,q_0]} - C$, respectively. Note:
\begin{equation*}
\begin{aligned}
	C + \epsilon_{C_{[n]}}
	&= A_{[\infty]}+\epsilon_{A_{[n]}} + \left(B_{[\infty]}+\epsilon_{B_{[n]}}\right)\left(C+\epsilon_{C_{[n-1]}} \right) \\
	&= A_{[\infty]} + B_{[\infty]} C + \epsilon_{B_{[n]}} C + \epsilon_{A_{[n]}} + \epsilon_{B_{[n]}}\epsilon_{C_{[n-1]}}\\
	\epsilon_{C_{[n]}}
	&= \epsilon_{B_{[n]}} C + \epsilon_{A_{[n]}} + \epsilon_{B_{[n]}}\epsilon_{C_{[n-1]}}\\
\end{aligned}
\end{equation*}
Given the subadditivity/submultiplicity of matrix norms, the fact that $|\epsilon_{A_{[n]}}|$ and $|\epsilon_{B_{[n]}}|$ are bounded, and that $\left|\epsilon_{A_{[n]}}\right| \to 0$ and $\left|\epsilon_{B_{[n]}}\right| \to 0$ as  $n\to\infty$ due to $f$ being $C_1$-smooth:
\begin{equation*}
\begin{aligned}
	\left|\epsilon_{C_{[n]}} - \epsilon_{B_{[n]}}\epsilon_{C_{[n-1]}}\right|
	&= \left|\epsilon_{B_{[n]}} C + \epsilon_{A_{[n]}}\right| \\
	\left|\epsilon_{C_{[n]}} - \epsilon_{B_{[n]}}\epsilon_{C_{[n-1]}}\right|
	&\to 0 \\
\end{aligned}
\end{equation*}
Note that $\nicefrac{\left|\epsilon_{C_{[n]}}\right|}{\left|\epsilon_{C_{[n-1]}}\right|} < \left|\epsilon_{B_{[n]}}\right|$ and $\left|\epsilon_{B_{[n]}}\right| \to 0 < 1$. By the ratio test, $\left|\epsilon_{C_{[n]}}\right| \to 0$ and thus $D_\psi q_{[n,q_0]} \to C$ as $n\to\infty$.
\hfill$\blacksquare$

\vspace{5mm}

Theorem \ref{the:main_conv} as well as the constraints developed in \cite{mooij2012sufficient} based on spectral radii necessarily imply Theorem \ref{the:main_diff} which together give a test the convergence of the forward pass as well as a specific formula for the backwards pass. (For certain matrix nroms, Theorem \ref{the:main_conv} provides a purely local test which can be computed in linear time or even faster with parallel computing.) This formula is theoretically useful but practically difficult given that such a matrix could potentially have tens of millions of rows and columns which means that, even if it is sparse, inverting it would be next to impossible. In addition, we do not need the full matrix, but merely its product with a particular vector (the gradient of the final loss function with respect to its output) meaning that it is ultimately faster to approximate it as:
$$Cy = A_{[\infty]}y + A_{[\infty]}B_{[\infty]}y + A_{[\infty]}B_{[\infty]}^2y + A_{[\infty]}B_{[\infty]}^3y...$$ which is much easier especially given that $A_{[\infty]}$ and $B_{[\infty]}$ are likely sparse. There are multiple potential algorithms exist that could fall within this framework, two of which will be investigated in the following sections.

The particular problem with Theorem \ref{the:main_diff} is being able to tell if a point of convergence is surrounded by a region of linear convergence. One potential way of doing so would be to check that the spectral radius of $\left.D_qf\right|_{q=q_{[\infty]}}$ (or any induced matrix norm) is less than 1 which is already effectively a requirement to make the above series definitely converge. This would mean that the derivative could be calculated for the point-of-convergence but does not ensure that said point is unique.

\subsection{Parallel Loopy Belief Propagation}
Loopy belief propagation (LBP) is a form of message passing algorithm in which each node sends a message to each adjacent edge containing information about its beliefs about its own probability distribution and the edges reply with information about the node's neighbours, each allowing the other to intelligently update themselves.

Defining $p_{(L_i,L_j,(x,y))} = \prod_{P \in L_i.P^* \triangle L_j.P^*}e^{-\psi_P((x,y))}$ where $\triangle$ is the \textit{symmetric difference} or \textit{exclusive union} operator (i.e. $A\triangle B$ is the set of elements that appear in either $A$ or $B$ but not both), LBP for our pairwise conditional random field can be expressed as the following set of equations:
\begin{equation*}
\begin{aligned}
r(x\to(x,y),L) & \gets e^{\psi_L(x)} \prod_{z \in \mathcal{N}(x)/\{y\}} r((x,z)\to x,L) \mbox{ for messages from node } x \mbox{ to edge } (x,y)\\
r((x,y)\to x,L) & \gets \sum_{L'} p_{(L,L',(x,y))} r(y \to (x,y),L') \mbox{ for messages from edge } (x,y) \mbox{ to node } x \\
m(x,L) & \propto e^{\psi_L(x)} \prod_{y \in \mathcal{N}(x)} r((x,y)\to x,L)
\end{aligned}
\end{equation*}
In practice, these messages can become arbitrarily large requiring them to be normalised and making it difficult to ensure convergence. In order to address this, we can apply a normalising term to the messages coming from the nodes towards the edges at each iteration ensuring that they sum to 1. Also, in order to make the analysis a bit easier, the messages from the edges back to the nodes will be expressed on a logarithmic sclae. Overall, this gives us the equations:
\begin{equation}
\begin{aligned}
r_{[n,r_0]}(x\to(x,y),L) & = \underset{L}{\mbox{softmax}}\left\{\psi_L(x) + \sum_{z \in \mathcal{N}(x)/\{y\}} R_{[n-1,r_0]}((x,z)\to x,L)\right\}\\
R_{[n,r_0]}((x,y)\to x,L) & = \ln \sum_{L'} r_{[n,r_0]}(y \to (x,y),L')p_{(L,L',x,y)} 
\\
m_{[n,r_0]}(x\to(x,y),L) & = \underset{L}{\mbox{softmax}}\left\{\psi_L(x) + \sum_{y \in \mathcal{N}(x)} R_{[n,r_0]}((x,y)\to x,L)\right\}\\
\end{aligned}
\end{equation}

For these equations, it is natural to consider this algorithm to be an iterative one over the messages themselves rather than the marginal probability estimates. This is similar to the transformation used in \cite{mooij2012sufficient} (for binary models) except generalised to any number of labels as well as including the normalisation terms explicitly. (The latter is beneficial for ensuring that all messages do not simply converge towards 0 in theory, which would give unstable results in practice.)

As per Theorem \ref{the:main_conv}, we must show that the derivatives with respect to $\psi$ exist for a single iteration (which is trivial) and we must bound the derivatives with respect to the messages. Note that the derivatives are:
\begin{equation*}
\begin{aligned}
    \frac{\delta r_{[n,r_0]}(x\to(x,y),L_i)}{\delta R_{[n-1,r_0]}((x,z)\to x,L_j)} & = \left\{\begin{array}{ll}
         0 & \mbox{if } z \not\in \mathcal{N}(x)/\{y\} \\
         r_{[n,r_0]}(x\to(x,y),L_i)\left(1-r_{[n,r_0]}(x\to(x,y),L_i)\right) & \mbox{if } L1=L2 \mbox{ and }z \in \mathcal{N}(x)/\{y\} \\
         -r_{[n,r_0]}(x\to(x,y),L_i)r_{[n,r_0]}(x\to(x,y),L_j) & \mbox{if } L1 \neq L2 \mbox{ and }z \in \mathcal{N}(x)/\{y\}
    \end{array} \right.\\
    \frac{\delta R_{[n,r_0]}((x,y)\to x,L_i)}{\delta r_{[n,r_0]}(y\to(x,y),L_j)} 
    & = \frac{p_{(L_i,L_j,x,y)}}{\sum_{L_k}r_{[n,r_0]}(y\to(x,y),L_k)p_{(L_i,L_k,x,y)}}
\end{aligned}
\end{equation*}
Combining these together yield all the non-zero derivatives across an entire iteration with respect to $R$ as:
\begin{equation*}
\begin{aligned}
    \frac{\delta R_{[n,r_0]}( (x,y)\to x,L_i)}{\delta R_{[n-1,r_0]}((y,z) \to y,L_j)} 
    & =  \sum_{L_k} \frac{\delta R_{[n,r_0]}((x,y) \to x,L_i)}{\delta r_{[n,r_0]}(y \to (x,y),L_k)} \frac{\delta r_{[n,r_0]}(y\to (x,y),L_k)}{\delta R_{[n-1,r_0]}((y,z)\to y,L_j)} \\
    & = \sum_{L_k}
    \frac{p_{(L_i,L_k,x,y)}}{\sum_{L_l}r_{[n,r_0]}(y\to (x,y)),L_l)p(L_i,L_l,x,y)} \frac{\delta r_{[n,r_0]}(y\to (x,y),L_k)}{\delta R_{[n-1,r_0]}((y,z)\to y,L_j)} \\
    & = r_{[n,r_0]}(y\to (x,y),L_j)
    \left(\frac{p_{(L_i,L_j,x,y)}-\sum_{L_k}r_{[n,r_0]}(y\to (x,y),L_k)p_{(L_i,L_k,x,y)}}{\sum_{L_k}r_{[n,r_0]}(y\to (x,y),L_k)p_{(L_i,L_k,x,y)}} \right)
\end{aligned}
\end{equation*}

Consider the bounds places on $\Psi$ to that the $p_1$ induced matrix norm:
\begin{equation}\label{eq:lbp_e1}
\begin{aligned}
    &\sum_{z,L_j} \left|\frac{\delta R_{[n,r_0]}( (x,y)\to x,L_i)}{\delta R_{[n-1,r_0]}((y,z) \to y,L_j)}\right| \\
    = & \sum_{z \in \mathcal{N}(y)/\{x\},L_j} r_{[n,r_0]}(y\to (x,y),L_j)\left|\frac{p_{(L_i,L_j,x,y)}-\sum_{L_k}r_{[n,r_0]}(y\to (x,y)),L_k)p_{(L_i,L_k,x,y)}}{\sum_{L_k}r_{[n,r_0]}(y\to (x,y)),L_k)p_{(L_i,L_k,x,y)}} \right| \\
    \leq & \sum_{z \in \mathcal{N}(y)/\{x\}} \frac{2 (1-\sqrt{p_m})^2}{1-p_m}
    \\ & \mbox{\hspace{2.35cm}where } p_m = \underset{L_k}{\min}\; p_{(L_i,L_k,x,y)}\mbox{\hspace{0.5cm}(see Appendix \ref{app:lbp_e1} for derivation)} \\
= & 2(|\mathcal{N}(y)|-1)\frac{(1-\sqrt{p_m})^2}{1-p_m}
\end{aligned}
\end{equation}
Setting this to be strictly less than 1 for all pairs of $y$ and $L_i$ ensures the requirement of Theorem \ref{the:main_conv}.
\begin{equation*}
\begin{aligned}
    \sqrt{p_m} & > \frac{2|\mathcal{N}(y)|-3}{2|\mathcal{N}(y)|-1} \mbox{ or } \max_{L_i,L_j} \sum_{P \in L_i.P* \triangle L_j.P*} \psi_P((x,y)) < 2 \ln \left(\frac{2|\mathcal{N}(y)|-1}{2|\mathcal{N}(y)|-3}\right)
\end{aligned}
\end{equation*}
If this constraint holds for all edges, we have a bound which ensures convergence for the logarithmically scaled edge-to-node messages which means that the bounded node-to-edge messages must also converge as well. 

This is a weaker bound than those presented in equations (15) and (16) of \cite{mooij2012sufficient} as it explicitly prohibits the values of the messages $R$ from both trending towards positive or negative infinity at the same time which is not needed in their formulation which couples these messages together effectively only measuring their difference. However, our formulation does apply to non-binary problems where their formulation does not.

If we try to formulate an equivalent to Theorem \ref{the:main_conv} using the $p_\infty$ induced matrix norm:
\begin{equation}\label{eq:lbp_e2}
\begin{aligned}
    &\sum_{x,L_i} \left|\frac{\delta R_{[n,r_0]}( (x,y)\to x,L_i)}{\delta R_{[n-1,r_0]}((y,z) \to y,L_j)}\right| \\
    = & \sum_{x \in \mathcal{N}(y)/\{z\},L_i} r_{[n,r_0]}(y\to (x,y),L_j)\left|
    \frac{p(L_i,L_j,x,y) - \sum_{L_k} r_{[n,r_0]}(y\to (x,y),L_k) p_{(L_i,L_k,x,y)}}{\sum_{L_k}r_{[n,r_0]}(y\to (x,y),L_k)p_{(L_i,L_k,x,y)}} \right| \\
    = & \sum_{x \in \mathcal{N}(y)/\{z\}} r_{[n,r_0]}(y\to (x,y),L_j) \sum_{L_i} \left|
    \frac{p(L_i,L_j,x,y) - \sum_{L_k} r_{[n,r_0]}(y\to (x,y),L_k) p_{(L_i,L_k,x,y)}}{\sum_{L_k}r_{[n,r_0]}(y\to (x,y),L_k)p_{(L_i,L_k,x,y)}} \right|
\end{aligned}
\end{equation}

At this stage, bounding the terms becomes difficult without additional assumptions. However, we can bound these terms specifically for binary models (See Appendix \ref{app:lbp_e2}) which yields:
\begin{equation}
    \sum_{x,L_i} \left|\frac{\delta R_{[n,r_0]}( (x,y)\to x,L_i)}{\delta R_{[n-1,r_0]}((y,z) \to y,L_j)}\right|
    \leq \sum_{w \in \mathcal{N}(y)/\{z\}} \frac{1-p_{(L_1,L_2,w,y)}}{1+p_{(L_0,L_1,w,y)}} \mbox{ for the two labels $L_0$ and $L_ 1$} 
\end{equation}
which means that for binary models, this constraint can be written as:
\begin{equation}
\forall z, \forall y \in \mathcal{N}(z), \sum_{w\in \mathcal{N}(y)/\{z\} } \frac{1-p_{(L_0,L_1,w,y)}}{1+p_{(L_0,L_1,w,y)}} < 1
\label{eq:plbp_const_infty}
\end{equation}
which indicates that parallel LBP must converge for binary models structured as 1D chains or loops.

This constraint happens to be exactly equivalent to that given by Equation (15) of \cite{mooij2012sufficient} despite using a different norm and not having a pairing mechanism for messages, indicating that the logarithmic domain with normalisation built-in to the node-to-edge messages is equivalent their tanh-based domain. However, it is much more difficult to apply this formulation outside of binary models due to a lack of an analytic optimal configuration. As with the formulas from \cite{mooij2012sufficient}, the constraint based on the $p_\infty$ norm is less restrictive than the $p_1$-based constraint at least in the binary case.

\subsection{Mean-Field Iteration}\label{sec:meanfield}
As shown in the previous section, LBP requires additional buffers to hold messages passed between nodes and factors. This means that for each edge connecting two variables in our graph, at least four buffers must be saved, two for the node's messages towards the edge, and two for the reply from the edge back to the nodes. Unlike LBP, the idea behind mean-field iteration (MFI) is that instead of sending messages, each node should be made aware of the (estimated) mean values of its surroundings and use them to directly update itself accordingly. This in turn requires far less memory, although it may not be as accurate as LBP.

For the sake of simplicity, we will only consider $\psi_L(x)$ to be non-zero when $L \in \mathcal{L}$. This iteration is based on the calculation of the energy for the marginal probability distribution using the mean-field estimate, that is, that the values of $u$ can be replaced their marginal estimates, $m$, for the purpose of estimating other variables:
\begin{equation}
    \tilde{E}(x,L) = \psi_L(x) + \sum_{P \in L.P^*} \hspace{-1mm}\left( \sum_{e=(x,y)} \psi_P(e) \left(2m(y,P) - 1\right) \right)
    \label{eq:meanenergy}
\end{equation}
where $m$ follows the same constraints as $u$ except that it is continuous rather than discrete. That is, $m(x,L) \in [0,1]$, $m(x,P) =\sum_{L\in P.C}m(x,L)$, and $m(x,S)=1$. (A derivation can be found in Appendix \ref{app:mf_deriv}.) The marginal estimates can then be updated according to this energy:
\begin{equation}
    m(x,L) \gets \frac{\mbox{exp}\left(\tilde{E}(x,L)\right)}{\underset{L' \in \mathcal{L}}{\sum} \mbox{exp}\left(\tilde{E}(x,L')\right)}
    \label{eq:meanupdate}
\end{equation}

\subsubsection{Derivatives and constraints}\label{sec:constraint}
For the sake of notation, we will write each iteration of the forward pass as $m_{[n,m_0]}(x,L)$ where $m_0$ represents the initial estimate of the marginal probability before iteration and $n$ is the number of iterations performed. Similarly define $\tilde{E}_{[n,m_0]} = \tilde{E}_{[\psi,m_{[n,m_0]}]}$. This has the following recursive definition:
\begin{equation}
\begin{aligned}
    m_{[0,m_0]}(x,L) & = m_0(x,L) \\
    m_{[n+1,m_0]}(x,L) & = \frac{\mbox{exp}\left(\tilde{E}_{[n,m_0]}(x,L)\right)}{\underset{L' \in \mathcal{L}}{\sum} \mbox{exp}\left(\tilde{E}_{[n,m_0]}(x,L')\right)}
\end{aligned}\label{ref:forward}
\end{equation}

Here we have a choice between considering this as an iteration over the marginal probability estimates, $m$, or the energy of said estimates, $E$. Given that they are continuous functions of each other, if one converges the other must as well and if we have the derivative of one with respect to $\psi$ we can easily compute the derivative of the other with respect to $\psi$. We will continue viewing this as an iteration over the energy estimates for the leaf labels only.

It is beneficial to clarify the derivatives for a single forward iteration:
\begin{equation}
\begin{aligned}
    \frac{\delta \tilde{E} (x_1,L_i)}{\delta \psi_{L_k}(x_2)} & = \left\{ \begin{array}{ll} 1 & \text{ if } L_i = L_k \mbox{, } L_k \in \mathcal{L} \mbox{ and } x_1=x_2 \\ 0 & \text{ else} \end{array} \right. \\
    \frac{\delta \tilde{E}(x_1,L_1)}{\delta \psi_{L_2}(e)} & = \left\{ \begin{array}{ll}
     \underset{L_3 \in L_2.C^*}{\sum}(2m_{[n,m_0]}(y,L_3)-1) & \mbox{ if } L_1 \in L_2.C^* \mbox{ and } e=(x,y) \\ 0 & \mbox{ else} \end{array} \right. \\
\end{aligned}\label{eq:ders_base}
\end{equation}
Noting that
\begin{equation}
\begin{aligned}
    \frac{\delta \tilde{E}_{[n+1,m_0]}(x_1,L_i)}{\delta m_{[n+1,m_0]}(x_2,L_k)} & 
    = \left\{ \begin{array}{ll} \sum_{P \in L_i.P^* \cap L_k.P^*}2\psi_P(e) & \text{ if } \exists e=(x_1,x_2) \\ 0 & \text{ else} \end{array} \right.
\end{aligned}
\end{equation}
and
\begin{equation}
\begin{aligned}
    \frac{\delta m_{[n+1,m_0]}(x,L_k)}{\delta \tilde{E}_{[n,m_0]}(x,L_j)} & = \left\{ \begin{array}{ll} m_{[n+1,m_0]}(x,L_j)\left(1-m_{[n+1,m_0]}(x,L_j)\right) & \text{ if } L_k=L_j \\ -m_{[n+1,m_0]}(x,L_j) m_{[n+1,m_0]}(x,L_k) & \text{ else }\end{array} \right.
\end{aligned}\label{eq:const_der}
\end{equation}
we can calculate the derivative across an entire iteration as:
\begin{equation*}
\begin{aligned}
& \frac{\delta \tilde{E}_{[n,m_0]}(x,L_i)}{\delta \tilde{E}_{[n-1,m_0]}(x',L_j)} = \sum_{L_k}\frac{\delta \tilde{E}_{[n,m_0]}(x,L_i)}{\delta m_{[n,m_0]}(x',L_k)} \frac{\delta m_{[n-1,m_0]}(x',L_k)}{\delta \tilde{E}_{[n-1,m_0]}(x',L_j)} \\
    & =\sum_{L_k} \left(\left\{\begin{array}{ll}\sum_{P \in L_i.P^* \cap L_k.P^*}2\psi_P(e) & \mbox{if }\exists e=(x,x') \\ 0 & \mbox{else}\end{array} \right. \right) \left(\left\{ \begin{array}{ll} m_{[n-1,m_0]}(x',L_j)(1-m_{[n-1,m_0]}(x',L_j)) & \text{if }L_j=L_k \\ - m_{[n-1,m_0]}(x',L_j)m_{[n-1,m_0]}(x',L_k) & \text{else}\end{array}\right.\right) \\
    & = 2 m_{[n-1,m_0]}(x',L_j)\sum_{L_k}\left(\sum_{P \in L_i.P^* \cap L_k.P^*}\psi_P(e)\right)\left\{ \begin{array}{ll}(1-m_{[n-1,m_0]}(x',L_k)) & \text{ if }L_j=L_k \\ -m_{[n-1,m_0]}(x',L_k) & \text{ if }L_j\neq L_k \end{array}\right. \\
    & = 2 m_{[n-1,m_0]}(x',L_j)\left(\left(\sum_{P \in L_i.P^* \cap L_j.P^*}\psi_P(e)\right) - \sum_{L_k}\left( m_{[n-1,m_0]}(x',L_k) \sum_{P \in L_i.P^* \cap L_k.P^*}\psi_P(e) \right) \right) \\
\end{aligned}
\end{equation*}

We can now find a bound on $\Psi$ that is equivalent to that on the $p_1$ induced matrix norm as:
\begin{equation}\label{eq:mfi_eq}
\begin{aligned}
& \sum_{x',L_j}\left|\frac{\delta \tilde{E}_{[n,m_0]}(x,L_i)}{\delta \tilde{E}_{[n-1,m_0]}(x',L_j)}\right| \\
    & = \sum_{e=(x,x')}\sum_{L_j} 2m_{[n-1,m_0]}(x',L_j)\left| \left(\sum_{P \in L_i.P^* \cap L_j.P^*}\psi_P(e)\right) - \sum_{L_k}\left( m_{[n-1,m_0]}(x',L_k) \sum_{P \in L_i.P^* \cap L_k.P^*}\psi_P(e) \right) \right| \\
    & \leq \sum_{e=(x,x')} \left(\sum_{P \in L_i.P^*}\psi_P(e) - \underset{L_j}{\min}\left(\sum_{P \in L_i.P^* \cap L_j.P^*}\psi_P(e)\right) \right) \mbox{\hspace{10mm}\tiny See Appendix \ref{app:mfi_eq} for derivation.} \\
    & = \sum_{e=(x,x')} \underset{L_j}{\max}\left(\sum_{P \in L_i.P^* / L_j.P^*}\psi_P(e)\right)
\end{aligned}
\end{equation}
In the case where the label ordering forms a tree with no redundant labels, then this formula can be simplified as 
$$\underset{L_j}{\max}\left(\sum_{P \in L_i.P^* / L_j.P^*}\psi_P(e)\right) = \sum_{P \in L_i.P^*}\psi_P(e)$$
since a label can always be found that shares no parents (other than the root label) in common. This covers the standard binary and Potts models.

Alternatively, we can find a bound equivalent to Theorem \ref{the:main_conv} using the $p_\infty$ induced matrix norm as:
\begin{equation}\label{eq:mfi_eq2}
\begin{aligned}
& \sum_{x',L_j}\left|\frac{\delta \tilde{E}_{[n,m_0]}(x',L_j)}{\delta \tilde{E}_{[n-1,m_0]}(x,L_i)}\right| \\
    & = \sum_{e=(x,x')}\sum_{L_j}2m_{[n-1,m_0]}(x,L_i)\left| \left(\sum_{P \in L_i.P^* \cap L_j.P^*}\psi_P(e)\right) - \sum_{L_k}\left( m_{[n-1,m_0]}(x,L_k) \sum_{P \in L_j.P^* \cap L_k.P^*}\psi_P(e) \right) \right| \\
    & = 2\sum_{e=(x,x')} m_{[n-1,m_0]}(x,L_i)(1-m_{[n-1,m_0]}(x,L_i)) \sum_{L_j}\left| \left(\sum_{P \in L_i.P^* \cap L_j.P^*}\psi_P(e)\right) - \sum_{L_k \neq L_i} m'(x,L_k) \left(\sum_{P \in L_j.P^* \cap L_k.P^*}\psi_P(e)\right) \right| \\
    &\mbox{\hspace{65mm}where } m'(x,L_k) = \frac{m_{[n-1,m_0]}(x,L_k)}{1-m_{[n-1,m_0]}(x,L_i)} > 0 \\
    &\mbox{\hspace{65mm}and noting } \sum_{L_k \neq L_i} m'(x,L_k) = 1 \\
    & \leq \frac{1}{2} \sum_{e=(x,x')}  \; \underset{\substack{\lambda_L \geq 0\\\underset{L\neq L_i}{\sum} \lambda_L = 1}}{\max} \; \sum_{L_j}\left| \left(\sum_{P \in L_i.P^* \cap L_j.P^*}\psi_P(e)\right) - \sum_{L_k \neq L_i} \lambda_{L_k} \left(\sum_{P \in L_j.P^* \cap L_k.P^*}\psi_P(e)\right) \right| \\
    &\mbox{\hspace{65mm}by renaming } m'(x,L) \mbox{ to } \lambda_L \mbox{ and noting that maximising $\lambda_L$ can be done} \\
    &\mbox{\hspace{65mm}independently from } m_{[n-1,m_0]}(x,L_i) \mbox{ which has an obvious maximum of }\nicefrac{1}{2}\\
    & = \frac{1}{2} \sum_{e=(x,x')}  \; \underset{L_k\neq L_i}{\max} \; \sum_{L_j}\left| \left(\sum_{P \in L_i.P^* \cap L_j.P^*}\psi_P(e)\right) - \left(\sum_{P \in L_j.P^* \cap L_k.P^*}\psi_P(e)\right) \right| \\
    &\mbox{\hspace{65mm}given } \sum_{L_k \neq L_i} \lambda_{L_k} \left(\sum_{P \in L_j.P^* \cap L_k.P^*}\psi_P(e)\right) \mbox{ forms a bounded convex polytope} \\
    &\mbox{\hspace{65mm}and thus such a maximum distance must occur at a vertex thereof} 
\end{aligned}
\end{equation}
This equation may be more unwieldy that the previous, but it is not impossible to simply for certain label ordering types. Thus if we constrain $\Psi$ to the space where either of the above equations is strictly less than 1, then we fulfill all the necessary requirements for Theorem \ref{the:main_conv}.

In certain cases, one bound might subsume the other. For example, Theorem \ref{the:main_conv} using the $p_1$ induced matrix norm for a binary models with labels $L_0$ and $L_1$ comes down to: $\sum_e \psi_{L_0}(e) < 1$ and $\sum_e \psi_{L_1}(e) < 1$ which is subsumed by that of the $p_\infty$ induced matrix norm which is $\sum_e \left( \psi_{L_0}(e)+\psi_{L_1}(e)\right) < 2$. For Potts' models, Theorem \ref{the:main_conv} using the $p_\infty$ induced matrix norm (translated as $\sum_e \max_{L_i\neq L_j}(\psi_{L_i}(e) + \psi_{L_j}(e))) < 2$) also always subsumes that using the $p_1$ induced matrix norm (translated as $\forall L_i \sum_e \psi_{L_i}(e) < 1$). For other more complex models, this might not be the case. 

The benefit of investigating mean-field iteration is that it provides us with information about when approaches from the literature may fail. Specifically, the approach taken by Zhang \textit{et al.} \cite{zheng2015conditional} does not explicitly constrain the regulation terms in the model which could lead to the gradient diverging if more than a few iterations are used. This means that very different results may be seen if different numbers of iterations are used, even though the same general model and configuration are applied.

\subsubsection{Tightness of the bound}\label{sec:tight}
In order to show the tightness of the above bound at least for some configurations, consider a binary segmentation of two vertices ($x_1$, $x_2$) connected by a single edge. The minimal $\psi$ ensuring at least one constraint is held with equality would be:
$$\psi_{L_0}\left((x_1,x_2)\right) + \psi_{L_1}\left((x_1,x_2)\right) = 2$$
For the other parameters, let us consider:
$$\psi_{L_0}(x_1)=\psi_{L_0}(x_2)=\psi_{L_1}(x_1)=\psi_{L_1}(x_2)=0$$
Noting that this configuration just equals the bound for both labels using the one existing edge which just violates the strict inequality. This yields the forward pass iteration as
$$\left[\begin{array}{c} \tilde{E}_{[n,m_0]}(x_1,L_0)\\\tilde{E}_{[n,m_0]}(x_2,L_0)\\\tilde{E}_{[n,m_0]}(x_1,L_1)\\\tilde{E}_{[n,m_0]}(x_2,L_1)\end{array}\right]=\left[\begin{array}{c} 
\psi_{L_0}(x_1)+\psi_{L_0}((x_1,x_2))\left(2\frac{\exp(\tilde{E}_{[n,m_0]}(x_2,L_0))}{\exp(\tilde{E}_{[n,m_0]}(x_2,L_0))+\exp(\tilde{E}_{[n,m_0]}(x_2,L_1))}-1\right) \\
\psi_{L_0}(x_2)+\psi_{L_0}((x_1,x_2))\left(2\frac{\exp(\tilde{E}_{[n,m_0]}(x_2,L_0))}{\exp(\tilde{E}_{[n,m_0]}(x_2,L_0))+\exp(\tilde{E}_{[n,m_0]}(x_2,L_1))}-1\right) \\
\psi_{L_1}(x_1)+\psi_{L_1}((x_1,x_2))\left(2\frac{\exp(\tilde{E}_{[n,m_0]}(x_2,L_1))}{\exp(\tilde{E}_{[n,m_0]}(x_2,L_0))+\exp(\tilde{E}_{[n,m_0]}(x_2,L_1))}-1\right) \\
\psi_{L_1}(x_2)+\psi_{L_1}((x_1,x_2))\left(2\frac{\exp(\tilde{E}_{[n,m_0]}(x_2,L_1))}{\exp(\tilde{E}_{[n,m_0]}(x_2,L_0))+\exp(\tilde{E}_{[n,m_0]}(x_2,L_1))}-1\right) \end{array}\right]$$
which has a fixed point at $m_{[n,m_0]}(x_1,L_0)=m_{[n,m_0]}(x_2,L_0)=m_{[n,m_0]}(x_1,L_1)=m_{[n,m_0]}(x_2,L_1)=0.5$. (Importantly, the convergence around this point is sub-linear, meaning that it does not fulfill the requirements for Theorem \ref{the:main_diff}.) For the sake of simplicity, assume that we initialise at this point, so $m_{[\infty,\psi]}(x_1)=m_{[\infty,\psi]}(x_1)=0.5$.

Even with this assumption, the derivatives do not converge. To see this, consider writing the backwards pass in matrix form as in the previous sections:
$$A_{[\infty]} = \left[ \begin{array}{cccccc} 1 & 0 & 0 & 0 & 0 & 0 \\ 0 & 1 & 0 & 0 & 0 & 0 \\ 0 & 0 & 1 & 0 & 0 & 0 \\ 0 & 0 & 0 & 1 & 0 & 0 \end{array} \right] \mbox{\hspace{5mm}and\hspace{5mm}} B_{[\infty]} = \left[ \begin{array}{cccc} 
0 & \nicefrac{1}{2}\psi_{L_0}(e) & 0 & -\nicefrac{1}{2}\psi_{L_0}(e) \\
\nicefrac{1}{2}\psi_{L_0}(e) & 0 & -\nicefrac{1}{2}\psi_{L_0}(e) & 0 \\
0 & -\nicefrac{1}{2}\psi_{L_1}(e) & 0 & \nicefrac{1}{2}\psi_{L_1}(e) \\
-\nicefrac{1}{2}\psi_{L_1}(e) & 0 & \nicefrac{1}{2}\psi_{L_1}(e) & 0 \\ \end{array} \right]$$
Thus, the characteristic polynomial of $B_{[\infty]}$ can be computed as:
\begin{equation*}
-\frac{1}{4}\lambda^2\left(\psi_{L_0}(e)+\psi_{L_1}(e)-2\lambda\right)\left(\psi_{L_0}(e)+\psi_{L_1}(e)+2\lambda\right)
\end{equation*}
which gives us one eigenvalue of
$2\lambda = \psi_{L_0}(e)+\psi_{L_1}(e)$ or $\lambda=1$.
Note that because $B_{[\infty]}$ has at least one eigenvalue of modulus 1, the formula $C = A_{[\infty]} + B_{[\infty]} A_{[\infty]} + B^2_{[\infty]} A_{[\infty]} + ... B_\infty^nA_{[\infty]} + ... $ diverges meaning that our gradients with respect to the parameters, $\psi$, are not defined.

This illustrates a key potential issue with the use of iterative mean-field estimation as it only takes two adjacent pixels with similar enough bias terms and a high enough regularisation term between them to create a point-source of instability in gradient computation even if the forward pass does converge or produce acceptable results. If a very small number of automatically-differentiated iterations are used, this may not be noticeable as they are insufficient to allow the gradients to converge at all. 

%% file: 3_exp_marg.tex
\section{Marginal probability estimation quality}

When investigating marginal probability estimators, a crucial aspect is to measure how well these estimators match their reference values given some gold-standard, if much more expensive, calculation method.

\subsection{Methods - Two-dimensional random Ising models}
In order to evaluate the relatively efficacy of our solution algorithms, we created a number of two-dimensional Ising models with vertices aligned in a 4-connected lattice of size $N \times N$. The bias term for each variable was chosen using independent unit normal distributions. The regularisation term was independently selected from a uniform distribution between 0 and $\beta\ln 2$ where $\alpha \in [0,1]$ is a parameter to control how close the regularisation terms could get to the constraint proposed in Equation \ref{eq:plbp_const_infty} for the central vertices in the lattice. This experiment was repeated with three values of $N \in \{5,10,20\}$ and four values of $\alpha \in \{0.25,0.5,0.75,1\}$.

We measured the accuracy of the estimated marginal distributions with respect to the reference beta distribution for variable $x$, $\beta_x$, by calculating:
\begin{itemize}
        \item the root mean squared error between the estimated marginal probability and the mean of the reference distributions, $$RMSE = \sqrt{\frac{\sum_{x\in\Omega} \left|m(x)-\mathbf{E}_{\beta_x}[u(x)]\right|^2}{|\Omega|}}$$
    \item the expected Kullback-Liebler divergence of the marginal probability estimate with the rate taken from the reference beta distribution
    $$\mathbf{E}[D_{KL}]= \sum_{x\in\Omega}\int_0^ 1 \beta_x[p]\left(p\ln \frac{p}{m(x)}+(1-p)\ln \frac{1-p}{1-m(x)}\right)dp $$
\end{itemize}
both of which can be solved analytically given the parameters of the reference beta distributions. These beta distributions are constructed using Gibbs sampling. The maximum probability distribution is used rather than a burn-in period and only 1 in every $N$ samples is saved in order to minimise the effect of correlation between consecutive samples. The samples are then used to estimate the marginal distributions for each variable using a Bayesian approach, treating them as samples from an unknown Poisson distribution and thus give us the parameters of a beta distribution to represent the probability density of the Poisson distribution's sole parameter. The number of samples to be taken from the Gibbs sampler was determined as $N^2/(2e) - 2$ where $e$ is the desired approximate error bound. (This bound is the approximate value of $\mathbf{E}[D_{KL}]$ for a value of $m(x)$ equal to the expected value of the reference beta distribution assuming a large number of samples.) We used a value of $e=0.005$ in all experiments.

\subsection{Methods - Two-dimensional random models with hierarchical labels}
Binary models do not capture the flexibility of these operators. In order to demonstrate their greater flexibility, we have repeated this marginal probability estimation experiment with another grid-like model. Unlike the Ising model discussed above, this model assigns one of five labels to each vertex, with certain labels having a preference for neighbouring labels that can be organised in a hierarchical structure, shown in Figure \ref{fig:syn_labelordering}. That is, labels $L_0$ and $L_1$ are more likely to be adjacent than $L_0$ and $L_2$ due to being closer in the tree.

\begin{figure}[h]
    \centering
    \includegraphics[width=0.275\linewidth]{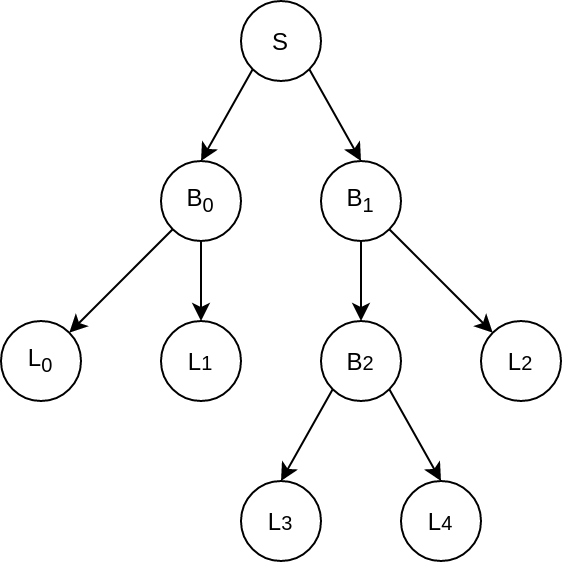}
    \caption{Label ordering for random hierarchical model experiment containing 5 leaf labels and 3 branches (i.e. ways of grouping leaf labels) along with the source, S, representing all labels.}
    \label{fig:syn_labelordering}
\end{figure}

As this model contains multiple labels, the reference distribution for the marginal estimates is a Dirichlet distribution, rather than a beta distribution. The same root-mean-squared error and expected Kullback-Liebler divergence can be computed. We have however increased the number of samples to $\frac{5}{2}(N^2/(2e) - 2)$ to account for the additional labels.

\subsection{Results}

\begin{figure}
     \centering
     \begin{subfigure}[b]{0.475\textwidth}
         \centering
         \includegraphics[width=\textwidth]{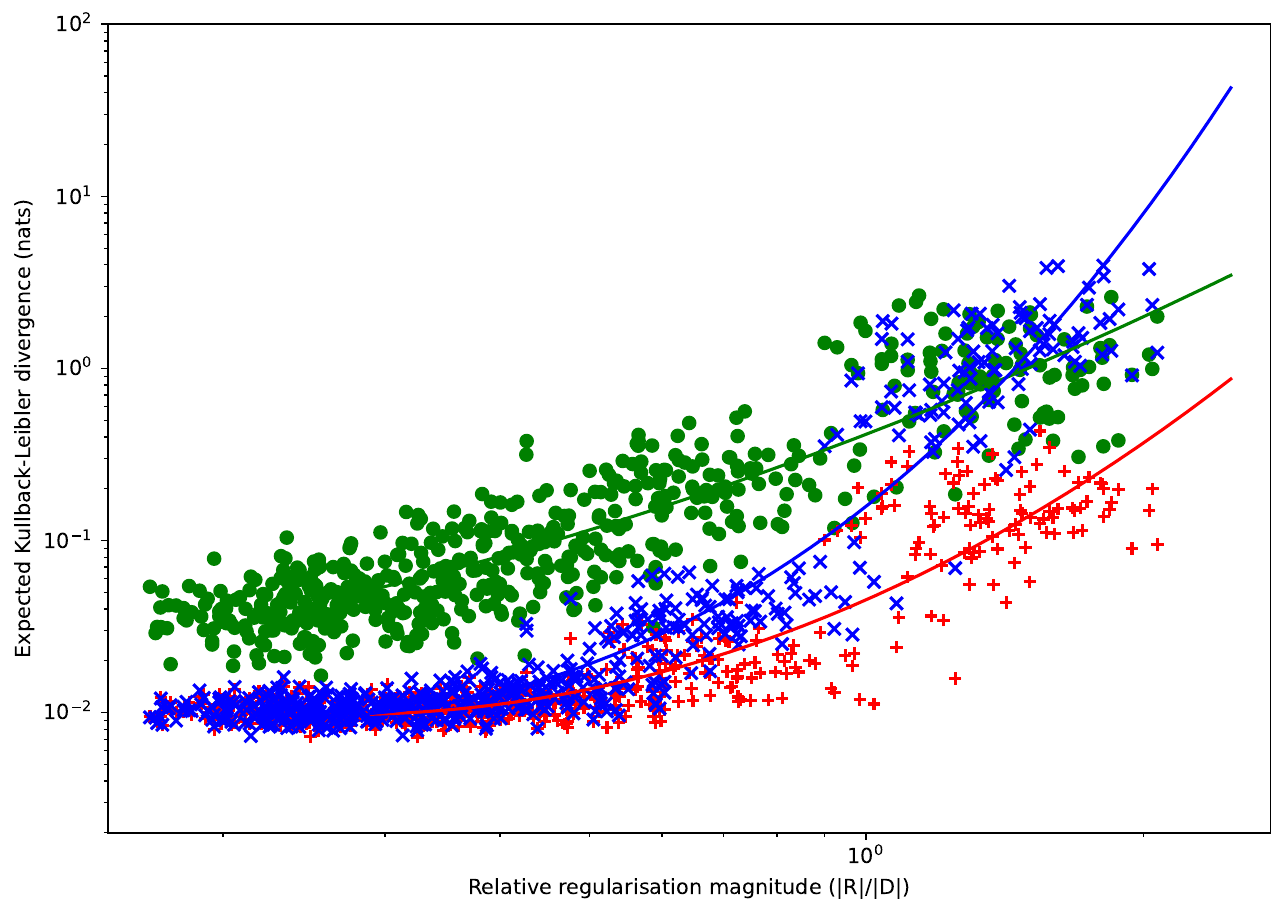}
         \caption{$\mathbf{E}[D_{KL}]$ for 5x5 Ising marginals}
     \end{subfigure}
     \begin{subfigure}[b]{0.475\textwidth}
         \centering
         \includegraphics[width=\textwidth]{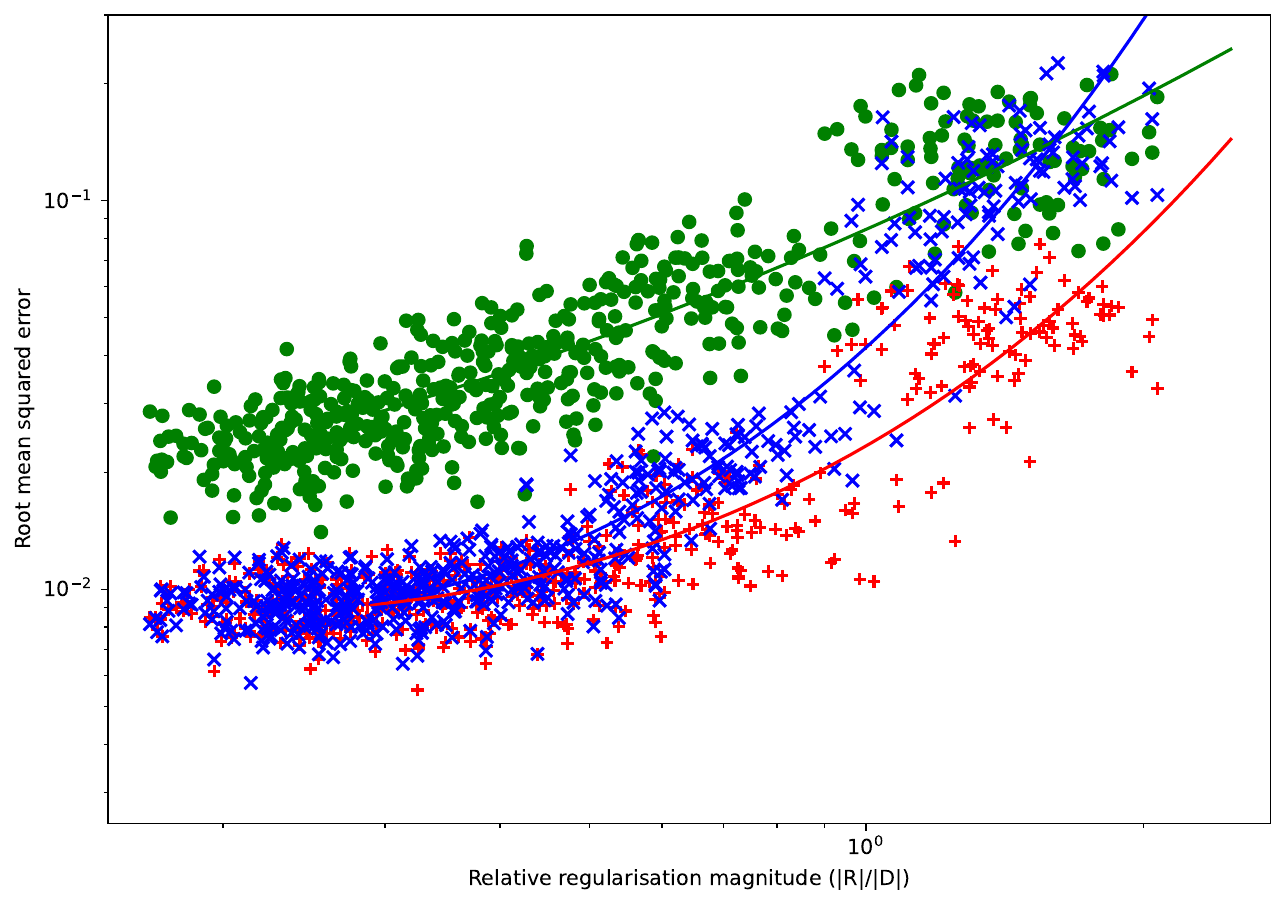}
         \caption{RMSE for 5x5 Ising marginals}
     \end{subfigure}
     \begin{subfigure}[b]{0.475\textwidth}
         \centering
         \includegraphics[width=\textwidth]{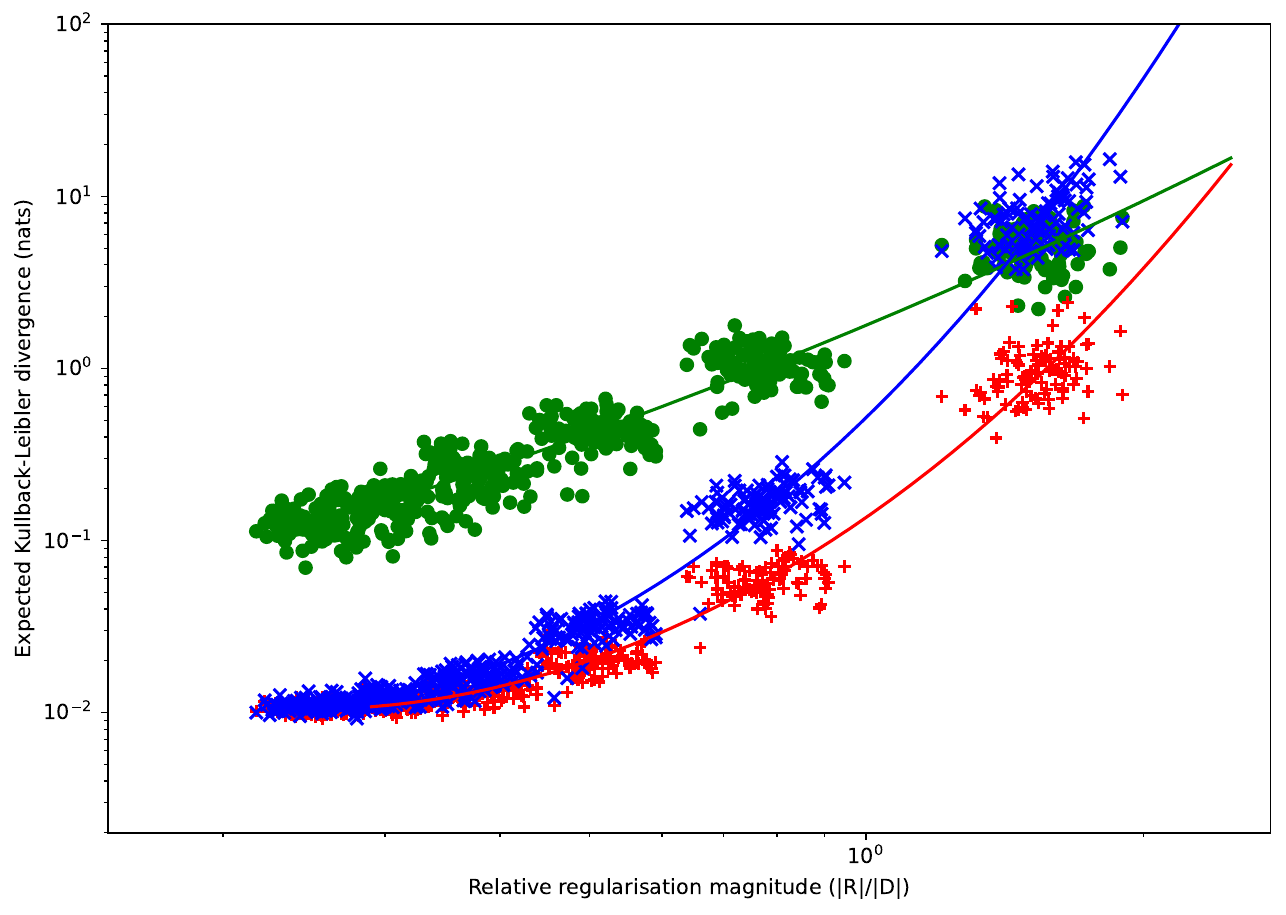}
         \caption{$\mathbf{E}[D_{KL}]$ for 10x10 Ising marginals}
     \end{subfigure}
     \begin{subfigure}[b]{0.475\textwidth}
         \centering
         \includegraphics[width=\textwidth]{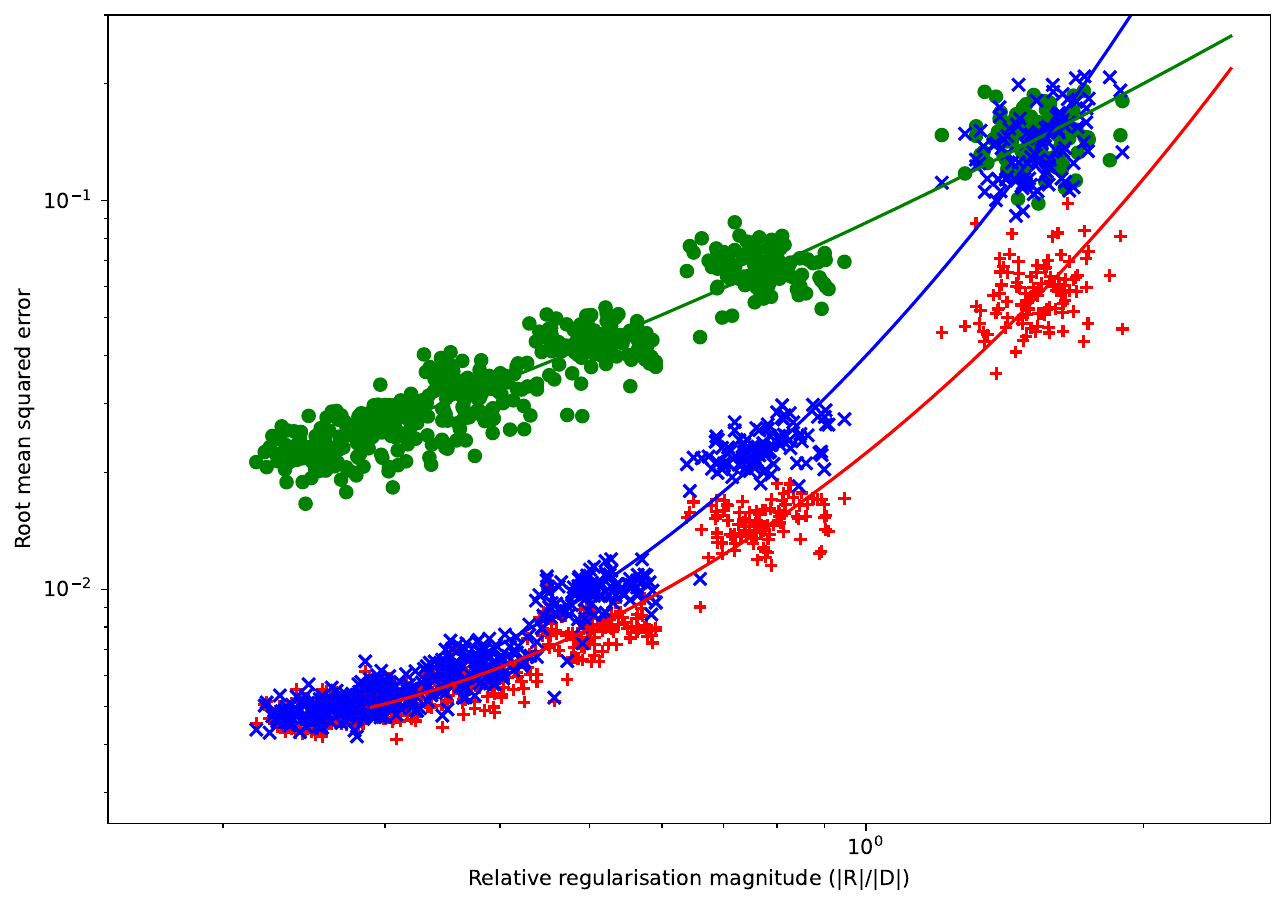}
         \caption{RMSE for 10x10 Ising marginals}
     \end{subfigure}
     \begin{subfigure}[b]{0.475\textwidth}
         \centering
         \includegraphics[width=\textwidth]{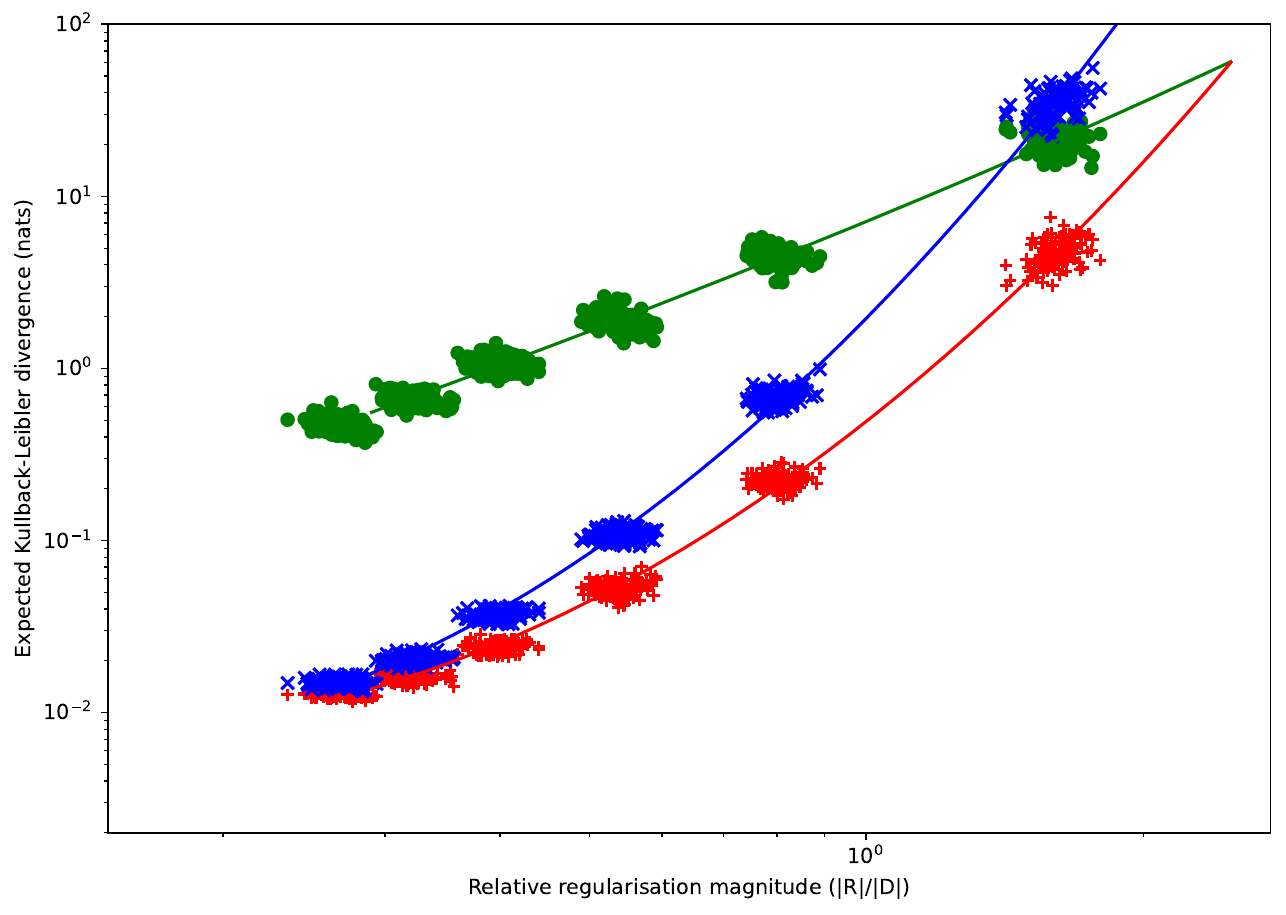}
         \caption{$\mathbf{E}[D_{KL}]$ for 20x20 Ising marginals}
     \end{subfigure}
     \begin{subfigure}[b]{0.475\textwidth}
         \centering
         \includegraphics[width=\textwidth]{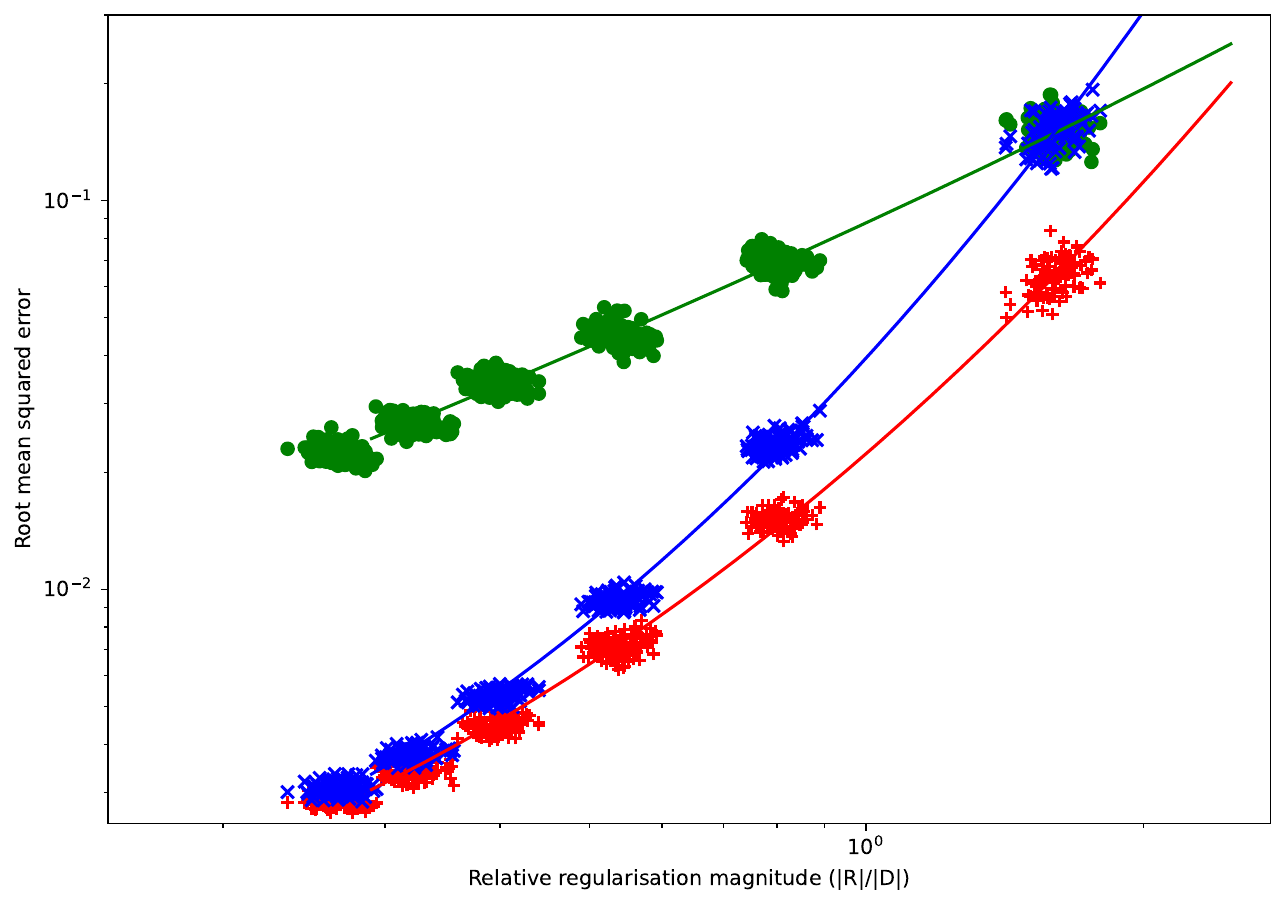}
         \caption{RMSE for 20x20 Ising marginals}
     \end{subfigure}
        \caption{Comparision of marginal distributions of variables between solution types ({\color{red}red:} Parallel loopy belief propagation, {\color{blue}blue:} Iterative mean field estimation, and {\color{ao(english)}green:} Bias term only with no regularisation) in terms of Root mean squared error (RMSE) and  expected KL Divergence ($\mathbf{E}[D_{KL}]$). Each point cloud includes 100 simulations.}
        \label{fig:stats_results}
\end{figure}

\begin{figure}
     \centering
     \begin{subfigure}[b]{0.475\textwidth}
         \centering
         \includegraphics[width=\textwidth]{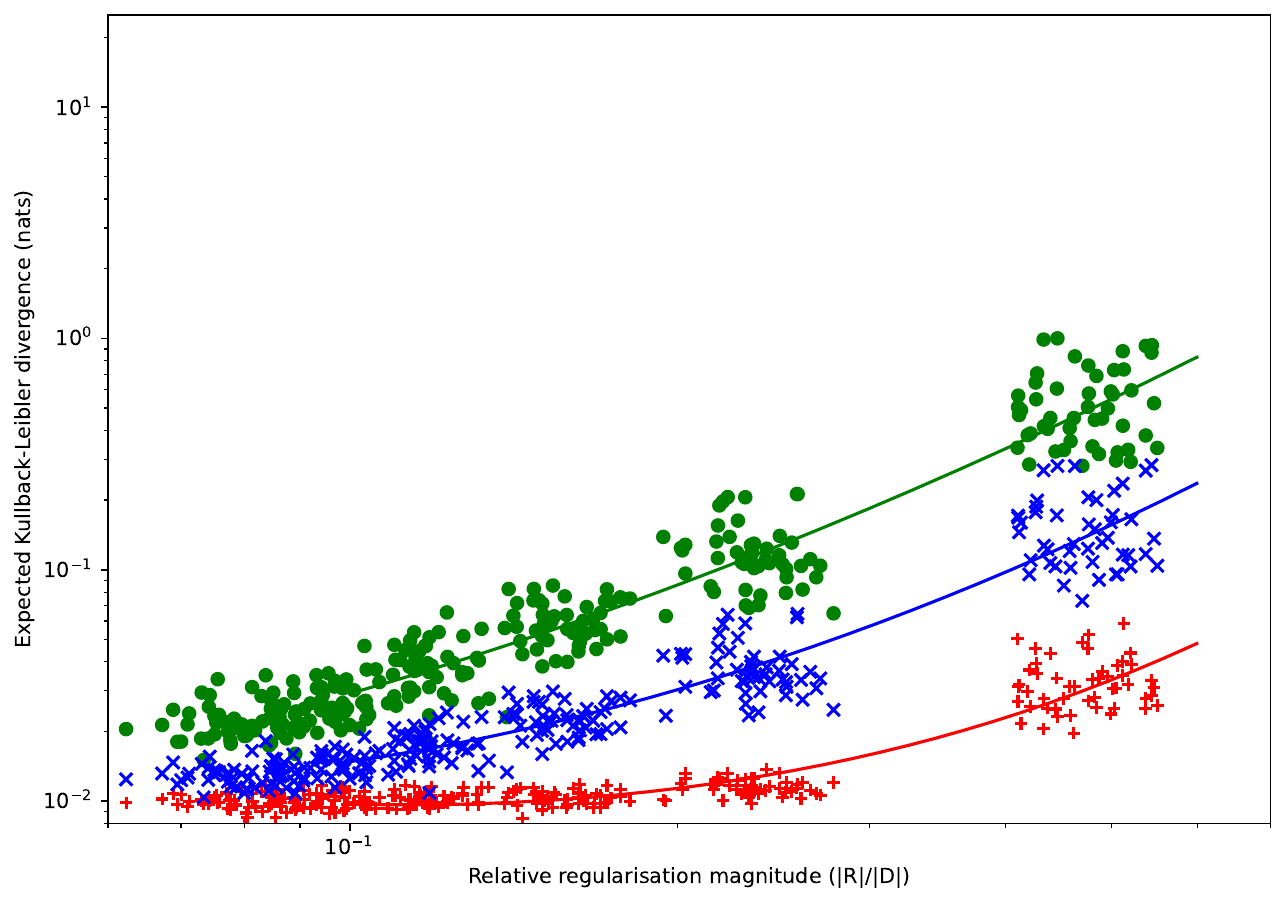}
         \caption{$\mathbf{E}[D_{KL}]$ for 5x5 hierarchical label marginals}
     \end{subfigure}
     \begin{subfigure}[b]{0.475\textwidth}
         \centering
         \includegraphics[width=\textwidth]{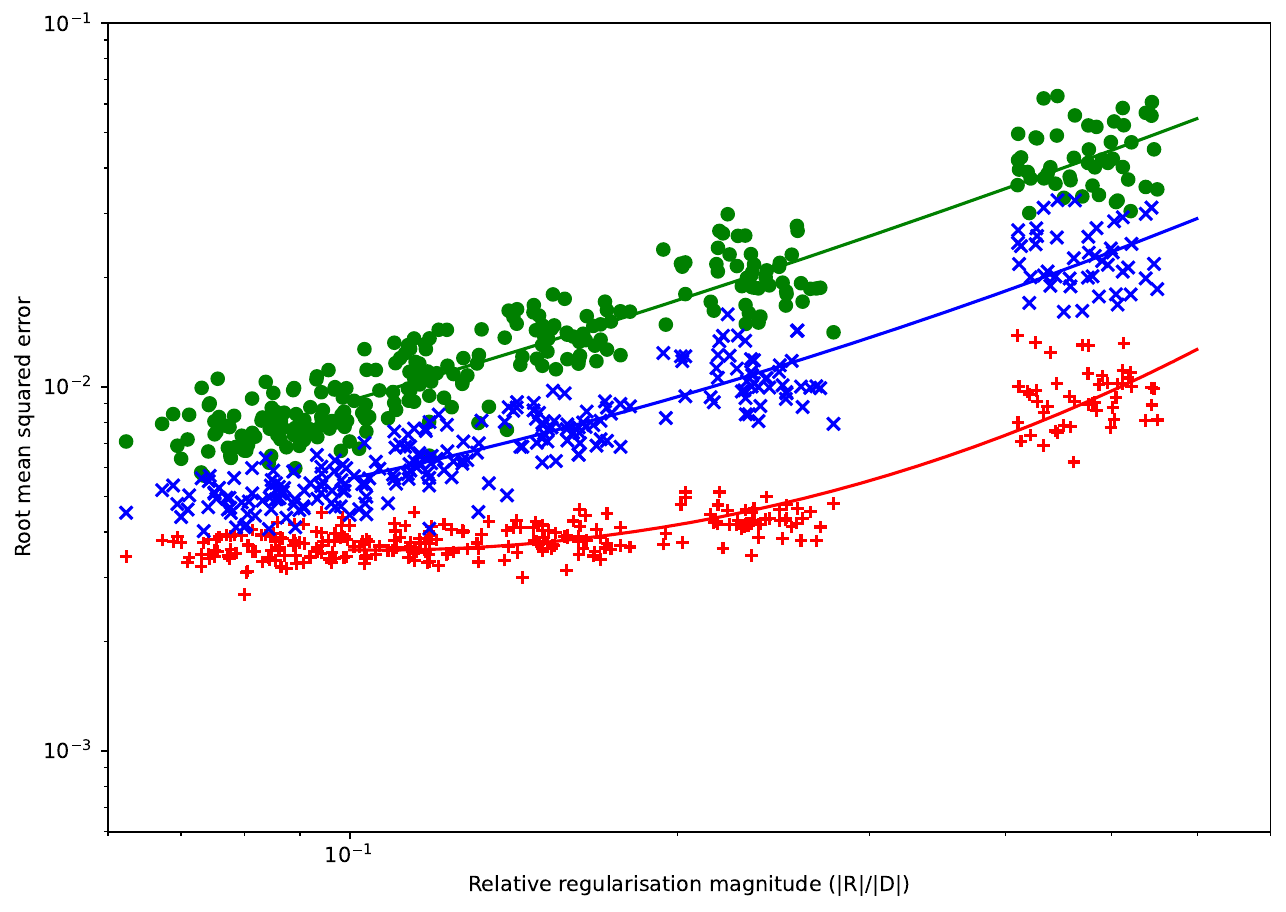}
         \caption{RMSE for 5x5 hierarchical label marginals}
     \end{subfigure}
     \begin{subfigure}[b]{0.475\textwidth}
         \centering
         \includegraphics[width=\textwidth]{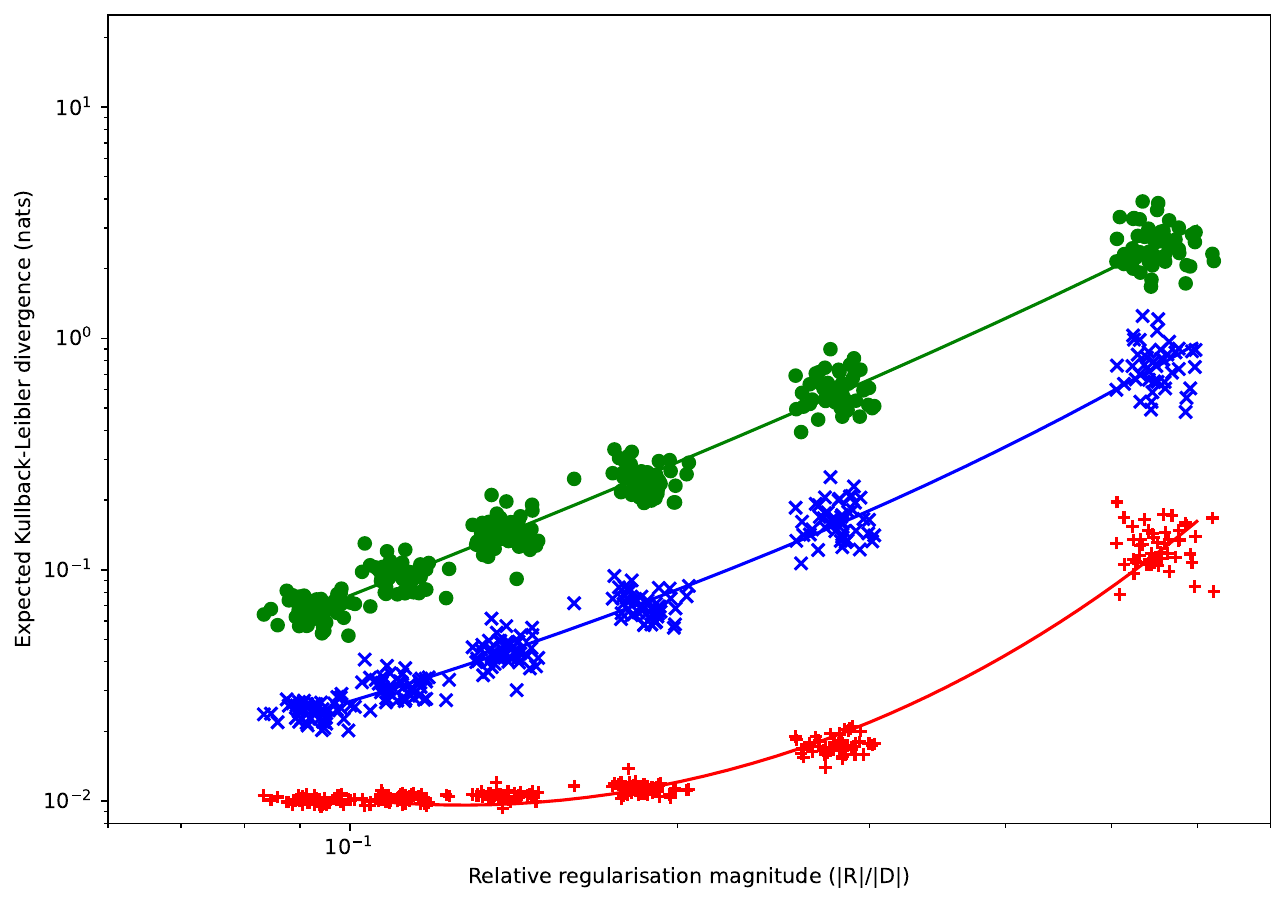}
         \caption{$\mathbf{E}[D_{KL}]$ for 10x10 hierarchical label marginals}
     \end{subfigure}
     \begin{subfigure}[b]{0.475\textwidth}
         \centering
         \includegraphics[width=\textwidth]{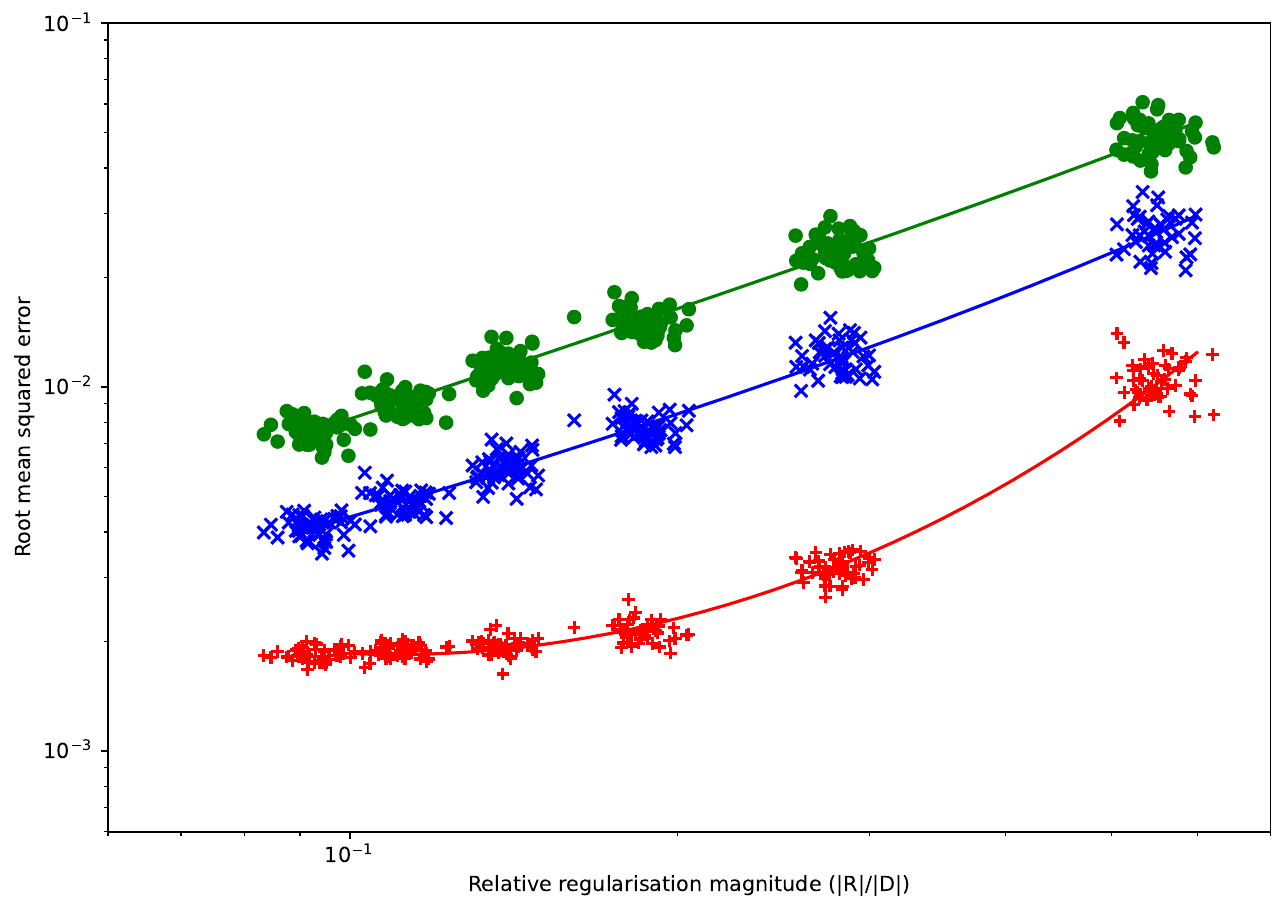}
         \caption{RMSE for 10x10 hierarchical label marginals}
     \end{subfigure}
     \begin{subfigure}[b]{0.475\textwidth}
         \centering
         \includegraphics[width=\textwidth]{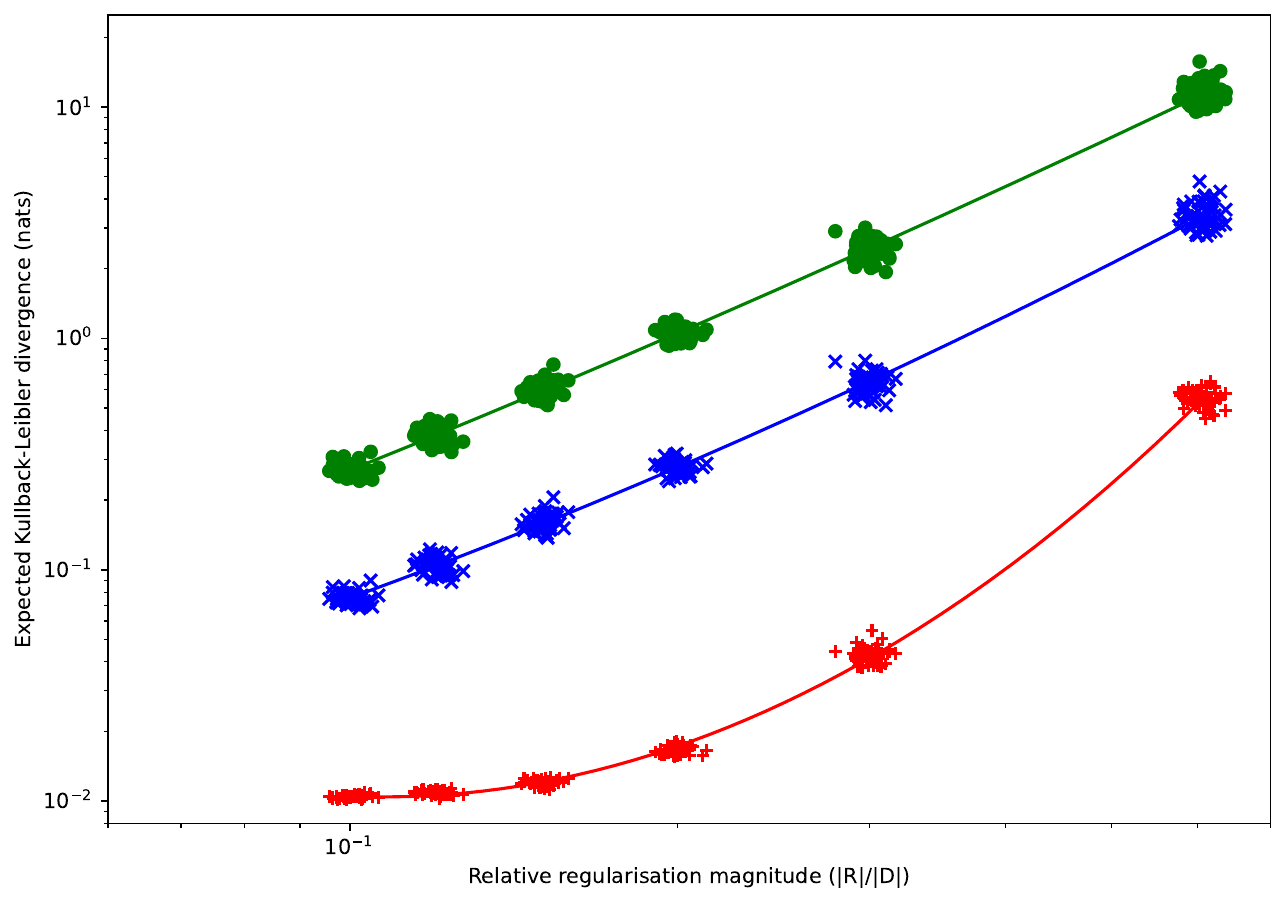}
         \caption{$\mathbf{E}[D_{KL}]$ for 20x20 hierarchical label marginals}
     \end{subfigure}
     \begin{subfigure}[b]{0.475\textwidth}
         \centering
         \includegraphics[width=\textwidth]{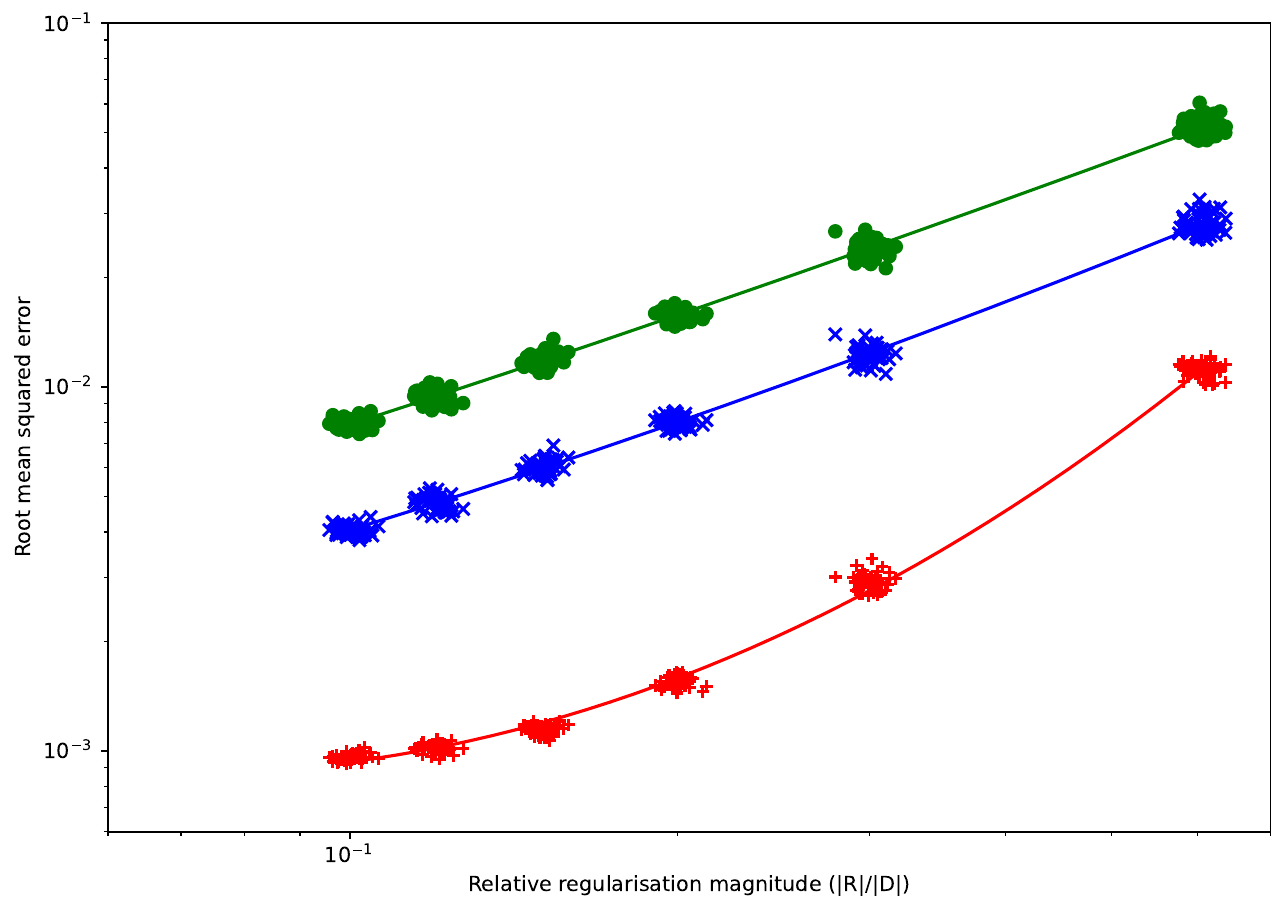}
         \caption{RMSE for 20x20 hierarchical label marginals}
     \end{subfigure}
        \caption{Comparision of marginal distributions of variables between solution types ({\color{red}red:} Parallel loopy belief propagation, {\color{blue}blue:} Iterative mean field estimation, and {\color{ao(english)}green:} Bias term only with no regularisation) in terms of Root mean squared error (RMSE) and  expected KL Divergence ($\mathbf{E}[D_{KL}]$). Each point cloud includes 50 simulations.}
        \label{fig:stats_results_h}
\end{figure}

The results for the Ising model and the hierarchical model are shown in Figures \ref{fig:stats_results} and \ref{fig:stats_results_h} respectively. The spread of the values is consistent with the fact that both metrics (as well as the relative regularisation values on the x-axis) involve a degree of averaging and larger lattices have a higher number of variables to average over, hence a smaller spread. In general, the pLBP consistently outperforms the other methods at higher regularisations, likely as they still fall within the region of convergence for pLBP but not for MFI. This was less pronounced at lower regularisation values where all three methods converge towards the exact marginal distribution. This generally confirms our intuition that pLBP is more stable than MFI, which allows it to have a larger region of convergence outside of which, the instability can lead to higher quantitative errors as the estimated marginals tend to converge to extreme values close to 0 and 1 rather than the true marginal.

The errors appear to be fairly well modelled via a polylogarithmic model in the log-log space for the two developed operators, while being largely linear for data-terms only for both the binary and hierarchical label model. The primary difference appears to be related to how close the models are to the point of divergence, with the boundaries tested in the hierchical model experiment being further, due to the more conservative nature of the bound developed in the earlier chapter.

%% file: 4_exp_learning.tex
\section{Synthetic binary image segmentation}

The previous section confirmed that pLBP, especially at higher regularisation values, more accurately models the underlying distribution encoded in the CRF. However, in machine learning practice, one is often less interested in the distribution itself than whether or not the maximum \textit{a posteriori} probability configuration of said distribution produces accurate predictions on a given task. In order to investigate this in a controlled manner, we have created a synthetic image segmentation experiment designed to include some of the more difficult elements of medical image segmentation.

\subsection{Loss Functions}
One of the benefits of using a probabilistic approach to understanding deep learning is that it gives us an immediate definition of a particular loss function, the \textit{negative log-likelihood} (NLL) loss. Specifically, given a database of $N$ parameter-configuration pairs, $\psi^{(i)},y^{(i)}, i \in [1..N]$, the NLL loss is defined as:
\begin{equation*}
    NLL =  - \frac{1}{N}\sum_{i=1}^N \log P[y^{(i)}|\psi^{(i)}] = \frac{1}{N} \sum_{i=1}^N \left(\log Z[\psi^{(i)}] - E_{[\psi^{(i)},y^{(i)}]}\right)
\end{equation*}
where $P[u|\psi]$ is the probability framework defined by the PGM. Certain well-known loses, cross-entropy and binary-cross-entropy for example, can be understood as NLL losses for simple PGM's with no regularisation terms.

Interestingly, for deep learning, the loss function is not as necessary in comparison to the gradient of said loss function with respect to the parameters of the model. For models with linear energy terms such as PGMs, these can be simplified as:
\begin{equation*}
\begin{aligned}
    \frac{\delta NLL}{\delta\psi^{(i)}_L(x)} = & \frac{1}{N} \frac{\delta\log Z[\psi^{(i)}]}{\delta \psi^{(i)}_L(x)} - \frac{1}{N}\frac{\delta E_{[\psi^{(i)},y^{(i)}]}}{\delta \psi^{(i)}_L(x)} \\
    = & \frac{1}{N} \frac{1}{Z[\psi^{(i)}]} \sum_{u\in\mathbb{U}} \frac{\delta E_{[\psi^{(i)},u]}}{\delta \psi^{(i)}_L(x)} - \frac{1}{N}\frac{\delta E_{[\psi^{(i)},y^{(i)}]}}{\delta \psi^{(i)}_L(x)} \\
    = & \frac{1}{N} \sum_{u\in\mathbb{U}} \frac{e^{E_{[\psi^{(i)},u]}}}{Z[\psi^{(i)}]} u_L(x) - \frac{1}{N} u^{(i)}_L(x) \\
    = & \frac{1}{N} \left( m_L(x) - u^{(i)}_L(x) \right) \\
\end{aligned}
\end{equation*}
where $m_L(x)$ is the marginal probability calculated for said bias term. (This is equivalent to the cross entropy loss for models without regularisation.) The same approach can also be applied for the regularisation terms. This means that even if one does not have a framework for directly estimating probabilities using a particular PGM, the only necessary thing for calculating the derivative of the NLL loss depending on it are the marginal distributions of the individual variables and the joint marginal distributions of adjacent variables.

The other loss function used in our experiment is the popular Dice coefficient loss \cite{sudre2017generalised}:
$$DCL = \frac{1}{N}\sum_{i=1}^N\sum_L \frac{2\sum_x m_L(x) y^{(i)}_L(x)}{\sum_x m_L(x) +\sum_xy^{(i)}_L(x)+\epsilon}$$
where $\epsilon$ is a very small positive constant.
These approaches allow our method to be more readily seen in the context of the state-of-the-art methods in integrating CRF's into deep learning, such as CRF-as-RNN \cite{zheng2015conditional}. The additional of the edge map as a subject for a Dice loss should also encourage lower regularisation values at edges in the ground-truth segmentation while increasing them elsewhere.

\subsection{Synthetic Data Generation}
In order to minimise questions of annotation uncertainty or data availability, our experiment uses synthetic data. Each image consists of a circle with a random radius of between 0.15 and 0.25 of the image size with a gradual central hypointensity. The image is then polluted with both low frequency and high frequency noise patterns as to mimic noise and distracting structures. This data has been designed to be visually more difficult to segment than most medical images with a purposefully low contrast-to-noise ratio of $\nicefrac{1}{2}$. Example image-segmentation pairs are given in Figure \ref{fig:sim_images}.

\begin{figure}
     \centering
     \begin{subfigure}[b]{0.225\textwidth}
         \centering
         \includegraphics[width=\textwidth]{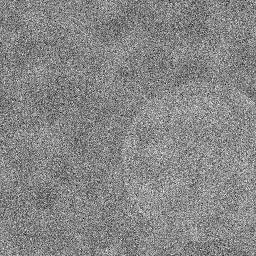}
     \end{subfigure}
     \begin{subfigure}[b]{0.225\textwidth}
         \centering
         \includegraphics[width=\textwidth]{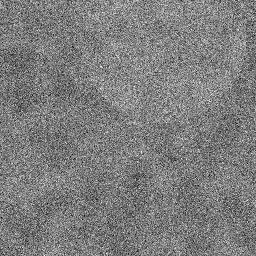}
     \end{subfigure}
     \begin{subfigure}[b]{0.225\textwidth}
         \centering
         \includegraphics[width=\textwidth]{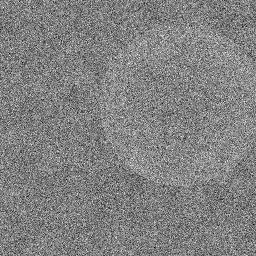}
     \end{subfigure}
     \begin{subfigure}[b]{0.225\textwidth}
         \centering
         \includegraphics[width=\textwidth]{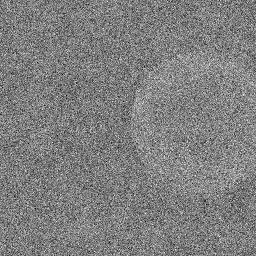}
     \end{subfigure}

    \vspace{1mm}
     
     \begin{subfigure}[b]{0.225\textwidth}
         \centering
         \includegraphics[width=\textwidth]{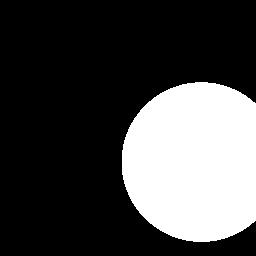}
     \end{subfigure}
     \begin{subfigure}[b]{0.225\textwidth}
         \centering
         \includegraphics[width=\textwidth]{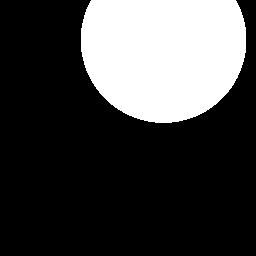}
     \end{subfigure}
     \begin{subfigure}[b]{0.225\textwidth}
         \centering
         \includegraphics[width=\textwidth]{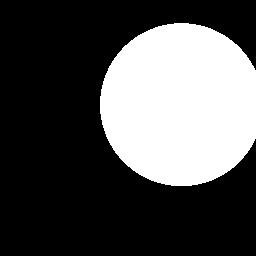}
     \end{subfigure}
     \begin{subfigure}[b]{0.225\textwidth}
         \centering
         \includegraphics[width=\textwidth]{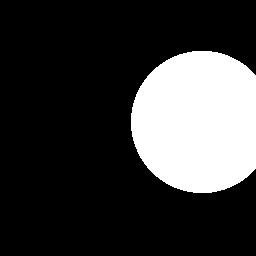}
     \end{subfigure}
     \caption{Example image (first row) and segmentation (second row) pairs of synthetic images for binary segmentation.}
     \label{fig:sim_images}
\end{figure}

\subsection{Networks}
In order to emphasise the effect of the models, a series of relatively simple networks (shown in Figure \ref{fig:sim_network_structure}) were used consisting of a number of convolutional layers of width 64 followed by rectified linear units. The final activation function is replaced with one of three options:
\begin{enumerate}
    \item \underline{A sigmoid function} representing the current paradigm without regularisation;
    \item \underline{A CRF marginal probability estimator using iterative mean-field estimation} which takes in two additional inputs (the regularisation term in the horizontal and vertical directions) and produces three outputs: the marginal probability distribution for each vertex similar to the sigmoid function above and two marginal probability distributions giving the probability of an edge existing between two adjacent vertices in either the horizontal or vertical direction; and
    \item \underline{A CRF marginal probability estimator using parallel loopy belief propagation} which has the same inputs/outputs as the MFI one described above.
\end{enumerate}
For the two CRF-based activation functions, additional regularisation terms are necessary. In order to minimise the number of additional parameters in the CRF models, a single ``edge map'' (between 0 and 1) and the separate horizontal and vertical terms are set to be the mean value of this map for two adjacent vertices, multiplied by a constant to keep the algorithms provably convergent ($ln(2)$ for pLBP and $0.5$ for MFE). This means that the last layer has contains one convolution with 1 output channel for the bias term and a second convolution with 1 output channel for the edge map to be used by the CRF-based solution algorithms. This limits the added number of parameters to 65, which is small compared to the remainder of the network. All networks were given the same randomly initialised weights for all shared layers and presented with the same data in the same order, controlling for these sources of randomness and allowing for paired comparisons.

\begin{figure}[h]
    \centering
    \includegraphics[width=0.575\linewidth]{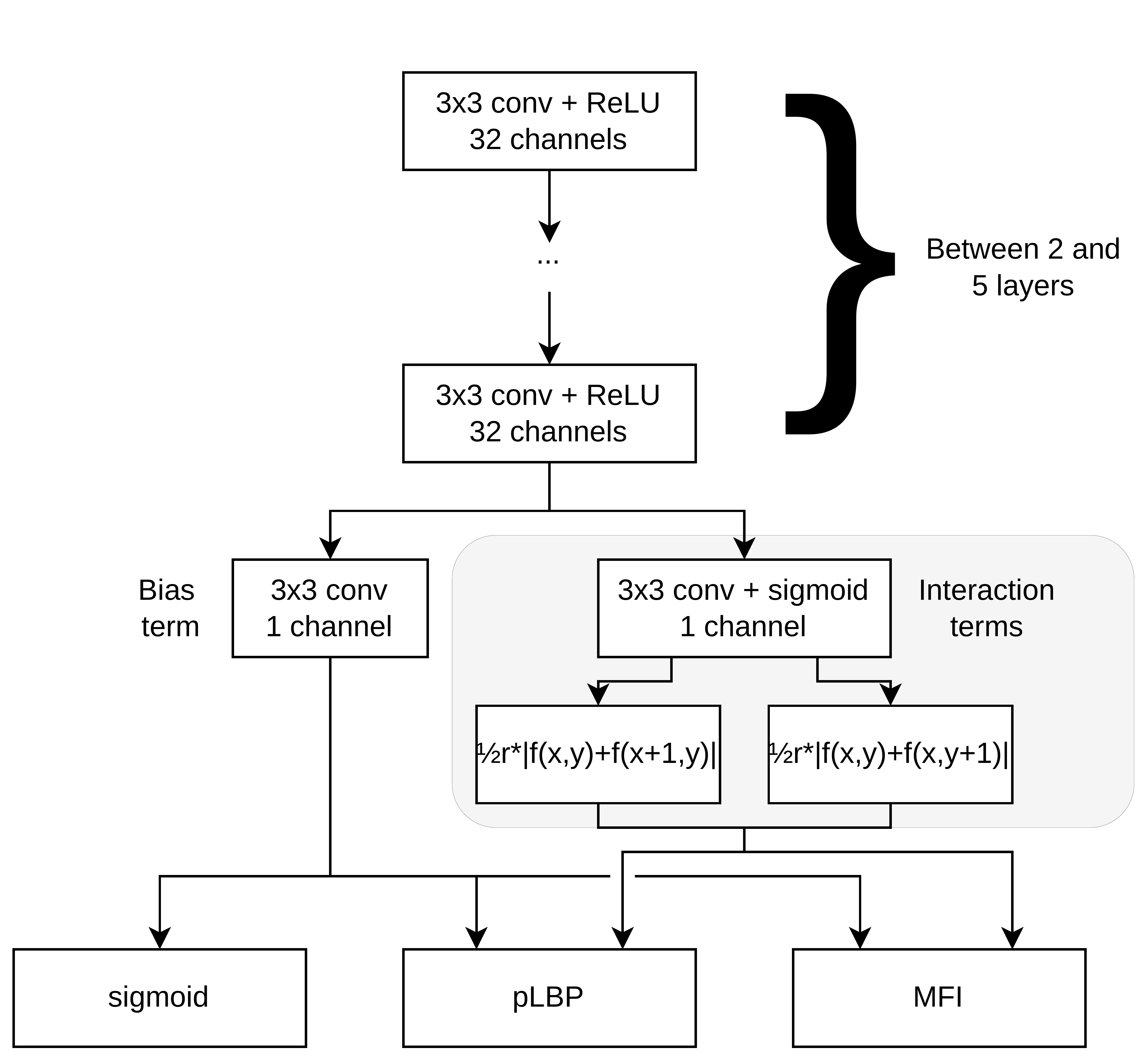}
    \caption{Network structure for simulated learning experiment}
    \label{fig:sim_network_structure}
\end{figure}

The experiment was repeated 10 times for each network configuration with different random initialisations. This repetition allows us to gauge the robustness of any performance differences.

\subsection{Results}

\begin{figure}[p]
    \centering
    \includegraphics[width=0.85\linewidth]{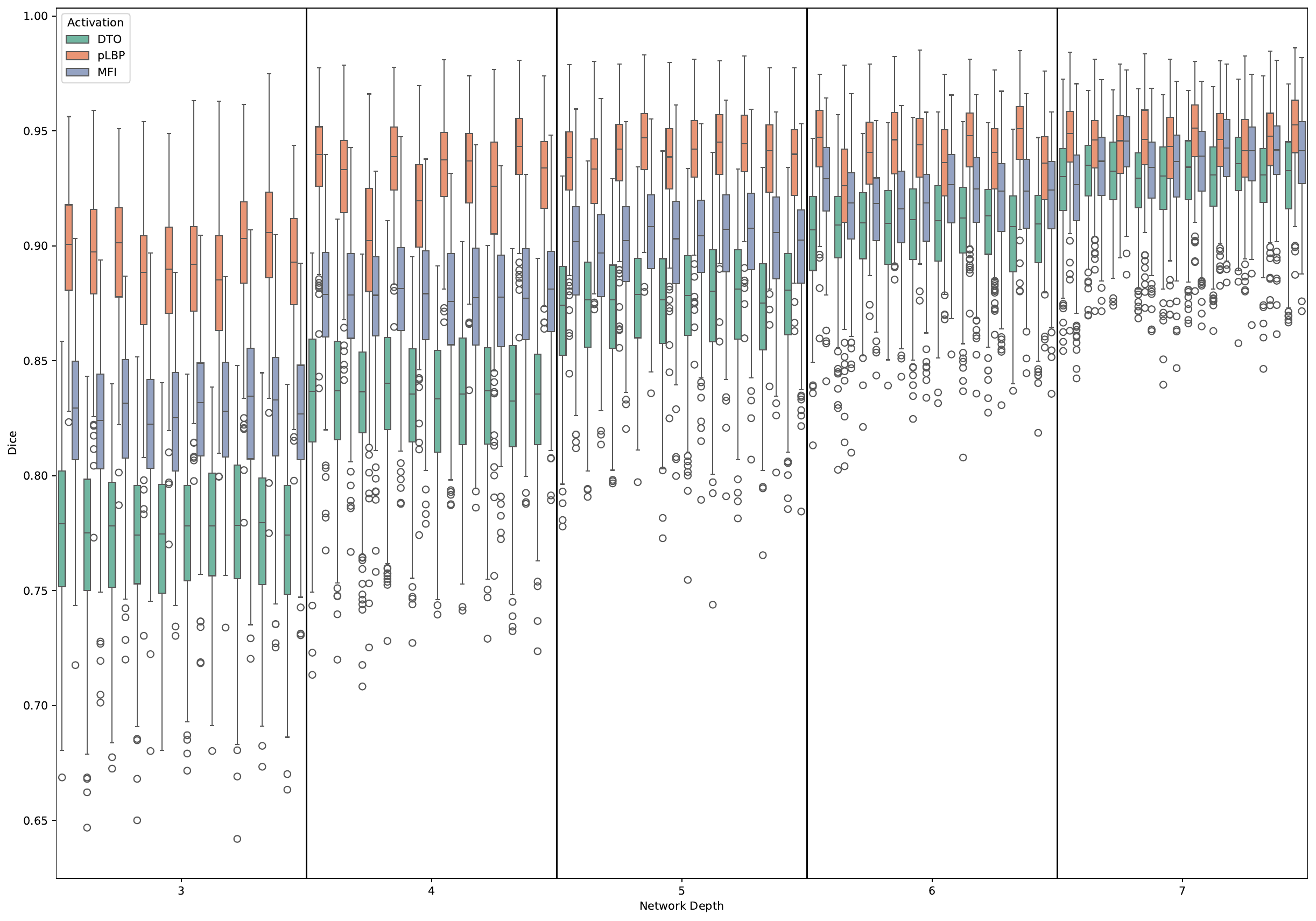}
    \caption{Dice scores for synthetic image segmentation across 10 replications for each network depth using the negative log-likelihood loss function. }
    \label{fig:synth_seg_res_nll}
\end{figure}

\begin{figure}[p]
    \centering
    \includegraphics[width=0.85\linewidth]{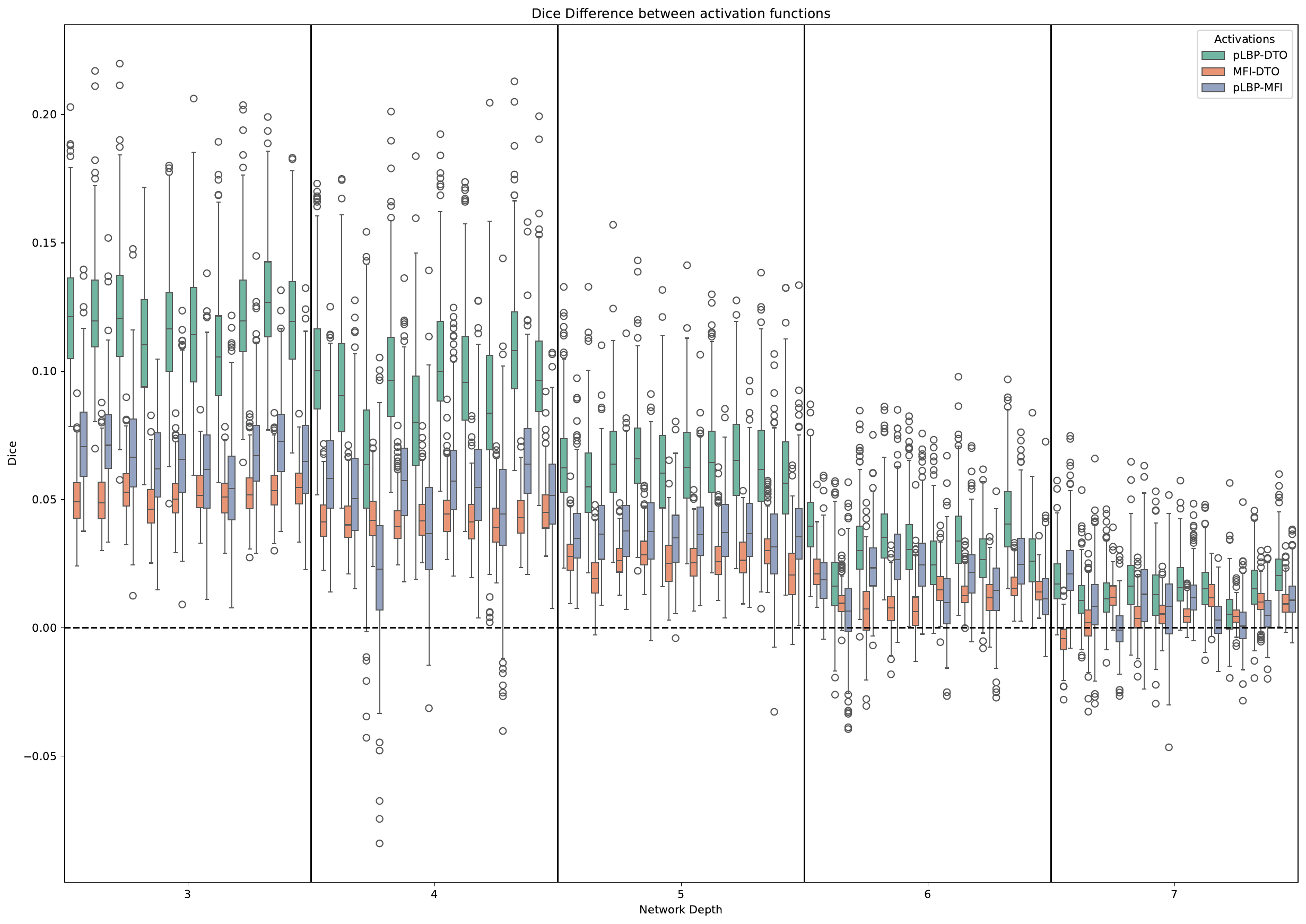}
    \caption{Distribution of paired differences in Dice scores for synthetic image segmentation across 10 replications for each network depth using the negative log-likelihood loss function. }
    \label{fig:synth_seg_diff_nll}
\end{figure}

\begin{figure}[p]
    \centering
    \includegraphics[width=0.8675\linewidth]{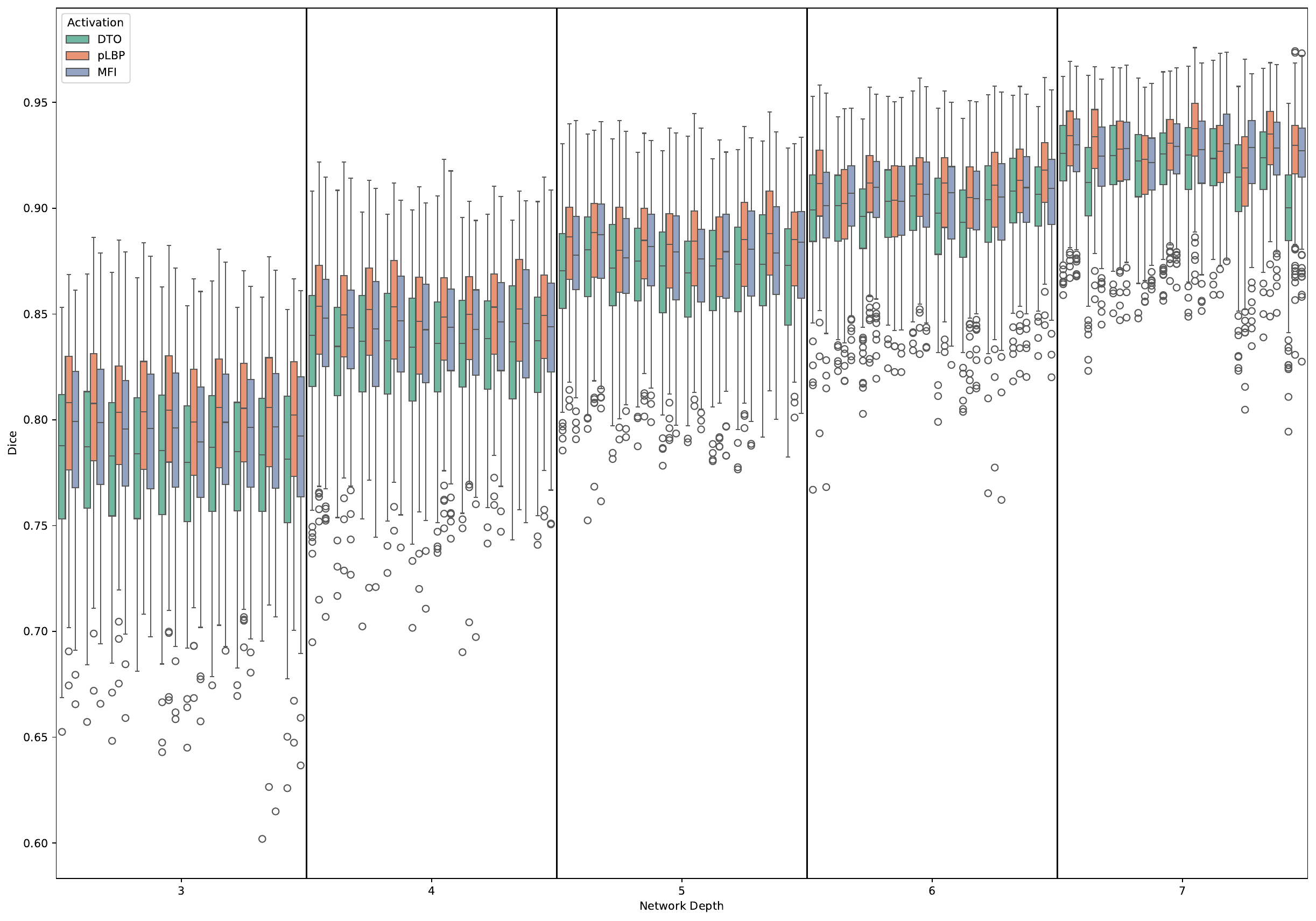}
    \caption{Dice scores for synthetic image segmentation across 10 replications for each network depth using the Dice loss function. }
    \label{fig:synth_seg_res_dice}
\end{figure}

\begin{figure}[p]
    \centering
    \includegraphics[width=0.8675\linewidth]{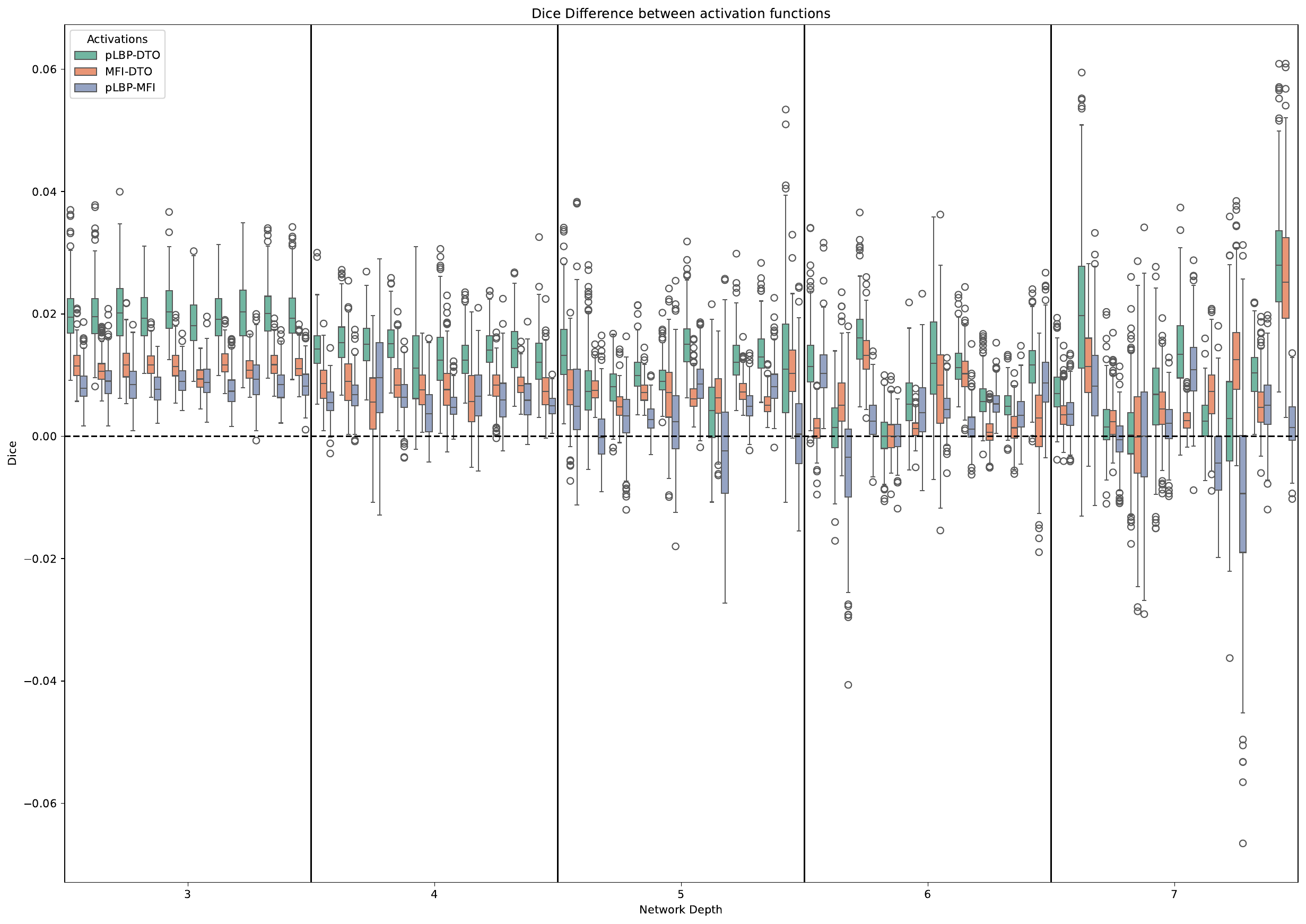}
    \caption{Distribution of paired differences in Dice scores for synthetic image segmentation across 10 replications for each network depth using the Dice loss function. }
    \label{fig:synth_seg_diff_dice}
\end{figure}

The raw Dice results and paired differences for networks trained using the NLL loss are given in Figures \ref{fig:synth_seg_res_nll} and \ref{fig:synth_seg_diff_nll}, respectively. The same for the Dice loss are given in Figures \ref{fig:synth_seg_res_dice} and \ref{fig:synth_seg_diff_dice}. In general, the difference between the regularised models and the unregularised sigmoid was more pronounced for smaller models and when using the NLL loss function rather than the Dice loss function. For the former, it makes sense that smaller models would benefit more from the regularisation as it could help to overcome the relatively small receptive fields otherwise had by the final convolution layers. This would give them more of an ability to average over wider spatial areas to overcome noise.

In terms of comparing between loss functions, it is worth noting that the non-regularised network seems to perform better with the Dice loss function compared to the NLL (i.e. binary cross entropy) loss function, replicating much of the literature on the subject. However, the opposite was seen with the regularised models, whoses performances were very similar to the non-regularised network. This can be explained by the fact that the NLL's gradient has components that target the bias and regularisation terms separately. This separation means that it is much more capable of training the regularisation terms at earlier intervals before high-quality bias terms can be found, whereas for the Dice loss, there is a fundamental issue where if we have two adjacent pixels that should have the same label but one of which is mislabelled, the regularisation term will receive conflicting gradients. That is, the regularisation term will be encouraged to increase in order to ``move'' probability from the correctly labelled variable to the incorrectly labelled one, while also being encouraged to decrease in order to prevent the movement of probability from the incorrectly labelled variable to the correctly labelled one. In terms of comparison with the literature, the Dice improvements are similar to the 2\% increase previously reported \cite{arnab2016higher} although this is principally seen with pLBP and not MFE. In addition, with the Dice loss, the amount of increase becomes much more unstable as the complexity of the network increases, possibly because of the increased complexity of the loss landscape.

Interestingly, with the NLL loss function, the performance for pLBP seems to flatten out rather early with a 4-layer network, seeing minimal improvements, largely suggesting that once a certain minimum receptive field is reached, the regularisation is enough to reach a maximal level of performance. With the Dice loss, the depressed regularisation terms discussed above would imply a more similar growth for both the regularised and non-regularised models.

\section{Synthetic structured multi-region image segmentation}

We repeated this synthetic experiment for multi-region segmentation as well to test the efficacy of learning more complex multi-region segmentation problems with heterogenous requirements between regions. This experiment was structured in a very similar manner to the previously described binary segmentation simulation.

\subsection{Synthetic Data Generation}
The simulated images were created using a nested foreground object with three nested regions of different intensities, placed on a two-part background. This leads to the same hierarchical structure as we evaluated previously (Figure \ref{fig:syn_labelordering}. The simulation used a mix of low- and high-frequency noise with contrast-to-noise ratio of $10$ to reflect the intermediate contrasts between similar-intensity objects. An example image is given in Figure \ref{fig:sim_images_hier}.

\begin{figure}[h]
     \centering
     \begin{subfigure}[b]{0.32\textwidth}
        \includegraphics[width=0.95\textwidth]{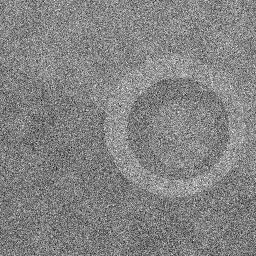}
     \end{subfigure}
     \begin{subfigure}[b]{0.32\textwidth}
        \includegraphics[width=0.95\textwidth]{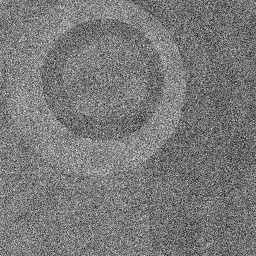}
     \end{subfigure}
     \begin{subfigure}[b]{0.32\textwidth}
        \includegraphics[width=0.95\textwidth]{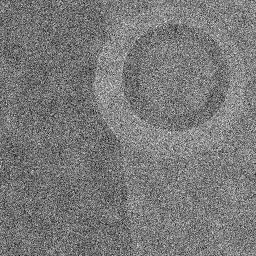}
     \end{subfigure}

    \vspace{2mm}
     
     \begin{subfigure}[b]{0.32\textwidth}
        \includegraphics[width=0.95\textwidth]{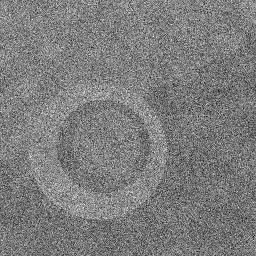}
     \end{subfigure}
     \begin{subfigure}[b]{0.32\textwidth}
        \includegraphics[width=0.95\textwidth]{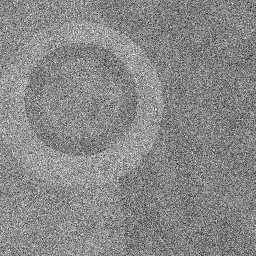}
     \end{subfigure}
     \begin{subfigure}[b]{0.32\textwidth}
        \includegraphics[width=0.95\textwidth]{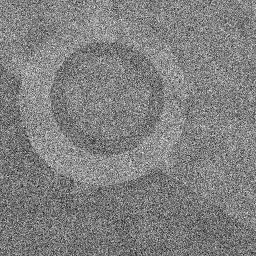}
     \end{subfigure}
     \caption{Example images with two-region background and three-nested-region foreground object.}
     \label{fig:sim_images_hier}
\end{figure}

\subsection{Networks and Loss Functions}
The networks used have the same structure as the previous experiment, shown in Figure \ref{fig:sim_network_structure}, although only with the first three depths. As with the previous simulation, the experiment was repeated 10 times for each network configuration with different random initialisations. This repetition allows us to gauge the robustness of any performance differences. The loss functions used were either the negative log-likelihood loss corresponding to the label ordering structure or the mean Dice loss function.

\begin{figure}[p]
    \centering
    \begin{subfigure}[b]{0.475\textwidth}
    \includegraphics[width=0.995\linewidth]{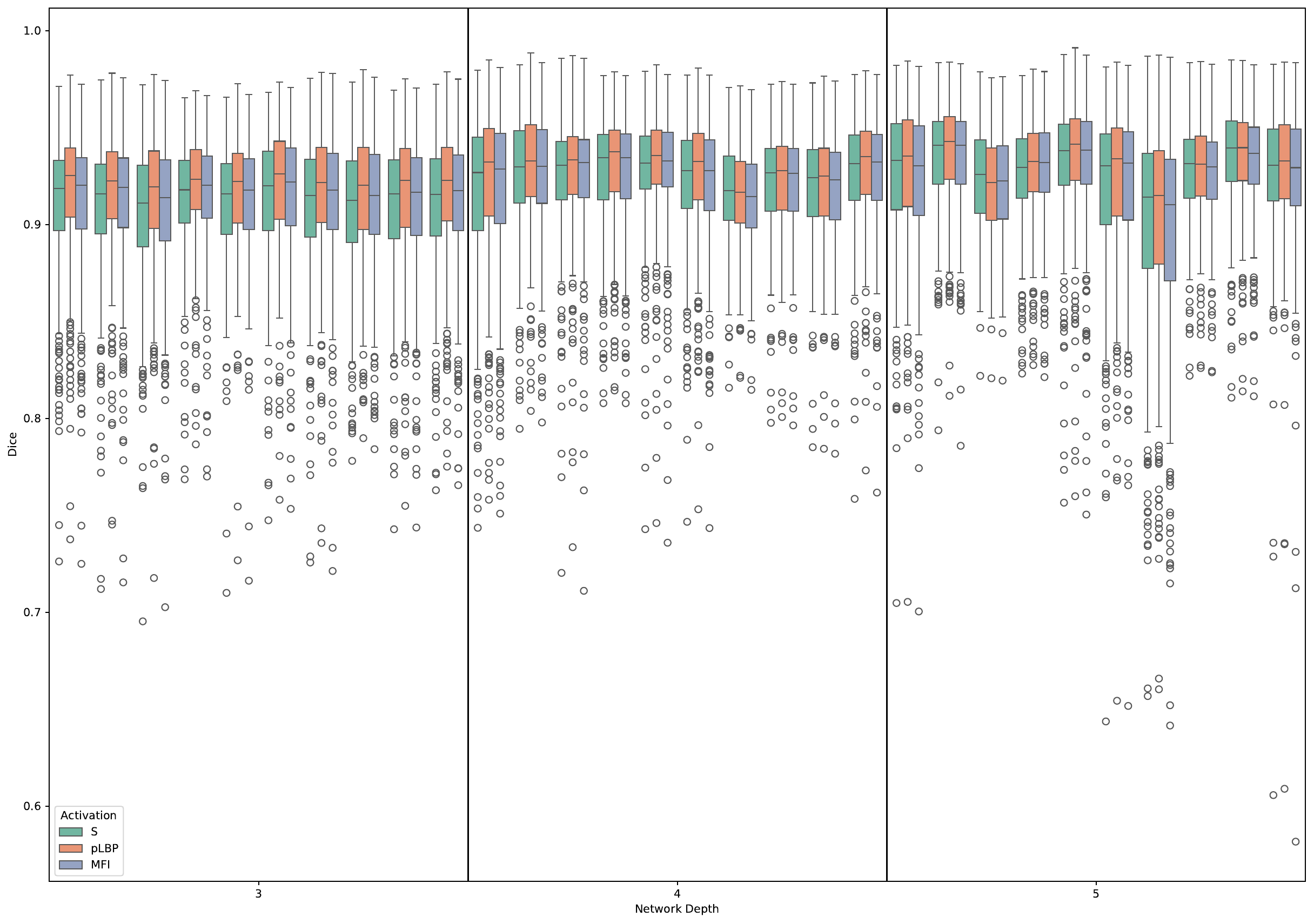}
    \end{subfigure}%
    \begin{subfigure}[b]{0.475\textwidth}
    \includegraphics[width=0.995\linewidth]{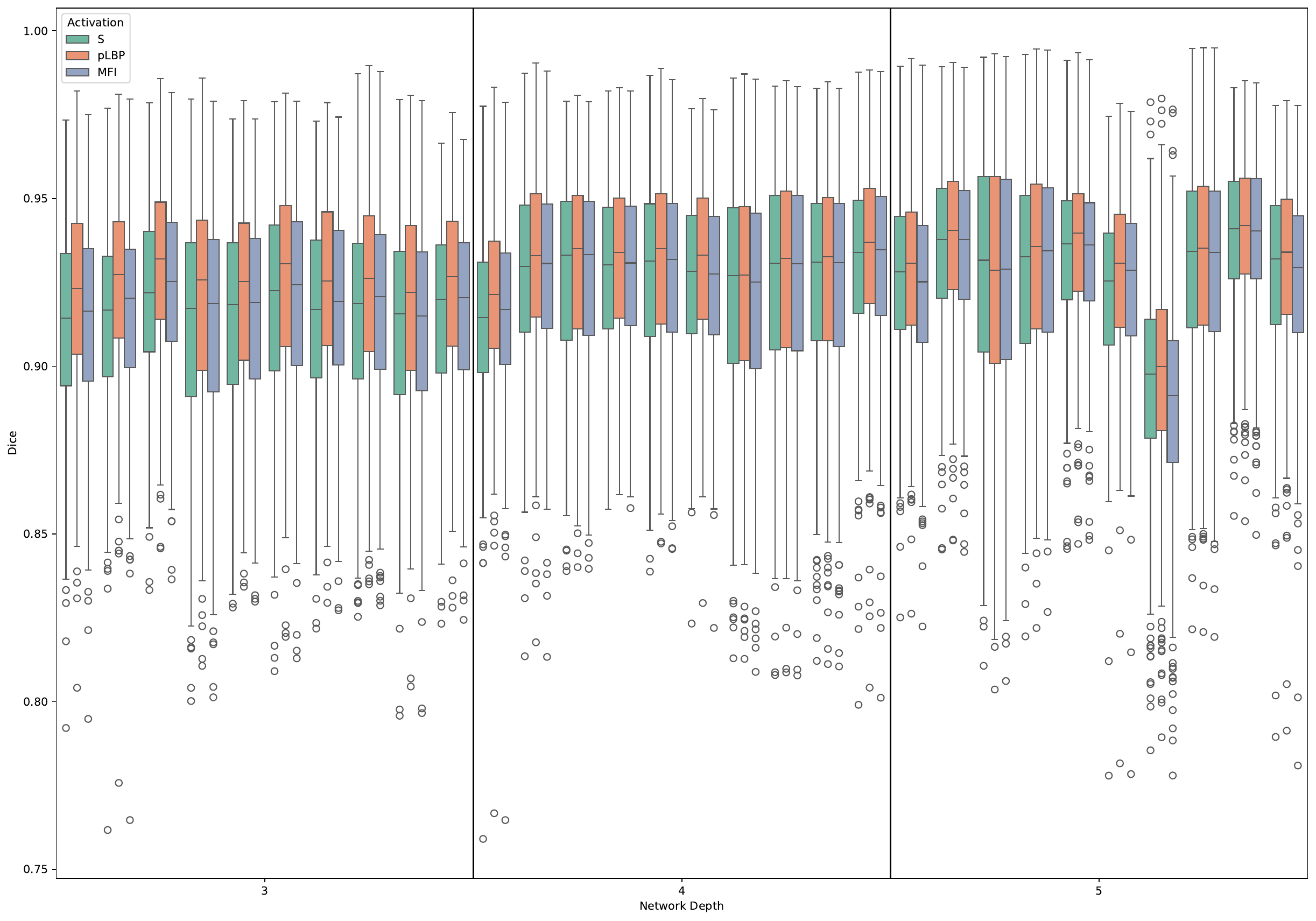}
    \end{subfigure}
    
    \begin{subfigure}[b]{0.475\textwidth}
    \includegraphics[width=0.995\linewidth]{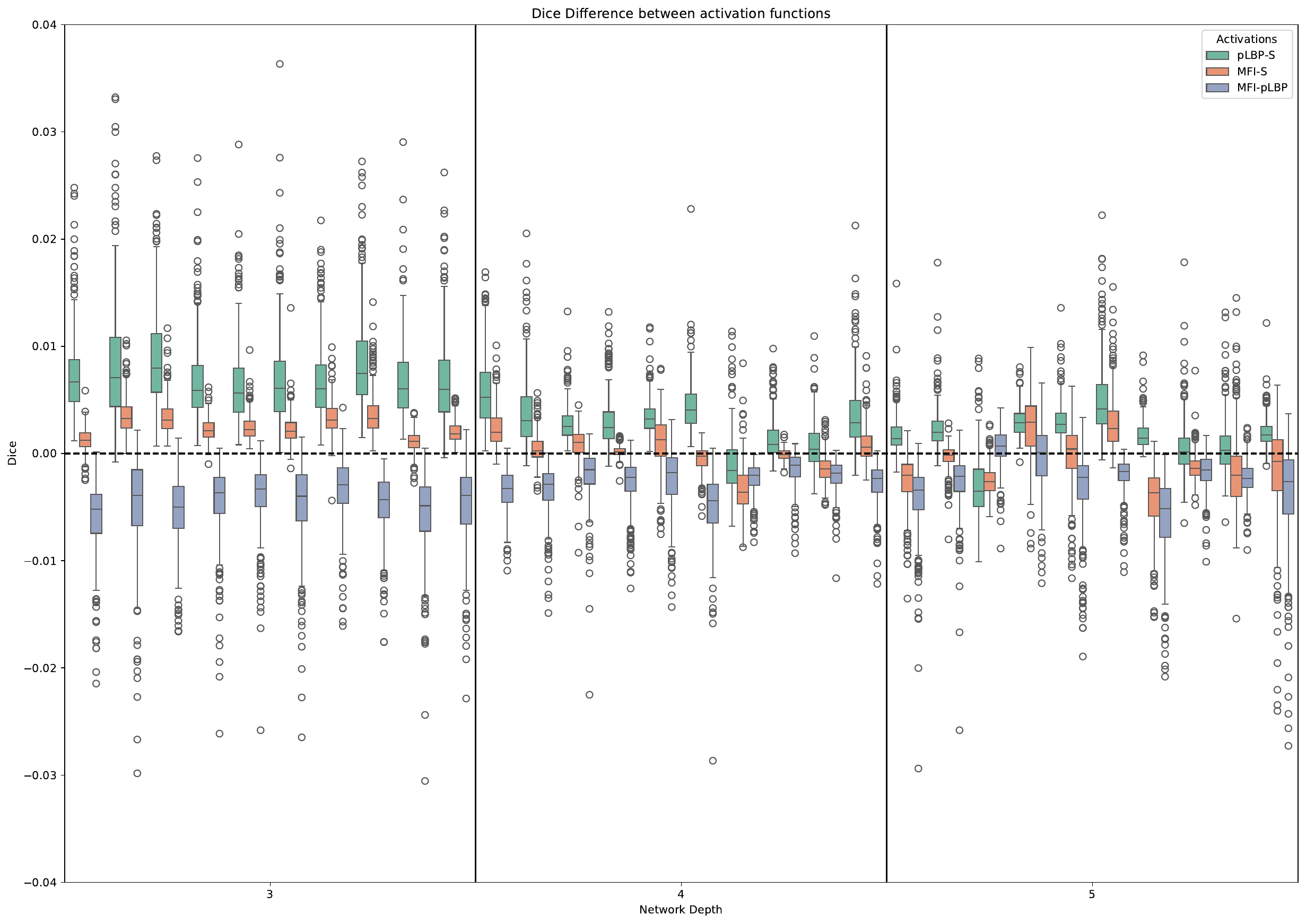}
    \end{subfigure}%
    \begin{subfigure}[b]{0.475\textwidth}
    \includegraphics[width=0.995\linewidth]{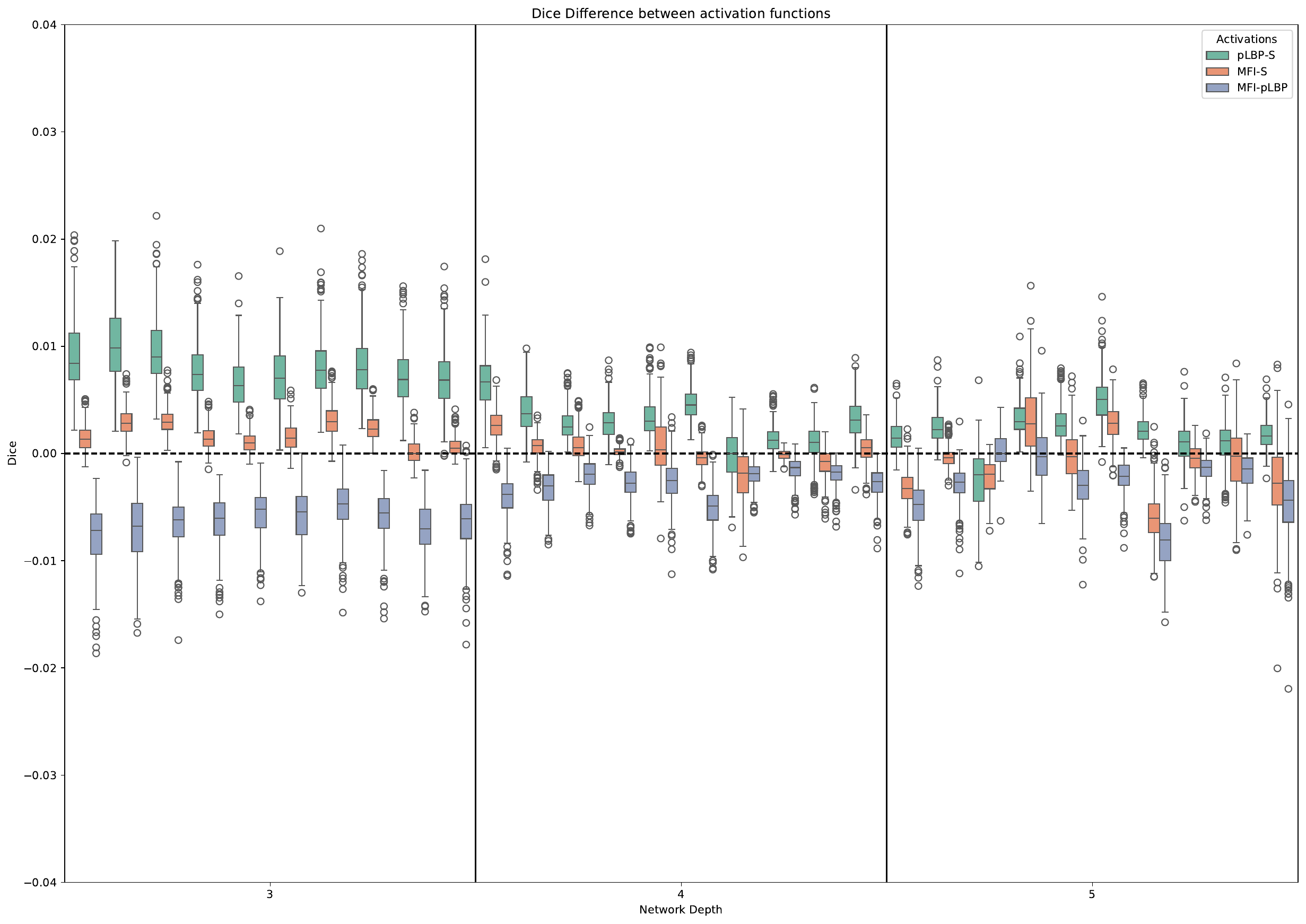}
    \end{subfigure}
    \caption{Dice scores (above) and differences (below) for the brighter (left) and darker (right) background component in the synthetic image segmentation across 10 replications for each network depth using the negative log-likelihood loss function. }
    \label{fig:synth_seg_res_mr_nll}
    
    \begin{subfigure}[b]{0.475\textwidth}
    \includegraphics[width=0.995\linewidth]{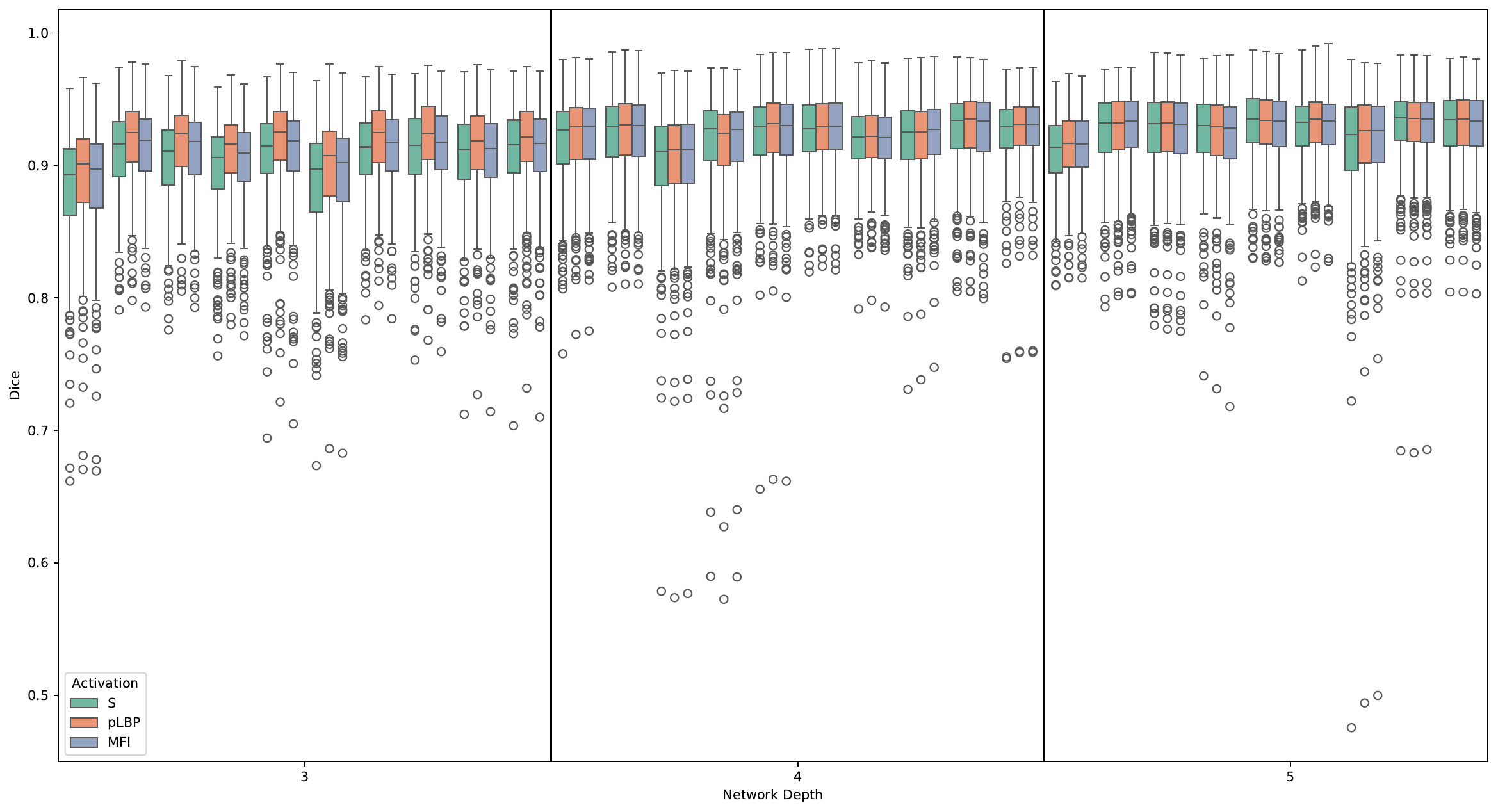}
    \end{subfigure}%
    \begin{subfigure}[b]{0.475\textwidth}
    \includegraphics[width=0.995\linewidth]{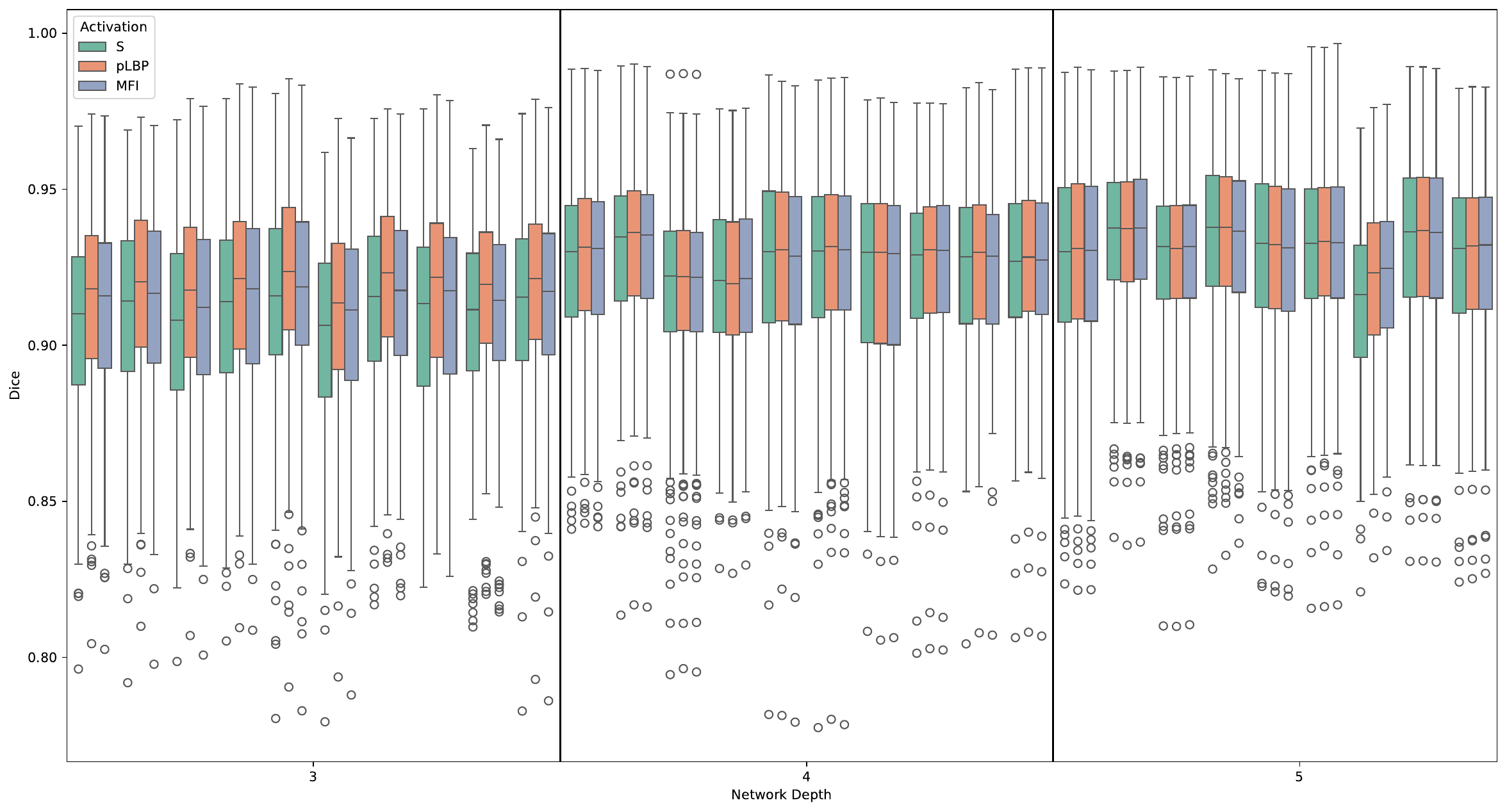}
    \end{subfigure}
    
    \begin{subfigure}[b]{0.475\textwidth}
    \includegraphics[width=0.995\linewidth]{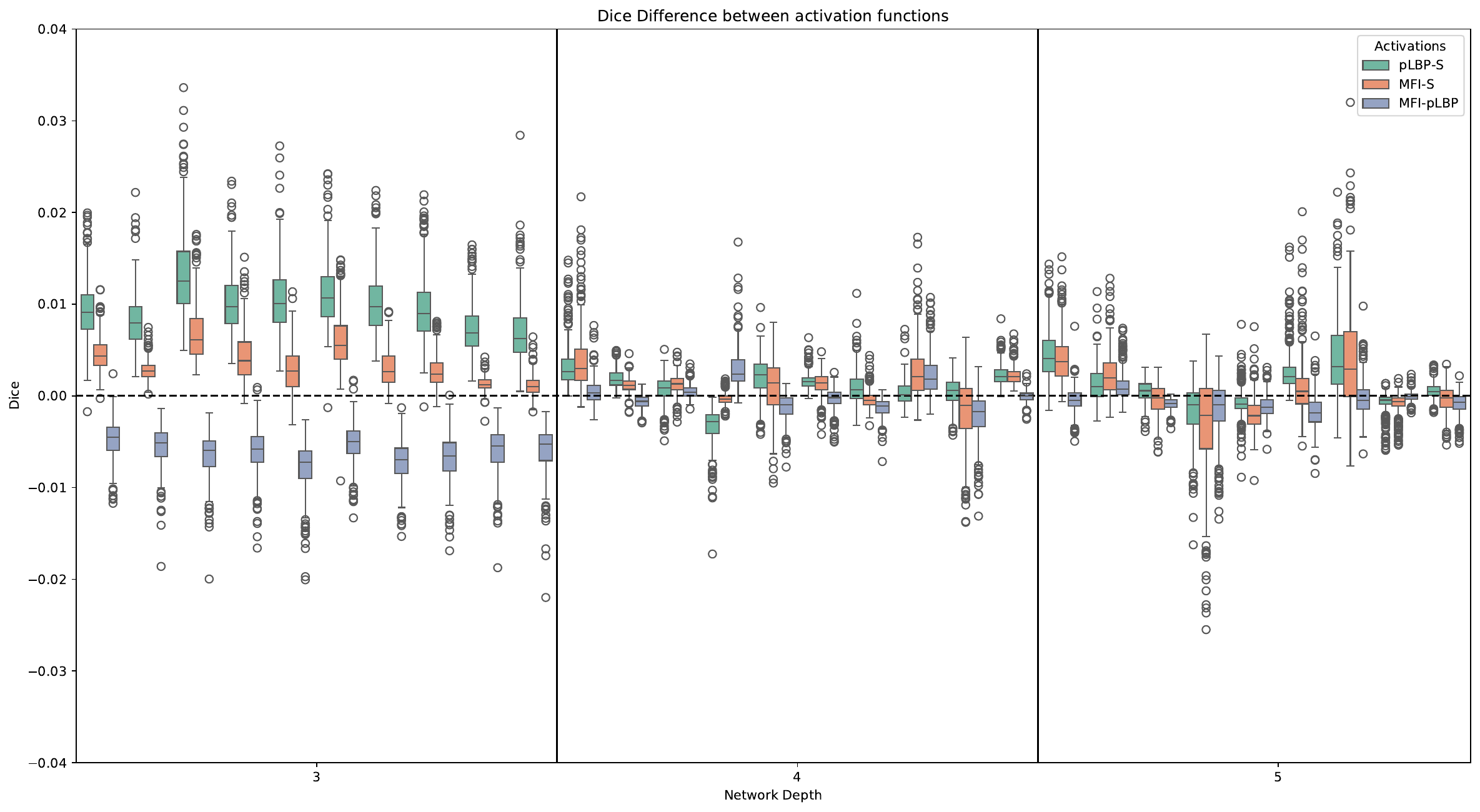}
    \end{subfigure}%
    \begin{subfigure}[b]{0.475\textwidth}
    \includegraphics[width=0.995\linewidth]{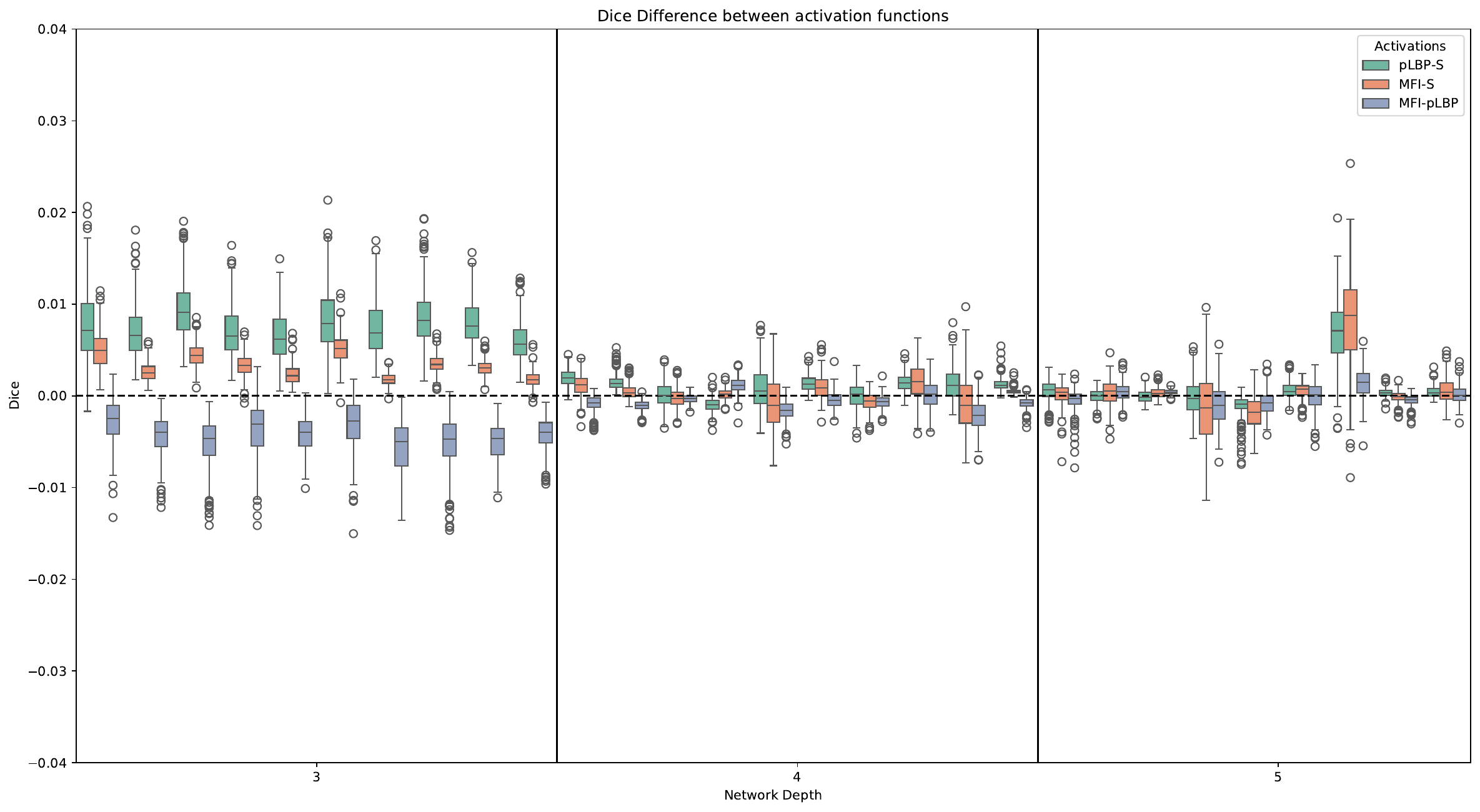}
    \end{subfigure}
    \caption{Dice scores (above) and differences (below) for the brighter (left) and darker (right) background component in the synthetic image segmentation across 10 replications for each network depth using the mean Dice loss function. }
    \label{fig:synth_seg_res_mr_dice}
\end{figure}

\begin{figure}[p]
    \centering
    \begin{subfigure}[b]{0.475\textwidth}
    \includegraphics[width=0.995\linewidth]{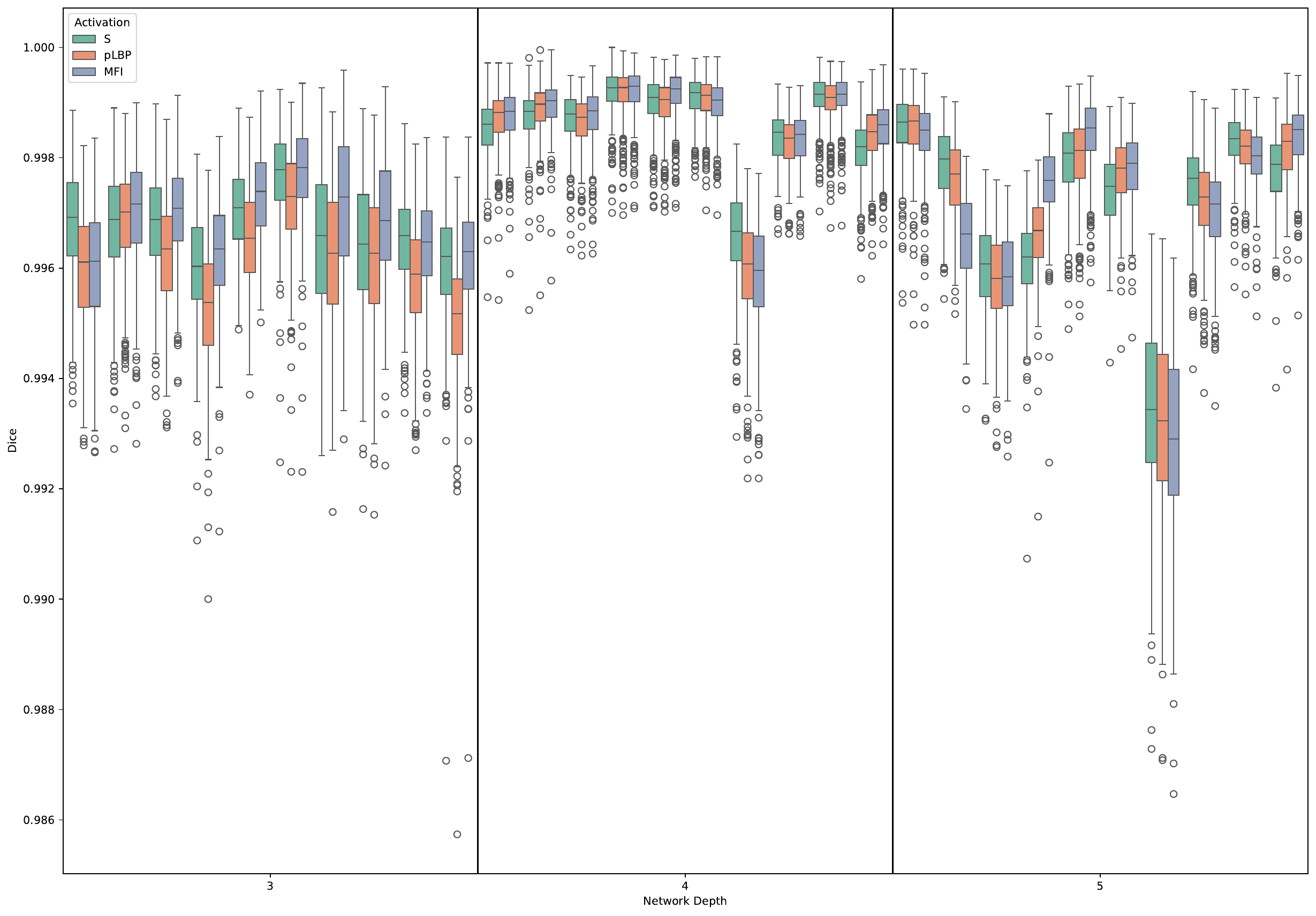}
    \end{subfigure}%
    \begin{subfigure}[b]{0.475\textwidth}
    \includegraphics[width=0.995\linewidth]{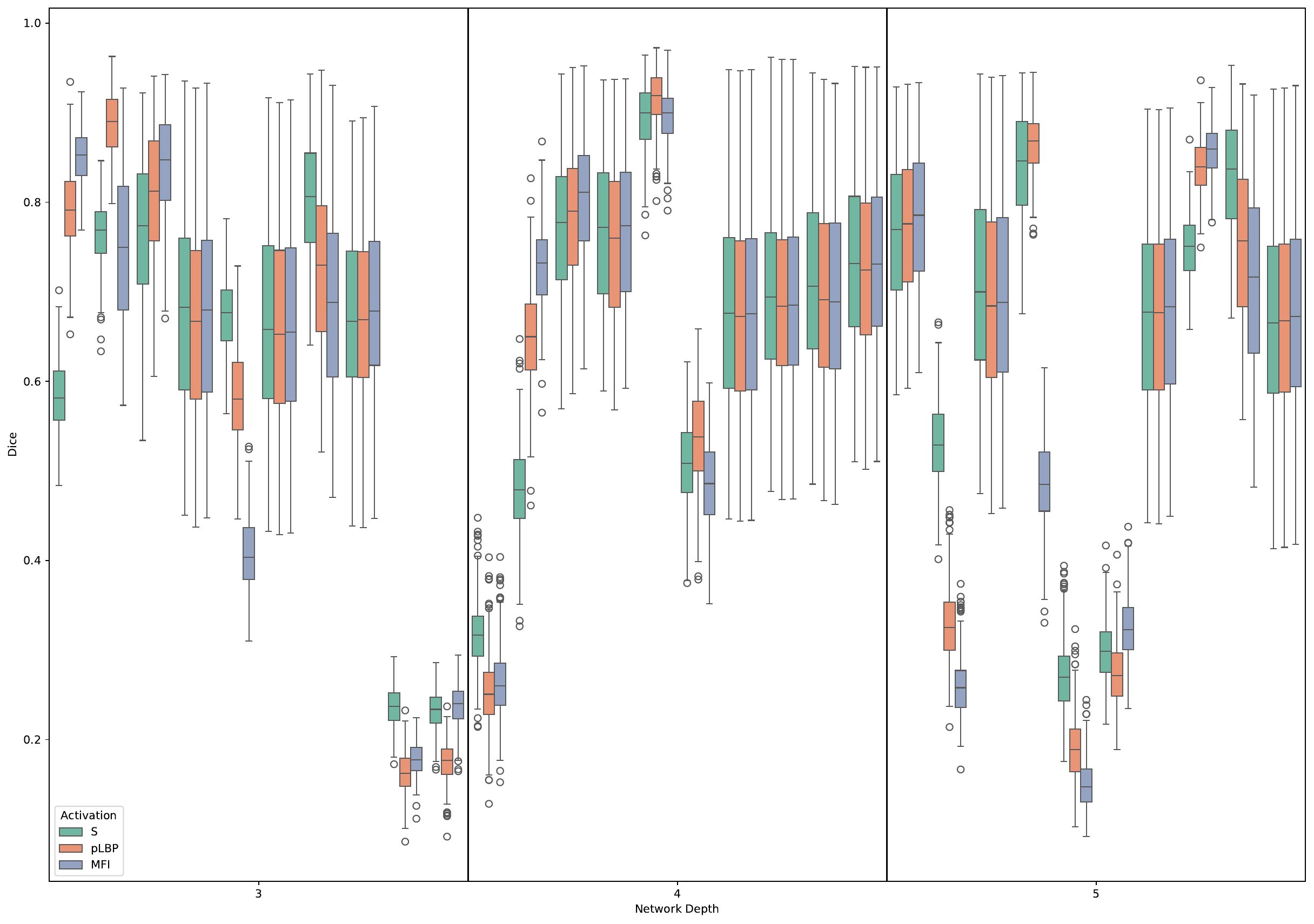}
    \end{subfigure}
    
    \begin{subfigure}[b]{0.475\textwidth}
    \includegraphics[width=0.995\linewidth]{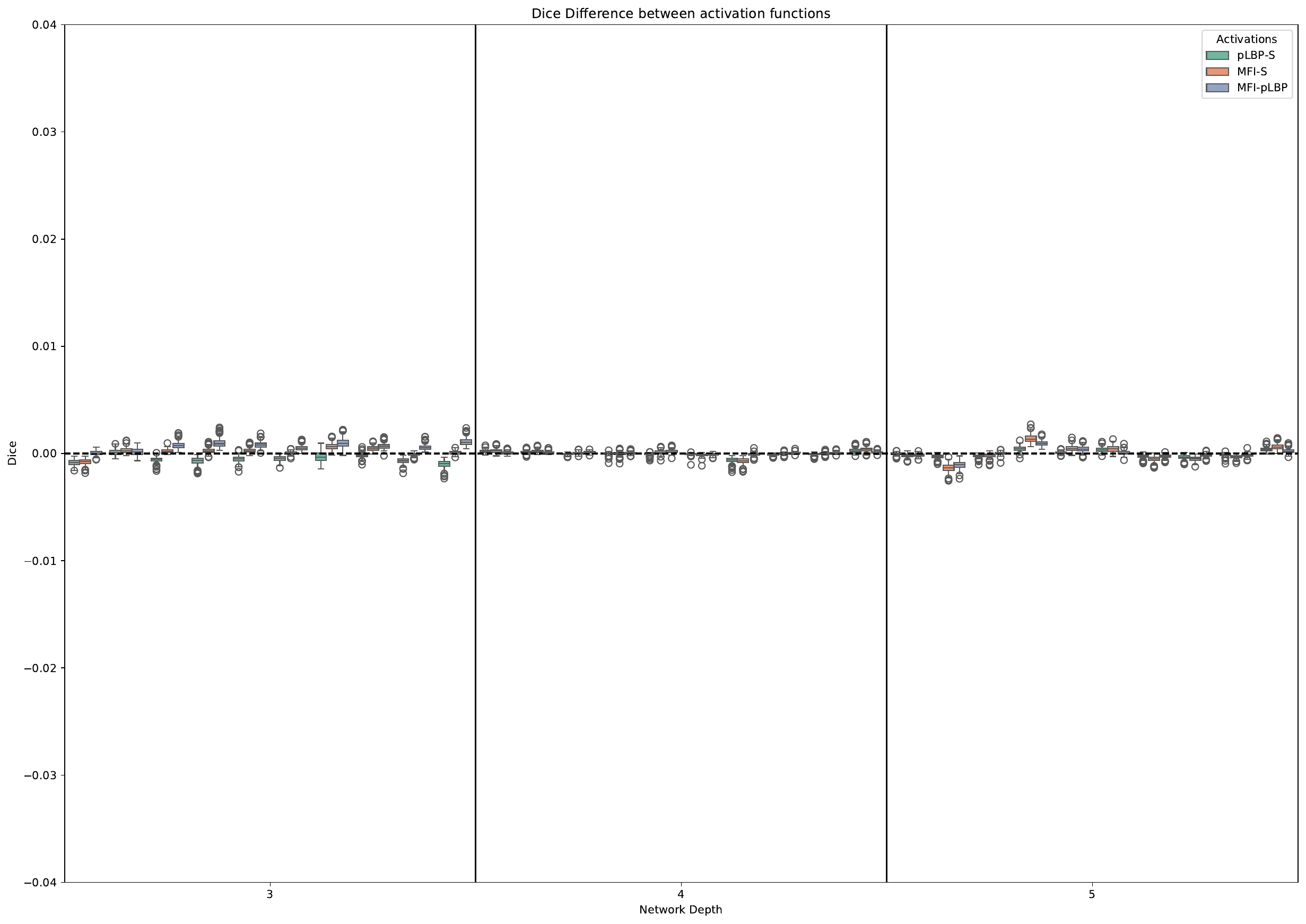}
    \end{subfigure}%
    \begin{subfigure}[b]{0.475\textwidth}
    \includegraphics[width=0.995\linewidth]{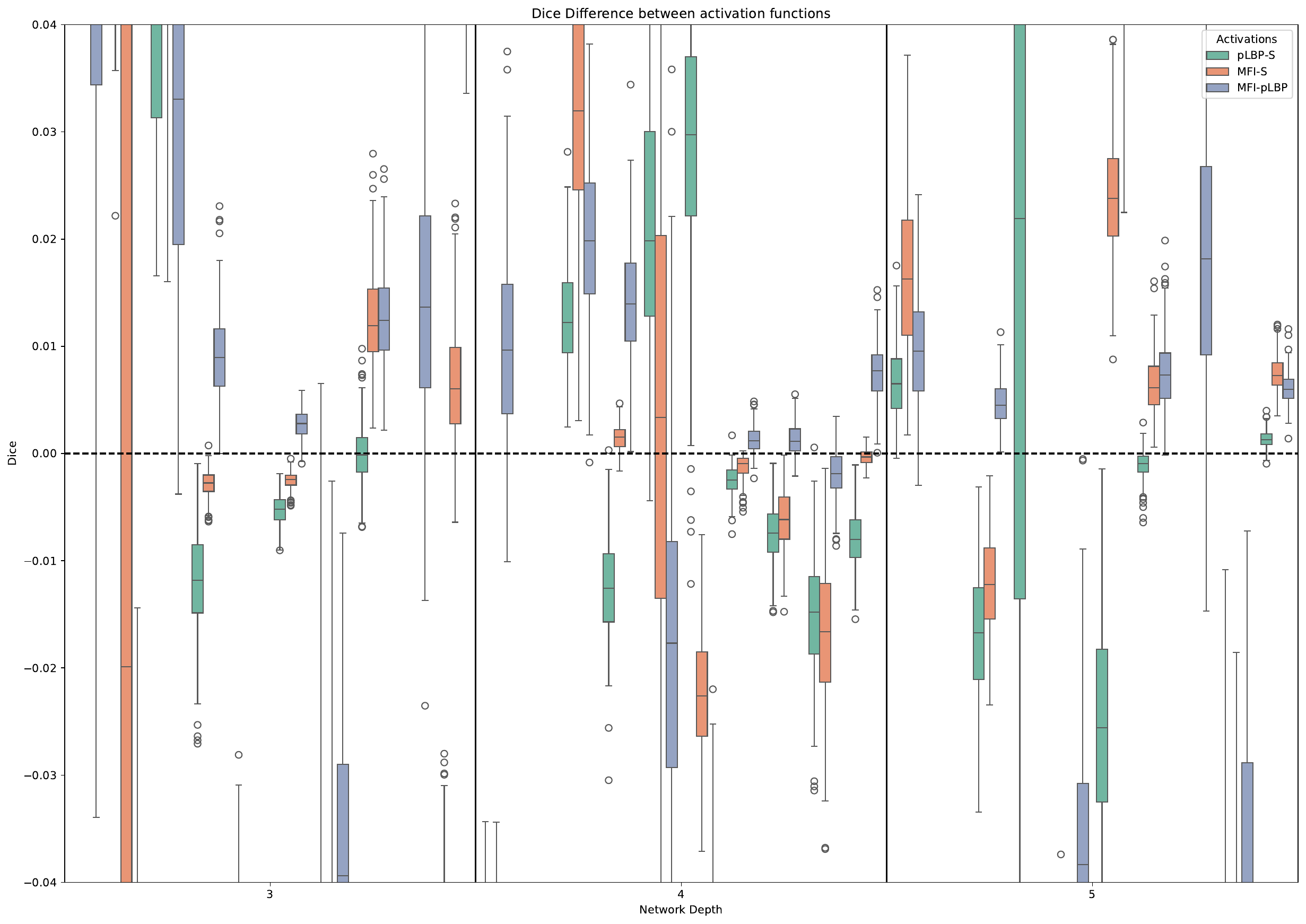}
    \end{subfigure}
    \caption{Dice scores (above) and differences (below) for the brightest and widest foreground component  (left) and for the darker and narrower foreground component  (right) in the synthetic image segmentation across 10 replications for each network depth using the negative log-likelihood loss function. }
    \label{fig:synth_seg_resf_mr_nll}
    
    \begin{subfigure}[b]{0.475\textwidth}
    \includegraphics[width=0.995\linewidth]{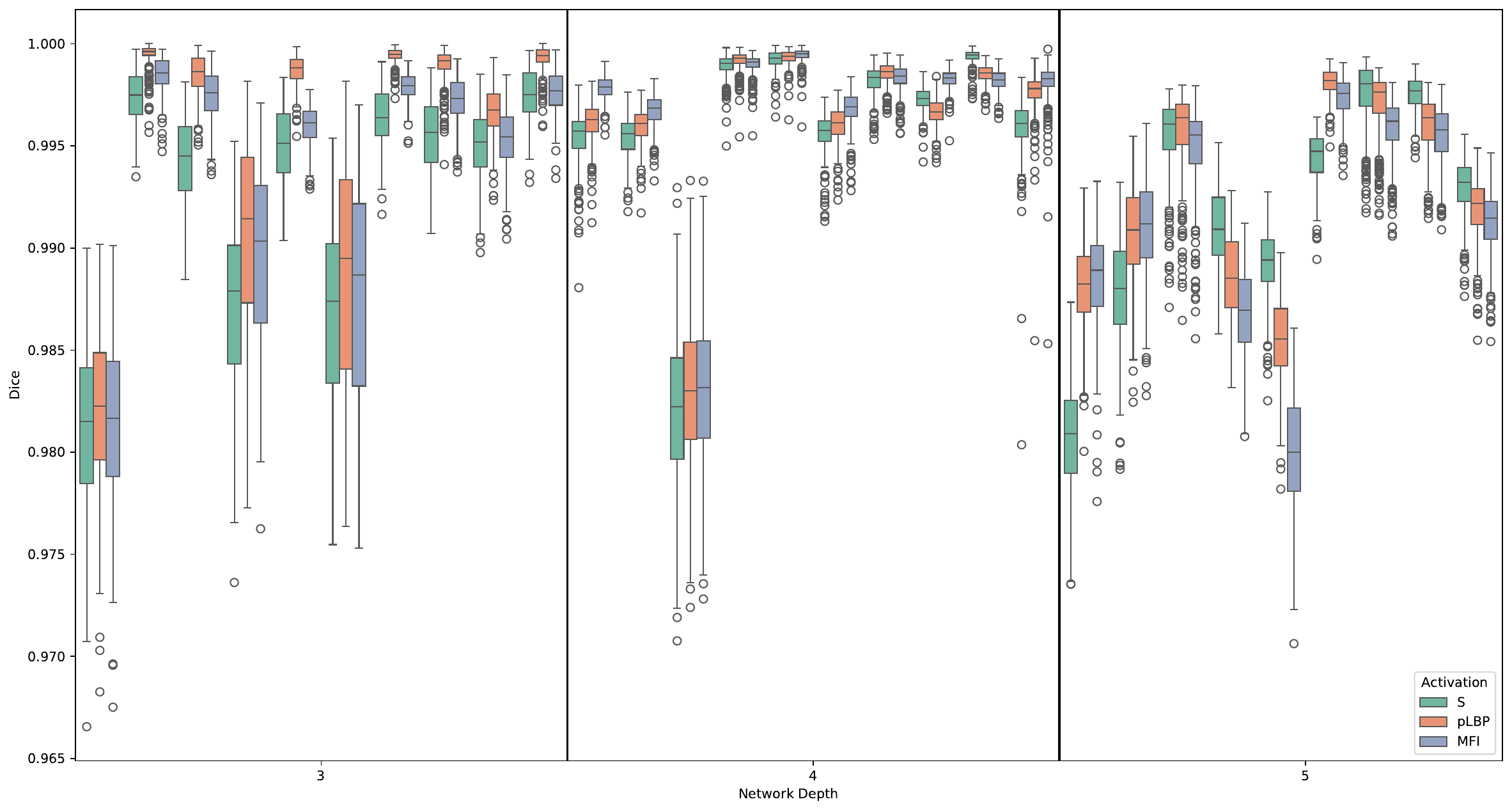}
    \end{subfigure}%
    \begin{subfigure}[b]{0.475\textwidth}
    \includegraphics[width=0.995\linewidth]{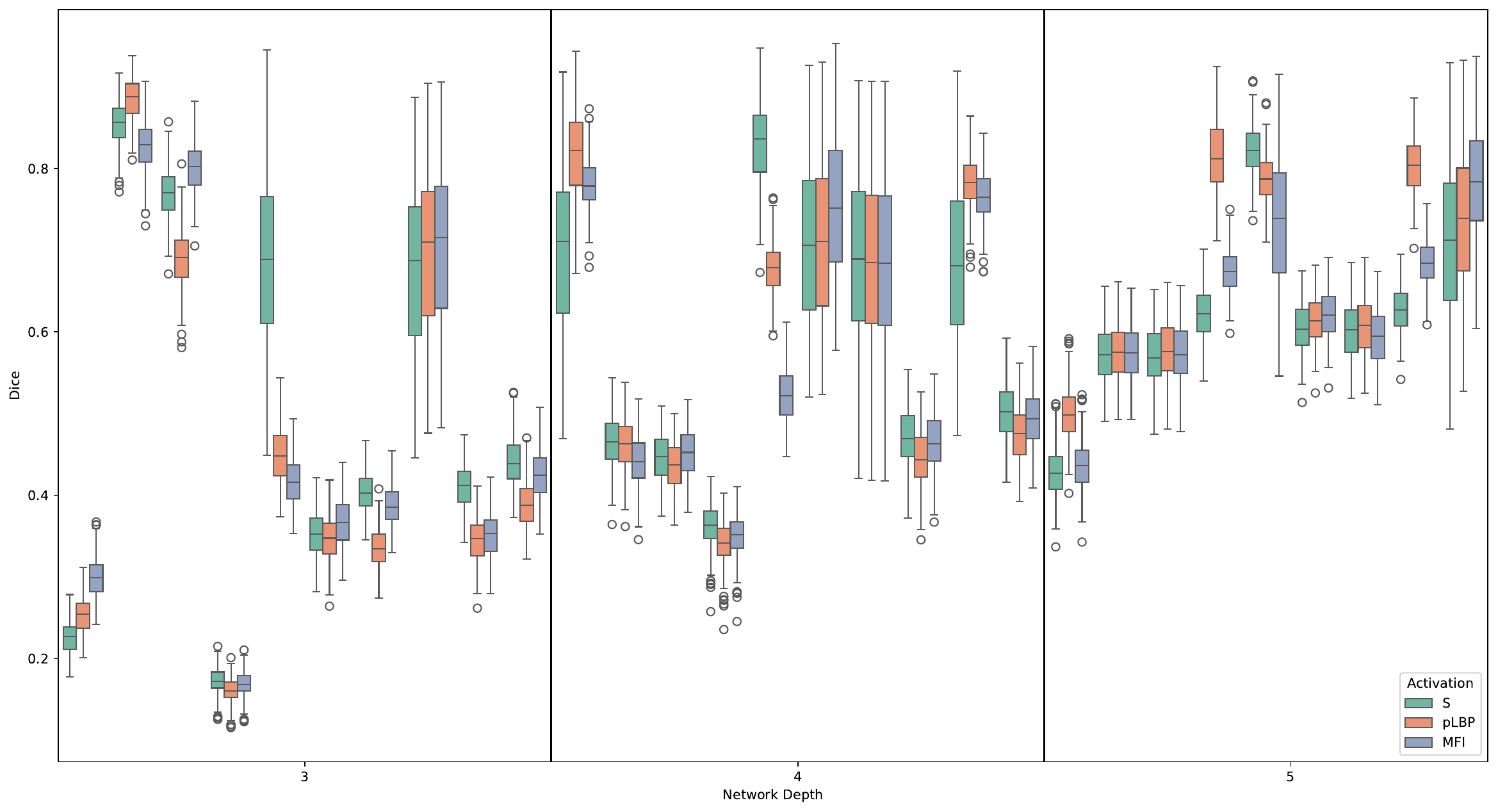}
    \end{subfigure}
    
    \begin{subfigure}[b]{0.475\textwidth}
    \includegraphics[width=0.995\linewidth]{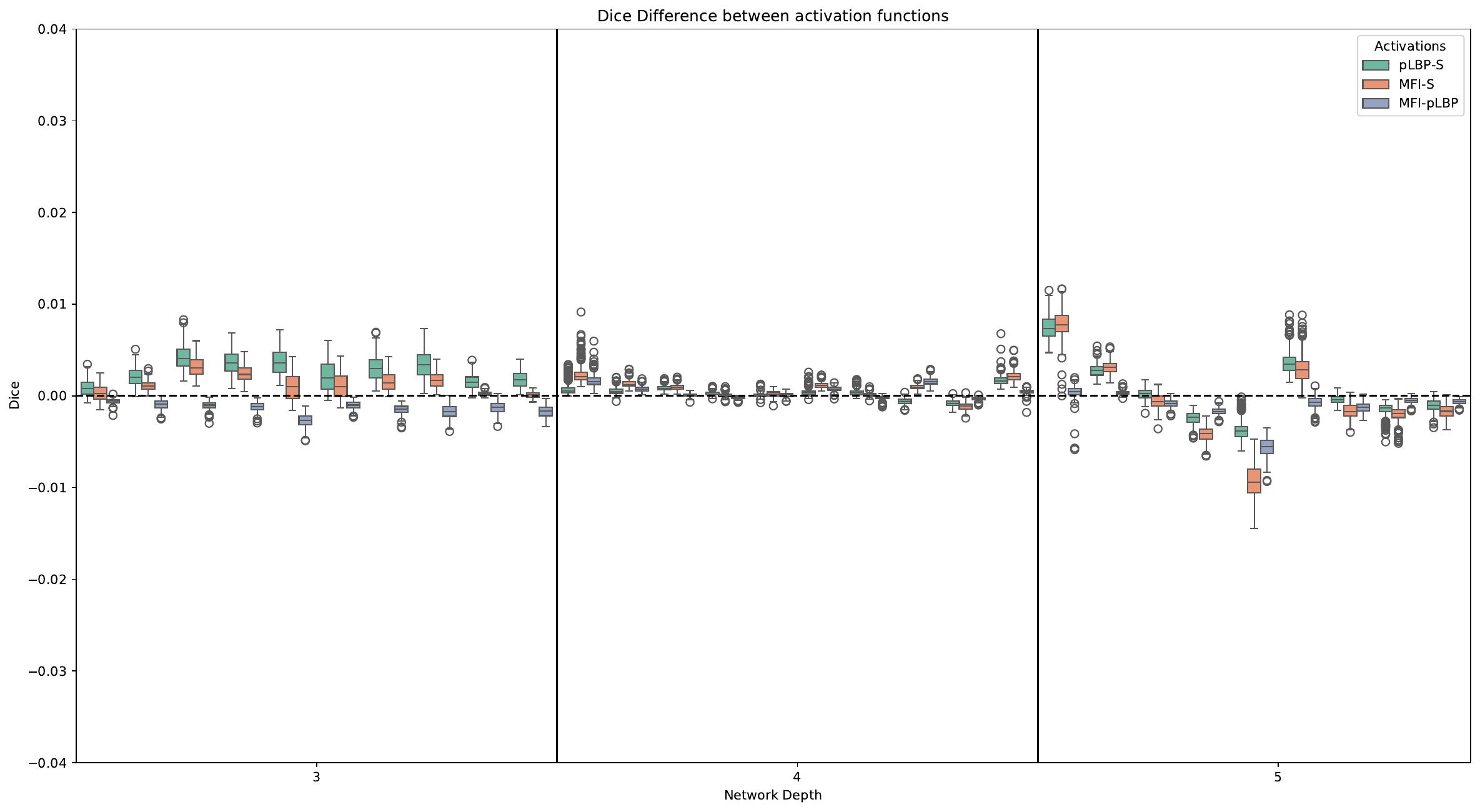}
    \end{subfigure}%
    \begin{subfigure}[b]{0.475\textwidth}
    \includegraphics[width=0.995\linewidth]{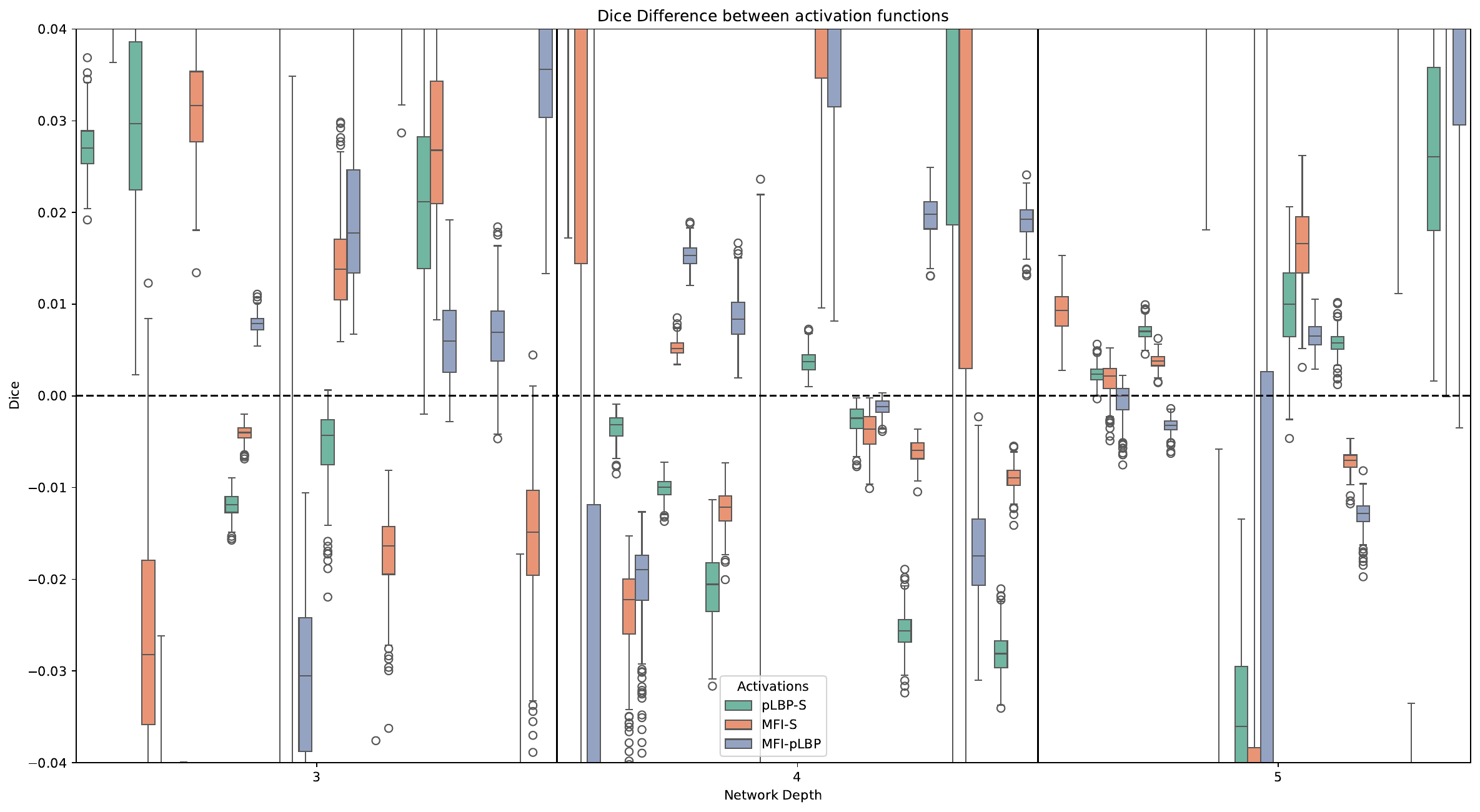}
    \end{subfigure}
    \caption{Dice scores (above) and differences (below) for the brightest and widest foreground component  (left) and for the darker and narrower foreground component  (right) in the synthetic image segmentation across 10 replications for each network depth using the mean Dice loss function. }
    \label{fig:synth_seg_resf_mr_dice}
\end{figure}

\subsection{Results}

The reduced noise meant that the overall performance started off significantly higher than in the binary case. This allowed for certain dynamics to play out earlier in terms of the network size, which was intended as to minimise the computational requirements of the simulation as a whole. It should also be noted that the regularisation limits are significantly more stringent for multi-label problems, making the performance generally more similar to the unregularised softmax/argmax-based comparator.

Due to the different structure sizes, it is unsurprising that different structures display different dynamics. The background structures, which have regularisation requirements closer to those of the binary segmentation experiment, showed very similar results with regularisation improving performance, particularly at lower network depths as shown in Figures \ref{fig:synth_seg_res_mr_nll} and \ref{fig:synth_seg_res_mr_dice}.

For the foreground components where regularisation requirements were more heterogeneous, the results were highly variable as well  as shown in Figures \ref{fig:synth_seg_resf_mr_nll} and \ref{fig:synth_seg_resf_mr_dice}. For the largest, brightest region, the data term was generally sufficient for high quality segmentation, although this was more unstable for the mean Dice loss function which was significantly more sensitive to smaller changes in the dark narrow strip in the foreground which could be missed almost entirely by all networks. Due to this variability, the effect of regularisation on these structures is difficult to tell, although it does suggest that it does not have a large effect, if any, on training stability.

%% file: 5_discussion.tex
\section{Discussion}
This article answers a question posed by previous researchers investigating the integration of pairwise Markov random fields and deep learning: the parameters under which a model is guaranteed to converge \cite{arnab2016higher}. Being able to quickly and easily identify this space is particularly important for deep learning as a model would need to be automatically constrained to said space during training time to prevent instability in order to prevent a single training step (such as generated NaN values) from infecting the entire model. 

Although previous research has been largely focused on mean-field iteration as the mechanism for estimating marginal probabilities in these models \cite{zheng2015conditional,arnab2016higher}, we have found that the more mathematically rigorous loopy belief propagation has strictly better convergence properties that are also bourne out in synthetic experiments on marginal probability accuracy and overall model learning.

We have also found that the use of these models with the appropriate negative log-likelihood loss function allows for smaller models to reach the same level of performance as larger ones, indicating that this regularisation scheme may be most useful when training data sizes are limited. This may however depend on the quality of the annotations (which is not an issue for synthetic data) as the gradient of the negative log-likelihood loss with respect to the regularisation terms is agnostic to where ``erroneous edges'' are placed rather than giving a smoother landscape in which edges closer to the ground truth ones are penalised less. Just as significant research has been performed on more nuanced loss functions for the interior pixels (such as the Dice loss investigated in our simulation), more adapted loss functions that take into account predicted and ground-truth edges may improve this behaviour for particular scenarios \cite{jadon2020survey}.

\subsection{Limitations}
Although the use of end-to-end deeply learned CRF models can significantly improve performance in low-memory and low-data settings, it does so at the expense of time. Although the number of iterations performed is variable, each is undoubtedly more complex than a single instantiation of an element-wise non-linear activation such as softmax or sigmoid. This gap between time and performance can be bridged through making the network more complex, as shown in our synthetic experiments, although this generally still leaves traditional activation functions with an advantage in terms of time, provided it has sufficient memory resources to handle the increased number of layers. Increasing network depth, however, implies a much larger number of weights, which in turn require more data, making regularised activation functions more appropriate in data-limited domains such as medical image segmentation \cite{salpea2022medical}.

\subsection{Future work}
One of the benefits of CRF regularisation is that there is a large amount of literature in the graph-cuts community which can furnish maximum \textit{a posteriori} solvers that can be paired with differentiable marginal estimators. In particular, pLBP allows for certain types of constraints, notably geometrical ones such as star-convexity \cite{veksler2008star,baxter2016shape}, to be incorporated which can be an important source of high-level knowledge that can be difficult for machine learning algorithms to learn, let alone recognise as a hard constraint. The current issue with such constraints is ensuring the convergence of the forward pass and existence of the derivative as the current proof structures rely on the $p_1$ and $p_\infty$ induced matrix norms being sub-unit which is trivially not the case in the presence of hard constraints which allow for ones to exist as elements in these matrices. Proofs based solely on this local information are necessarily going to fail as such a perspective cannot distinguish these constraints forming a cycle versus being valid star-convexity priors which cannot. Therefore, more complex proofs involving spectral graph theory, which can take a more global view of the constraints, are likely needed to put bounds on the regularisation terms implying convergence for models equipped with star-convexity priors.

\section{Conclusions}
This article presents some updated trivially parallelisable linear-time computable bounds for the convergence of parallel algorithms for estimating the marginal distribution as well as computing the derivative of said estimation in a family of pairwise Markov random fields that represent regularised node labelling problems in undirected graphs, that is, where two nodes being connected generally increases their probability of having the same label. These models represent a common, simple regularisation framework for semantic segmentation, especially in medical images where their can represent basic anatomical knowledge. The goal of having bounds that can be verified in constant-time using parallel computing  is to allow for them to be very quickly confirmed in the context of deep learning where such algorithms may need to be run a very large number of times for network training on large graphs (such as ones created from images under adjacency relationships) and thus efficiency is crucial.

%% file: A1_plbp_1.tex
\section{Supplement for Equation \ref{eq:lbp_e1} (Constraint using the $p_1$ induced Matrix Norm for Parallel Loopy Belief Propagation)}\label{app:lbp_e1}
The main part of Eq. \ref{eq:lbp_e1} on the fact that:
$$\underset{r_{[n,r_0]}(y_1\to e_1,L_k)}{
\mbox{max}}\sum_{L_j} r_{[n,r_0]}(y_1\to e_1,L_j)\left|\frac{
    p(L_1,L_j,e_1) - \sum_{L_k} r_{[n,r_0]}(y_1\to e_1,L_k) p(L_1,L_k,e_1)}{\sum_{L_k} r_{[n,r_0]}(y_1\to e_1,L_k) p(L_1,L_k,e_1)} \right|$$
takes on value
\begin{equation*}
        \frac{2(1-\sqrt{p_m})^2}{1-p_m} \mbox{ where } p_m = \underset{L_k}{\mbox{min}}\;p(L_1,L_k,e_1)
\end{equation*}
For the sake of notation, let us define:
\begin{equation*}
\begin{aligned}
    p_i &= p(L_1,L',e_1) \mbox{ ordering $L'$ such that } p_1 \geq p_2 \geq ... \geq p_m \\
    r_i &= r_{[n,r_0]}(y_1\to e_1,L') \mbox{ in the same order} \\
    \bar{p} &= \sum_i r_i p_i
\end{aligned}
\end{equation*}
and note that $p_i > 0$ and $p_1 = 1$ (and therefore $0<\bar{p}\leq1$) by the set-up of our model, and note that $r_i \geq 0$, and $\sum_i r_i = 1$ by the normalisation performed in the forward pass.

Without loss of generality, assume $p_1 \neq p_m$. For any given $r$, we can partition the our labels into two non-empty sets, specifically between those where $r_i \geq \bar{p}$ and those where $r_i < \bar{p}$, at a particular index $N$.
\begin{equation*}
\begin{aligned}
    E = & \sum_{i \leq N} \frac{r_i \left( 
    p_i - \bar{p} \right)}{\bar{p}} + \sum_{i > N} \frac{r_i \left(
    \bar{p} - p_i \right)}{\bar{p}} \\
    = & \frac{1}{\bar{p}}\left(\sum_{i \leq N} \left(
    r_i p_i - \sum_{j \leq N} r_i r_j p_j - \sum_{j > N} r_i r_j p_j \right) + \sum_{i > N} \left(
    \sum_{j \leq N} r_i r_j p_j + \sum_{j > N} r_i r_j p_j - r_i p_i \right)\right) \\
\end{aligned}
\end{equation*}
By differentiating this with respect to $r_i$ while holding $\bar{p}$ constant, we get two classes of derivatives:
\begin{equation*}
\begin{aligned}
\frac{\delta E}{\delta r_i} = & \left\{ \begin{array}{ll}
     \frac{\left((1-2r_i)p_i - \underset{j \leq N, j\neq i}{\sum} r_j(p_i+p_j) + \underset{j > N}{\sum} r_j(p_i-p_j)\right)\bar{p} - p_i\bar{p}E}{\bar{p}^2} & \mbox{ if } i \leq N \\
     \frac{\left((2r_i-1)p_i + \underset{j > N, j\neq i}{\sum} r_j(p_i+p_j) - \underset{j \leq N}{\sum} r_j(p_i-p_j)\right)\bar{p}-p_i\bar{p}E}{\bar{p}^2} & \mbox{ if } i > N
\end{array}\right. \\
\bar{p} \frac{\delta E}{\delta r_i} = & \left\{ \begin{array}{ll}
     p_i(1-E) - \underset{j \leq N}{\sum} r_j(p_i+p_j) + \underset{j > N}{\sum} r_j(p_i-p_j) & \mbox{ if } i \leq N \\
     p_i(-1-E) + \underset{j > N}{\sum} r_j(p_i+p_j) - \underset{j \leq N}{\sum} r_j(p_i-p_j) & \mbox{ if } i > N
\end{array}\right.
\end{aligned}
\end{equation*}
Considering the difference between $\frac{\delta E}{\delta r_i}$ where $1 < i \leq N$ and $\frac{\delta E}{\delta r_1}$ noting that they are of the same class.
\begin{equation*}
\begin{aligned}
    \bar{p}\left(\frac{\delta E}{\delta r_1} - \frac{\delta E}{\delta r_i} \right) = & 
     \left(p_1 - p_i\right)(1-E) - \underset{j \leq N}{\sum} r_j(p_1+p_j-p_i-p_j) + \underset{j > N}{\sum} r_j(p_1-p_j-p_i+p_j) \\
     = &
     \left(p_1 - p_i\right)(1-E) - \underset{j \leq N}{\sum} r_j(p_1-p_i) + \underset{j > N}{\sum} r_j(p_1-p_i) \\ 
     = &
     (p_1 - p_i) \left(1 - E - \underset{j \leq N}{\sum} r_j + \underset{j > N}{\sum}r_j\right) \\
     = &
     (p_1 - p_i) \left(1 - \underset{j \leq N}{\sum} r_j\left(1+\frac{p_i-\bar{p}}{\bar{p}}\right) + \underset{j > N}{\sum}r_j\left(1-\frac{\bar{p}-p_i}{\bar{p}} \right) \right) \\
     = &
     (p_1 - p_i) \left(1 - \underset{j \leq N}{\sum}\frac{r_i p_i}{\bar{p}} + \underset{j > N}{\sum}\frac{r_j p_i}{\bar{p}} \right) \geq 0\\
\end{aligned}
\end{equation*}
as $p_1 \geq p_i$, $\underset{j \leq N}{\sum}\frac{r_i p_i}{\bar{p}} \leq \underset{j}{\sum}\frac{r_i p_i}{\bar{p}} = 1$ and $\underset{j > N}{\sum}\frac{r_j p_i}{\bar{p}} \geq 0$.
Thus, any distribution of weight $r_i, i \leq N$ towards $r_1$ results in equal or higher values of $E$ regardless of $N$. A similar calculation can be done for distributing $r_i, i > N$ towards $r_m$. Therefore, the maximum must occur when all the weight is centralised in these two labels, i.e. $r_1+r_m = 1$ and $r_i = 0$ else. This simplifies our formula to:
\begin{equation*}
\begin{aligned}
    E = & \frac{r_1 \left(
    p_1 - r_1 p_1 - (1-r_1)p_m \right) + (1-r_1) \left(
    r_1 p_1 + (1-r_1)p_m - p_m \right)}{r_1 p_1 + (1-r_1)p_m}  \\
    = & \frac{r_1 \left(
    1 - r_1 - (1-r_1)p_m \right) + (1-r_1) \left(
    r_1 + (1-r_1)p_m - p_m \right)}{r_1 + (1-r_1)p_m}  \\
    =  & \frac{2 r_1 (1-r_1) (1-p_m)}{r_1+(1-r_1)p_m} \\
\end{aligned}
\end{equation*}
differentiating this with respect to $r_1$ yields:
\begin{equation*}
\begin{aligned}
\frac{dE}{dr_1} = & \frac{2(1-p_m)(1-2r_1)(r_1+(1-r_1)p_m) - 2r_1(1-r_1)(1-p_m)^2}{(r_1+(1-r_1)p_m)^2} \\
0 = & 2(1-p_m)(1-2r_1)(r_1+(1-r_1)p_m) - 2r_1(1-r_1)(1-p_m)^2 \\
0 = & (1-2r_1)(r_1+(1-r_1)p_m) - r_1(1-r_1)(1-p_m) \\
0 = & (-2(1-p_m)+1-p_m) r_1^2 + (1-p_m-2p_m-1+p_m)r_1 + p_m \\
0 = & (1-p_m) r_1^2 + 2p_m r_1 - p_m \\
r_1 = & \frac{-2p_m + \sqrt{4p_m^2+4(1-p_m)p_m} }{2(1-p_m)} \mbox{ is the only non-negative root}\\
r_1 = & \frac{\sqrt{p_m} - p_m }{1-p_m} \;\;\;\left(\mbox{and therefore } r_m = \frac{1-\sqrt{p_m}}{1-p_m} \right)\\
\end{aligned}
\end{equation*}
This gives us the desired distribution given $p_1 \neq \min p_m$. In the case where $\max_i p_i = \min_i p_i$, note that all values of $p_i$ are equal to $\bar{p}$ and thus $E=0$ regardless of $r$ so the desired distribution can still be considered maximal. 

Substituting these values back into the above formula allows us to compute the maximum value of $E$ as:
\begin{equation*}
\begin{aligned}
    E = & \frac{2 r_1 (1-r_1) (1-p_m)}{r_1+(1-r_1)p_m} \\
     = & \frac{2 \frac{\sqrt{p_m}-p_m}{1-p_m} \frac{1-\sqrt{p_m}}{1-p_m} (1-p_m)}{\frac{\sqrt{p_m}-p_m}{1-p_m}+\frac{1-\sqrt{p_m}}{1-p_m}p_m} \\
     = & \frac{2 (\sqrt{p_m}-p_m)(1-\sqrt{p_m})}{\sqrt{p_m}-p_m+(1-\sqrt{p_m})p_m} \\
     = & \frac{2 (1-\sqrt{p_m})^2}{1-p_m} \\
\end{aligned}
\end{equation*}

\hfill $\blacksquare$

%% file: A2_plbp_2.tex
\section{Supplement for Equation \ref{eq:lbp_e2} for Binary Models (Theorem \ref{the:main_conv} the $p_\infty$ induced Matrix Norm for Parallel Loopy Belief Propagation in Binary Models)}\label{app:lbp_e2}
Similar to the previous appendix, the key to Equation \ref{eq:lbp_e2} is finding the maximum  of:
\begin{equation*}
\begin{aligned}
    E = r_1 \sum_i \left|
    \frac{p_{1,i} - \sum_j r_j p_{i,j}}{\sum_j r_j p_{i,j}} \right|
\end{aligned}
\end{equation*}
where $p_{i,i}=1$, $p_{i,j} = p_{j,i}$ and $0 \leq p_{i,j} \leq 1$ subject to the constraints that $r_i \geq 0$ and $\sum_ir_i = 1$. For a binary model, that simplifies as:
\begin{equation*}
\begin{aligned}
    E & = r_1 \sum_i \left|\frac{p_{1,i} - \sum_j r_j p_{i,j}}{\sum_j r_j p_{i,j}} \right| \\
    & = r_1 \left( \left|\frac{(1-r_1)p_{1,1} - (1-r_1) p_{1,2}}{r_1 p_{1,1} + (1-r_1)p_{1,2}} \right| + \left|\frac{(1-r_1)p_{1,2} - (1-r_1) p_{2,2}}{r_1 p_{1,2} + (1-r_1)p_{2,2}} \right| \right) \\
    & = r_1 \left( \frac{(1-r_1)(1 - p_{1,2})}{r_1 + (1-r_1)p_{1,2}} + \frac{(1-r_1)(1-p_{1,2})}{r_1 p_{1,2} + (1-r_1)} \right) \\
    & = \frac{r_1 (1-r_1)(1-p_{1,2})(1+p_{1,2})}{(r_1^2+(1-r)^2)p_{1,2} + (1-r_1)r_1(p_{1,2}^2+1)} \\
\end{aligned}
\end{equation*}
Differentiating this with respect to $r_1$ gives us:

\begin{equation*}
\begin{aligned}
    \frac{\delta E}{\delta r_1} & = \frac{p_{1,2}(p_{1,2}^2 - 1) (2r_1 - 1)}{\left(p_{1,2}^2 (1-r_1) r_1 - p_{1,2} (-2 r_1^2 + 2 r_1 - 1) - (r_1 - 1) r_1\right)^2}\\
    0 & = p_{1,2}(p_{1,2}^2 - 1) (2r_1 - 1) \\
    r_1 & = 0.5 \mbox{ and  therefore } 1-r_1 = 0.5
\end{aligned}
\end{equation*}
Substituting back into the original equation yields:
\begin{equation*}
    E \leq \frac{0.25 (1-p_{1,2})(1+p_{1,2})}{0.5p_{1,2} + 0.25(p_{1,2}^2+1)} = \frac{1-p_{1,2}}{1 + p_{1,2}}
\end{equation*}
\hfill $\blacksquare$

%% file: A3_mfi_1.tex
\section{Mean-field estimation derivation}\label{app:mf_deriv}

As stated earlier, our overall equation of interest is:
\begin{equation*}
    E_{[\psi,u]} \hspace{-0.5mm}=\hspace{-0.5mm} \sum_L \hspace{-1mm}\left(\sum_x \hspace{-0.6mm}\psi_L(x) u_L(x) - \hspace{-3mm}\sum_{e=(x,y)}\hspace{-2.6mm} \psi_L(e)|u_L(x)-u_L(y)| \right)
\end{equation*}
with our goal being to create an approximate equation $\tilde{E}$ which contains only unary terms via the assumption of a static \textit{mean-field} which we can update. The common way to do this is to consider such a field $m$ which has similar properties to $u$ but relaxed:
\begin{equation*}
    m_L(x) \in [0,1] \mbox{, } u_P(x) = \mbox{\hspace{-3mm}}\sum_{L\in P.C^*\cap\mathcal{L}}\mbox{\hspace{-3mm}}m_L(x) \text{, and } m_S(x)=1
\end{equation*}
with the goal of $m$ being to approximate $u$ outside of an error term $d = u-m$ which is considered free to optimise/vary as needed. The first steps are to change the summation over the edges to a summation over variables (and take into account double-counting), change the summation over all labels to one over only leaf-labels, and then to transform the interaction term into an equivalent one that only only contains bilinear terms with respect to each $u$ variable. This gives us the equivalent equation:
\begin{equation*}
    E_{[\psi,u]} \hspace{-0.5mm}= \sum_{L\in\mathcal{L}} \sum_x\hspace{-1mm}\left( \hspace{-0.6mm}\psi_L(x) u_L(x) - \frac{1}{2}\sum_{y\in\mathcal{N}(x)}\sum_{L'\in\mathcal{L}/\{L\}}\left(\sum_{P \in L.P^* \triangle L'.P^*}\psi_P((x,y))\right)u_L(x)u_{L'}(y) \right)
\end{equation*}
where $\triangle$ is the disjoint union operator. (Note the double counting with respect to variable, $x$ and $y$ is taken into account but there can be no double counting with respect to the labels $L$ and $L'$ as each variable can only be assigned one leaf label, hence the $\nicefrac{1}{2}$ term for the interaction.)

We can now reformulate this equation with respect to $d$ as:
\begin{equation*}
\begin{aligned}
    E_{[\psi,u]} \hspace{-0.5mm}= & \underset{ \underset{L\in\mathcal{L}}{x}}{\sum}\hspace{-1mm}\left( \hspace{-0.6mm}\psi_L(x) (d_L(x)+m_L(x)) - \frac{1}{2}\sum_{y\in\mathcal{N}(x)}\sum_{L'\in\mathcal{L}/\{L\}}\left(\sum_{P \in L.P^* \triangle L'.P^*}\psi_P((x,y))\right)(d_L(x)+m_L(x))(d_{L'}(y)+m_{L'}(y)) \right) \\
    = & \underset{ \underset{L\in\mathcal{L}}{x}}{\sum}\hspace{-1mm}\left( \hspace{-0.6mm}\psi_L(x) d_L(x) - \frac{1}{2}\underset{ \underset{L'\in\mathcal{L}/\{L\}}{y\in\mathcal{N}(x)}}{\sum}\left(\sum_{P \in L.P^* \triangle L'.P^*}\psi_P((x,y))\right)(d_L(x)m_{L'}(y)+d_{L'}(y)m_L(x))+d_L(x)d_{L'}(y)) \right) + c
\end{aligned}
\end{equation*}
where $c$ is a constant with respect to $d$. The next step in mean-field iteration theory is to discard the interaction between error terms as approximately zero (when averaged over the entire field of label-variable pairs). This yields the following approximation:
\begin{equation*}
\begin{aligned}
    E_{[\psi,u]} \hspace{-0.5mm}\approx & \underset{ \underset{L\in\mathcal{L}}{x}}{\sum}\hspace{-1mm}\left( \hspace{-0.6mm}\psi_L(x) d_L(x) - \frac{1}{2}\underset{ \underset{L'\in\mathcal{L}/\{L\}}{y\in\mathcal{N}(x)}}{\sum}\left(\sum_{P \in L.P^* \triangle L'.P^*}\psi_P((x,y))\right)(d_L(x)m_{L'}(y)+d_{L'}(y)m_L(x))) \right) + c \\
    = & \underset{ \underset{L\in\mathcal{L}}{x}}{\sum}\hspace{-1mm}\left( \hspace{-0.6mm}\psi_L(x) d_L(x) - \underset{ \underset{L'\in\mathcal{L}/\{L\}}{y\in\mathcal{N}(x)}}{\sum}\left(\sum_{P \in L.P^* \triangle L'.P^*}\psi_P((x,y))\right)d_L(x)m_{L'}(y) \right) + c \\
    = & \underset{ \underset{L\in\mathcal{L}}{x}}{\sum}\hspace{-1mm}\left( \hspace{-0.6mm}\psi_L(x) - \underset{ \underset{L'\in\mathcal{L}/\{L\}}{y\in\mathcal{N}(x)}}{\sum}\left(\sum_{P \in L.P^* \triangle L'.P^*}\psi_P((x,y))\right)m_{L'}(y) \right)d_L(x) + c
\end{aligned}
\end{equation*}
As this estimate now contains only unary terms, we can replace $d$ with $u$ again, and determine the energy associated with assigning a particular label $L$ to a particular location $x$ as
\begin{equation*}
\begin{aligned}
    \tilde{E}(L,x) = & \left(\psi_L(x) - \underset{ \underset{L'\in\mathcal{L}/\{L\}}{y\in\mathcal{N}(x)}}{\sum}\left(\sum_{P \in L.P^* \triangle L'.P^*}\psi_P((x,y))\right)m_{L'}(y) \right)(1-m_L(x)) \\
    &- \sum_{L''\in\mathcal{L}/\{L\}} \left(\psi_{L''}(x) - \underset{ \underset{L'\in\mathcal{L}/\{L''\}}{y\in\mathcal{N}(x)}}{\sum}\left(\sum_{P \in L''.P^* \triangle L'.P^*}\psi_P((x,y))\right)m_{L'}(y) \right)m_{L''}(x) \\
    = & \psi_L(x) - \underset{ \underset{L'\in\mathcal{L}/\{L\}}{y\in\mathcal{N}(x)}}{\sum}\left(\sum_{P \in L.P^* \triangle L'.P^*}\psi_P((x,y))\right)m_{L'}(y) \\
    &- \sum_{L''\in\mathcal{L}} \left(\psi_{L''}(x) - \underset{ \underset{L'\in\mathcal{L}/\{L''\}}{y\in\mathcal{N}(x)}}{\sum}\left(\sum_{P \in L''.P^* \triangle L'.P^*}\psi_P((x,y))\right)m_{L'}(y) \right)m_{L''}(x) \\
    = & \psi_L(x) - \underset{ \underset{L'\in\mathcal{L}/\{L\}}{y\in\mathcal{N}(x)}}{\sum}\left(\sum_{P \in L.P^* \triangle L'.P^*}\psi_P((x,y))\right)m_{L'}(y)  + c\\
\end{aligned}
\end{equation*}
where $c$ is a constant with respect to $L$ and can thus be discarded since it does not affect this approximation of the probability distribution of $u_L(x)$. This can be further simplified as:
\begin{equation*}
\begin{aligned}
    \tilde{E}(L,x) = & \psi_L(x) - \underset{ \underset{L'\in\mathcal{L}}{y\in\mathcal{N}(x)}}{\sum}\left(\sum_{P \in L.P^* \triangle L'.P^*}\psi_P((x,y))\right)m_{L'}(y) \mbox{ noting $A\triangle A=\emptyset$} \\
    = & \psi_L(x) - \underset{y\in\mathcal{N}(x)}{\sum}\left(\underset{P\in L.P^* }{\sum}\psi_P((x,y))\underset{L'\not\in P.C^* }{\sum}m_{L'}(x) + \underset{P\not\in L.P^*}{\sum}\psi_P((x,y))\underset{L'\in P.C^* }{\sum}m_{L'}(x) \right) \\
    = & \psi_L(x) - \underset{y\in\mathcal{N}(x)}{\sum}\left(\underset{P\in L.P^* }{\sum}\psi_P((x,y))(1-m_P) + \underset{P\not\in L.P^* }{\sum}\psi_P((x,y))m_P\right) \\
    = & \psi_L(x) - \underset{y\in\mathcal{N}(x)}{\sum}\left(\underset{P\in L.P^* }{\sum}\psi_P((x,y))(1-2m_P) + \underset{P}{\sum}\psi_P((x,y))m_P\right)\\
    = & \psi_L(x) + \underset{y\in\mathcal{N}(x)}{\sum}\underset{P\in L.P^* }{\sum}\psi_P((x,y))(2 m_P(x)-1) + c
\end{aligned}
\end{equation*}
where $c$ is again a constant with respect to $L$, noting this gives us the energy estimate from Equation \ref{eq:meanenergy} barring differences in indexing.

The question then what should the mean-field, $m$, be. The obvious answer would be the marginal of the current estimate of the probability distribution, being easily computed from $\tilde{E}$ and updated iteratively, leading to the mean-field iteration estimator.

To estimate joint marginal probabilities between a specific pair of adjacent nodes, we can do the same mean-field estimation, although in practice, it would be more computationally efficient to use the vertex-wise mean field estimates to estimate both the vertex-wise marginals and the edge-wise joint marginals simultaneously at a final step.

%% file: A4_mfi_2.tex
\section{Supplement for Equation \ref{eq:mfi_eq}}\label{app:mfi_eq}
The main part of Eq. \ref{eq:mfi_eq} relies on the fact that:
$$\underset{m_{[n-1,m_0]}(x',L_l)}{
\mbox{argmax}}\sum_{L_j}m_{[n-1,m_0]}(x',L_j)\left| \left(\sum_{P \in L_1.P^* \cap L_j.P^*}\psi_P(e)\right) - \sum_{L_k}\left( m_{[n-1,m_0]}(x',L_k) \sum_{P \in L_1.P^* \cap L_k.P^*}\psi_P(e) \right) \right|$$
takes on the configuration
\begin{equation*}
    \begin{aligned}
        m_{[n-1,m_0]}(x',L_1) = & \;\nicefrac{1}{2} \mbox{ noting } L_1 = \underset{L_l}{\mbox{argmax}}\;\left(\underset{P\in L_1.P^*\cap L_l.P^*}{\sum}\psi_P((x,x'))\right)\\
        m_{[n-1,m_0]}(x',L_m) = &\; \nicefrac{1}{2} \mbox{ where } L_m = \underset{L_l}{\mbox{argmin}}\;\left(\underset{P\in L_1.P^*\cap L_l.P^*}{\sum}\psi_P((x,x'))\right) \\
        m_{[n-1,m_0]}(x',L_l) = &\;  0 \mbox{ else}
    \end{aligned}
\end{equation*}
For the sake of notation, let us define:
\begin{equation*}
\begin{aligned}
    p_i &= \left(\underset{P\in L_1.P^*\cap L_l.P^*}{\sum}\psi_P((x,x'))\right) \mbox{ ordering $L_l$ such that } p_1 \geq p_2 \geq ... \geq p_m \\
    r_i &= m_{[n-1,m_0]}(x',L_l) \mbox{ in the same order} \\
    \bar{p} &= \sum_i r_i p_i
\end{aligned}
\end{equation*}
and note that $p_i > 0$ by the set-up of our model, and note that $r_i \geq 0$, and $\sum_i r_i = 1$ by the normalisation performed in the forward pass.

Without loss of generality, assume $p_1 \neq p_m$. For any given $r$, we can partition the our labels into two non-empty sets, specifically between those where $r_i \geq \bar{p}$ and those where $r_i < \bar{p}$, at a particular index $N$.
\begin{equation*}
\begin{aligned}
    E = & \sum_{i \leq N} r_i \left( 
    p_i - \bar{p} \right) + \sum_{i > N} r_i \left(
    \bar{p} - p_i \right) \\
    = & \left(\sum_{i \leq N} \left(
    r_i p_i - \sum_{j \leq N} r_i r_j p_j - \sum_{j > N} r_i r_j p_j \right) + \sum_{i > N} \left(
    \sum_{j \leq N} r_i r_j p_j + \sum_{j > N} r_i r_j p_j - r_i p_i \right)\right) \\
\end{aligned}
\end{equation*}
By differentiating this with respect to $r_i$, we get two classes of derivatives:
\begin{equation*}
\begin{aligned}
\frac{\delta E}{\delta r_i} = & \left\{ \begin{array}{ll}
     \left((1-2r_i)p_i - \underset{j \leq N, j\neq i}{\sum} r_j(p_i+p_j) + \underset{j > N}{\sum} r_j(p_i-p_j)\right) & \mbox{ if } i \leq N \\
     \left((2r_i-1)p_i + \underset{j > N, j\neq i}{\sum} r_j(p_i+p_j) - \underset{j \leq N}{\sum} r_j(p_i-p_j)\right) & \mbox{ if } i > N
\end{array}\right. \\
\frac{\delta E}{\delta r_i} = & \left\{ \begin{array}{ll}
     p_i - \underset{j \leq N}{\sum} r_j(p_i+p_j) + \underset{j > N}{\sum} r_j(p_i-p_j) & \mbox{ if } i \leq N \\
     p_i + \underset{j > N}{\sum} r_j(p_i+p_j) - \underset{j \leq N}{\sum} r_j(p_i-p_j) & \mbox{ if } i > N
\end{array}\right.
\end{aligned}
\end{equation*}
Considering the difference between $\frac{\delta E}{\delta r_i}$ where $1 < i \leq N$ and $\frac{\delta E}{\delta r_1}$ noting that they are of the same class.
\begin{equation*}
\begin{aligned}
    \frac{\delta E}{\delta r_1} - \frac{\delta E}{\delta r_i} = & 
     p_1 - p_i - \underset{j \leq N}{\sum} r_j(p_1+p_j-p_i-p_j) + \underset{j > N}{\sum} r_j(p_1-p_j-p_i+p_j) \\
     = &
     p_1 - p_i - \underset{j \leq N}{\sum} r_j(p_1-p_i) + \underset{j > N}{\sum} r_j(p_1-p_i) \\ 
     = &
     (p_1 - p_i) \left(1 - \underset{j \leq N}{\sum} r_j + \underset{j > N}{\sum}r_j\right) \\
     = &
     (p_1 - p_i) \left(2\underset{j > N}{\sum}r_j\right) \geq 0
\end{aligned}
\end{equation*}
Thus, any distribution of weight $r_i, i \leq N$ towards $r_1$ results in equal or higher values of $E$ regardless of $N$. A similar calculation can be done for distributing $r_i, i > N$ towards $r_n$. Therefore, the maximum must occur when all the weight is centralised in these two labels, i.e. $r_1+r_n = 1$ and $r_i = 0$ else. This simplifies our formula to:
\begin{equation*}
\begin{aligned}
    E = & r_1 \left(
    p_1 - r_1 p_1 - (1-r_1)p_m \right) + (1-r_1) \left(
    r_1 p_1 + (1-r_1)p_m - p_m \right)\\
    = & r_1 (1-r_1)(p_1-p_m) + (1-r_1)r_1(p_1-p_m)\\
    = & 2r_1 (1-r_1)(p_1-p_m)
\end{aligned}
\end{equation*}
which has an obvious maximum value at $r_1=\nicefrac{1}{2}$.
This gives us the desired distribution given $p_1 \neq \min p_m$. In the case where $\max_i p_i = \min_i p_i$, note that all values of $p_i$ are equal and equal to $\bar{p}$ and thus $E=0$ regardless of the value of $r$ so $r_1=r_m=\nicefrac{1}{2}$ can still be considered maximal. \hfill $\blacksquare$